\documentclass[10pt]{article}
\usepackage{geometry}
\usepackage{amsmath,amssymb,amsthm,mathrsfs,nicefrac}
\usepackage{mathtools}
\usepackage{latexsym,amscd,amsbsy,dsfont,amsfonts}
\usepackage{graphicx}
\usepackage{xcolor}
\usepackage[pdftex,colorlinks=true]{hyperref}
\usepackage{caption}
\usepackage{physics}
\usepackage{tensor}
\usepackage{tikz}
\usetikzlibrary{bayesnet}
\usetikzlibrary{fit}
\usepackage[most]{tcolorbox}
\tcbuselibrary{theorems}
\usepackage{empheq}
\usepackage{enumitem}
\usepackage{wrapfig}
\usepackage{pythonhighlight}

\usepackage{cite}
\newcommand{\comprimi}{\medmuskip=0mu
\thinmuskip=0mu
\thickmuskip=0mu}

\usepackage{float}
\usepackage{graphicx}
\def\new{\text{new}}

\newcommand{\tot}{\textrm{tot}}

\def\dd{\text{d}}

\def\sign{\text{sign}}

\DeclareMathOperator{\Prox}{Prox}
\DeclareMathOperator*{\Extr}{Extr}
\DeclareMathOperator*{\Argmin}{argmin}

\newcommand{\bsA}{{\boldsymbol{\mathsf A}}}
\newcommand{\bG}{{\boldsymbol{G}}}
\newcommand{\bT}{{\boldsymbol{T}}}
\newcommand{\bOne}{{\boldsymbol{1}}}
\newcommand{\bR}{{\boldsymbol{R}}}
\newcommand{\bhR}{{\hat{\boldsymbol{R}}}}

\newcommand{\bH}{{\boldsymbol{H}}}
\newcommand{\bA}{{\boldsymbol{A}}}
\newcommand{\bC}{{\boldsymbol{C}}}
\newcommand{\bS}{{\boldsymbol{S}}}
\newcommand{\bhS}{{\hat{\boldsymbol{S}}}}
\newcommand{\bU}{{\boldsymbol{U}}}
\newcommand{\bhU}{{\hat{\boldsymbol{U}}}}
\newcommand{\bh}{{\boldsymbol{h}}}
\newcommand{\bW}{{{\boldsymbol{{W}}}}}
\newcommand{\bw}{{\boldsymbol{w}}}
\newcommand{\bQ}{{\boldsymbol{Q}}}
\newcommand{\bV}{{\boldsymbol{V}}}
\newcommand{\bsQ}{{\boldsymbol{\mathsf Q}}}
\newcommand{\bsS}{{\boldsymbol{\mathsf S}}}
\newcommand{\bM}{{\boldsymbol{m}}}
\newcommand{\bhQ}{{\hat{\boldsymbol{Q}}}}
\newcommand{\bhV}{{\hat{\boldsymbol{V}}}}
\newcommand{\bhM}{{\hat{\boldsymbol{m}}}}
\newcommand{\bMM}{{{\boldsymbol{M}}}}

\newcommand{\bbb}{{\boldsymbol{b}}}
\newcommand{\bv}{{\boldsymbol{v}}}
\newcommand{\ba}{{\boldsymbol{a}}}
\newcommand{\bbB}{{\boldsymbol{B}}}
\newcommand{\bF}{{\boldsymbol{F}}}
\newcommand{\bZ}{{\boldsymbol{Z}}}
\newcommand{\bY}{{\boldsymbol{Y}}}
\newcommand{\bTheta}{{\boldsymbol{\Theta}}}
\newcommand{\bSigma}{{\boldsymbol{\Sigma}}}
\newcommand{\bPi}{{\boldsymbol{\Pi}}}

\newcommand{\bphi}{{\boldsymbol{\phi}}}
\newcommand{\bsigma}{{\boldsymbol{\sigma}}}
\newcommand{\blambda}{{\boldsymbol{\lambda}}}
\newcommand{\bomega}{{\boldsymbol{\omega}}}

\newcommand{\bXi}{{\boldsymbol{\Xi}}}
\newcommand{\bmu}{{\boldsymbol{\mu}}}
\newcommand{\bvarphi}{{\boldsymbol{\varphi}}}
\newcommand{\bx}{{\boldsymbol{x}}}
\newcommand{\bn}{{\boldsymbol{n}}}
\newcommand{\bX}{{\boldsymbol{X}}}
\newcommand{\bI}{{\boldsymbol{I}}}
\newcommand{\by}{{\boldsymbol{y}}}
\newcommand{\bff}{{\boldsymbol{f}}}
\newcommand{\be}{{\boldsymbol{e}}}
\newcommand{\beeta}{{\boldsymbol{\eta}}}
\newcommand{\bxi}{{\boldsymbol{\xi}}}

\newcommand{\bu}{{\boldsymbol{u}}}

\def\mat#1{\text{#1}}
\renewcommand{\vec}[1]{\bm{#1}}

\title{Learning Gaussian Mixtures with Generalised Linear Models: Precise Asymptotics in High-dimensions}

\usepackage{authblk}

\author[1]{Bruno Loureiro}
\author[2]{Gabriele Sicuro}
\author[3]{C\'edric Gerbelot}
\author[1]{Alessandro Pacco}
\author[1]{\\Florent Krzakala}
\author[4]{Lenka Zdeborov\'a}
{
\affil[1]{\small IdePHICS Lab. EPFL, Lausanne}
\affil[2]{\small Department of Mathematics, King's College London}
\affil[3]{\small Laboratoire de Physique de l'\'Ecole Normale Sup\'erieure, Universit\'e PSL, CNRS, Sorbonne Universit\'e}
\affil[4]{\small SPOC, EPFL, Lausanne}
}
\newtheorem{theorem1}{Theorem}
\newtheorem{theorem2}[theorem1]{Theorem}
\newtheorem{corollary1}[theorem1]{Corollary}

\newtheorem{Lemma}[theorem1]{Lemma}
\newtheorem{definition}{Definition}

\usepackage{hyperref}
\hypersetup{pdfauthor={IdePHICS},pdftitle={MultiClass},%
            colorlinks, linktocpage=true, pdfstartpage=1, pdfstartview=FitV,%
    breaklinks=true, pdfpagemode=UseNone, pageanchor=true, pdfpagemode=UseOutlines,%
    plainpages=false, bookmarksnumbered, bookmarksopen=true, bookmarksopenlevel=1,%
    hypertexnames=true, pdfhighlight=/O,%
    urlcolor=orange, linkcolor=blue, citecolor=red, 
        }

\usepackage[cal=euler]{mathalfa}
\usepackage{libertine}
\usepackage[libertine,smallerops]{newtxmath}
\usepackage[T1]{fontenc}

\allowdisplaybreaks
\begin{document}

\maketitle

\begin{abstract}
  Generalised linear models for multi-class classification problems are one of the fundamental building blocks of modern machine learning tasks. In this manuscript, we characterise the learning of a mixture of $K$ Gaussians with generic means and covariances via empirical risk minimisation (ERM) with any convex loss and regularisation. In particular, we prove exact asymptotics characterising the ERM estimator in high-dimensions, extending several previous results about Gaussian mixture classification in the literature. We exemplify our result in two tasks of interest in statistical learning: a) classification for a mixture with sparse means, where we study the efficiency of $\ell_1$ penalty with respect to $\ell_2$; b) max-margin multi-class classification, where we characterise the phase transition on the existence of the multi-class logistic maximum likelihood estimator for $K>2$. Finally, we discuss how our theory can be applied beyond the scope of synthetic data, showing that in different cases Gaussian mixtures capture closely the learning curve of classification tasks in real data sets.
\end{abstract}

\section{Introduction}\label{sec:intro}
A recurring observation in modern deep learning practice is that neural networks often defy the standard wisdom of classical statistical theory. For instance, deep neural networks typically achieve good generalisation performances at a regime in which it interpolates the data, a fact at odds with the intuitive bias-variance trade-off picture stemming from classical theory \cite{Geman1992, Hastie2001, Belkin2019}. Surprisingly, many of the ``exotic'' behaviours encountered in deep neural networks have recently been shown to be shared by models as simple as overparametrised linear classifiers \cite{Hastie2020, Belkin2018}, e.g., the aforementioned benign over-fitting \cite{Bartlett2020}. Therefore, understanding the generalisation properties of simple models in high-dimensions has proven to be a fertile ground for elucidating some of the challenging statistical questions posed by modern machine learning practice \cite{Mei2019, Gerace2020, ghorbani2019limitations, Goldt2020, Goldt2020b, Loureiro2021, Sur2020, Mignacco2020, Refinetti2021, Candes2020}. 

In this manuscript, we pursue this enterprise in the context of a commonly used model for high-dimensional classification problems: the Gaussian mixture. Indeed, it has been recently argued that the features learned by deep neural networks trained on the cross-entropy loss ``collapse'' in a mixture of well-separated clusters, with the last layer acting as a simple linear classifier \cite{Donoho2020}. Another observation put forward in \cite{Couillet2020} is that data obtained using generative adversarial networks behave as Gaussian mixtures.
Here, we derive an exact asymptotic formula characterising the performance of generalised linear classifiers trained on $K$ Gaussian clusters with generic covariances and means. Our formula is valid for any convex loss and penalty, encompassing popular tasks in the machine learning literature such as ridge regression, basis pursuit, cross-entropy minimisation and max-margin estimation. This allow us to answer relevant questions for statistical learning, such as: what is the separability threshold for $K$-clustered data? How does regularisation affects estimation? Can different penalties help when the means are sparse? 
We also extend the observation of \cite{Couillet2020} showing that the learning curves of binary classification tasks on {\it real data} are indeed well captured by our asymptotic analysis. 
\paragraph{Model definition} We consider learning from a $d$-dimensional mixture of $K$ Gaussian clusters $\mathcal C_{k\in[K]}$. The data set is obtained by sampling $n$ pairs $(\bx^{\nu},\by^{\nu})_{\nu \in [n]} \in \mathbb{R}^{d+K}$ identically and independently. We adopt the one-hot encoded representation of the labels, i.e., if $\bx^\nu\in\mathcal C_k$, then $\by^{\nu}=\be_k$, $k$th basis vector of $\mathbb R^K$. We will denote the matrix of concatenated samples $\bX \in \mathbb{R}^{d\times n}$. The mixture density then reads:
\begin{equation}
\label{joint}
 P(\bx,\by)
 = \sum_{k=1}^{K}y_k\rho_k\mathcal{N}\left(\bx\left|\bmu_k,\bSigma_k\right.\right),
\end{equation}
where $\mathcal N(\bx|\bmu,\bSigma)$ is the multivariate normal distribution with mean $\bmu$ and covariance matrix $\bSigma$. The matrix of concatenated means is denoted $\bMM\in \mathbb{R}^{d \times K}$. In Eq.~\eqref{joint}, $\forall k$, $\rho_k=P(\by=\be_k)\geq 0$, $\bmu_k\in\mathbb R^d$ and $\bSigma_k\in\mathbb R^{d\times d}$ is positive-definite. We will consider the estimator obtained by minimising the following empirical risk:
\begin{align}
\label{ERM}
&\mathcal R(\bW,\bbb)\equiv \sum_{\nu=1}^n\ell\left(\by^\nu,\frac{\bW\bx^\nu}{\sqrt d}+\bbb\right)+\lambda r(\bW), \\
& (\bW^{\star},\bbb^{\star}) \equiv \Argmin_{\bW\in \mathbb{R}^{K\times d},\,\bbb \in \mathbb{R}^{K}}\mathcal R(\bW,\bbb)\, , 
\end{align}
where $\bW\in\mathbb R^{K\times d}$ and $\bbb\in\mathbb R^{K}$ are the weights and bias to be learned, $\ell$ is a convex loss function, and $r$ is a regularisation function whose strength is tuned by the parameter $\lambda\geq0$. For example the loss function $\ell$ can represent the composition of a cross-entropy loss with a softmax thresholding on the linear part of Eq.~\eqref{ERM}.
We will characterise the distribution of the estimator $(\bW^{\star},\bbb^{\star})$, and we will evaluate the average training loss defined as
\begin{equation}
\epsilon_\ell=\frac{1}{n}\sum_{\nu=1}^n\ell\left(\by^\nu,\frac{\bW^\star\bx^\nu}{\sqrt d}+\bbb^\star\right),
\end{equation}
as well as the average training error $\epsilon_t$ and generalisation error $\epsilon_g$, defined as the misclassification rates:
\begin{equation}
\epsilon_t=\frac{1}{n}\sum_{\nu=1}^n\mathbb I\left[\by^\nu\neq\hat\by\left(\frac{\bW^\star\bx^\nu}{\sqrt d}+\bbb^\star\right)\right],\ \
\epsilon_g=\mathbb E_{(\bx^{\rm new},\by^{\rm new})}\left[\mathbb I\left[\by^{\rm new}\neq\hat\by\left(\frac{\bW^\star\bx^{\rm new}}{\sqrt d}+\bbb^\star\right)\right]\right],\nonumber
\end{equation}
where $(\bx^{\rm new},\by^{\rm new})$ is a new unseen data point sampled from the distribution in Eq.~\eqref{joint}. In the previous equations, we have used the function $\hat \by\colon \mathbb R^K\to\mathbb R^K$, so that $\hat y_{k}(\bx)\coloneqq\mathbb I(\max_\kappa x_\kappa=x_k)$. 

The \textbf{main contributions} in this manuscript are the following:
\begin{description}[wide = 0pt,noitemsep]
\item[(C1)] In Sec.~\ref{sec:main} and Appendix \ref{sec:app:proof} we prove closed-form equations characterizing the asymptotic distribution of the matrix of weights $\bW^\star \in \mathbb{R}^{K\times d}$, enabling the exact computation of key quantities such as the training and generalisation error. Our proof method solves shortcomings of previous approaches by introducing a novel approximate message-passing sequence, building on recent advances in this framework, that is of independent interest. 
\item[(C2)] In Sec.~\ref{sec:examples:sparse} we study the problem of classifying an anisotropic mixture with sparse means, where the strong or weak directions in the data are correlated with the non-zero components of the mean as in \cite{Donoho2008}. We study how learning the sparsity with an $\ell_1$ penalty improves the classification performance.
\item[(C3)] In Sec.~\ref{sec:examples:multi} we study the performance of the cross-entropy estimator in the limit of vanishing regularisation $\lambda\to 0^{+}$ for $K$ Gaussian clusters as a function of the sample complexity $\alpha=\nicefrac{n}{d}$; we show that a phase transition takes place at a certain value $\alpha^\star_K$ between a regime of complete separability of the data and a regime in which the correct classification of almost all points in the data set is not possible. We also investigate the effect of $\lambda > 0$ regularisation on the generalisation error, comparing the $K>2$ case with the results given in the literature for $K=2$ \cite{Mignacco2020,Thrampoulidis2020}.  
\item[(C4)] In Sec. \ref{sec:examples:mnist} we investigate the applicability of our formula beyond the Gaussian assumption by applying it to classification tasks on \emph{real data}. We show that for different tasks and losses, it closely captures the real learning curves, even when data is mapped through a non-linear feature map. This further shows that Gaussian mixtures are a good surrogate model for investigating real classification tasks, as put forward in \cite{Couillet2020}.
\end{description}
\paragraph{Relation to previous work} The analysis of Gaussian mixture models in the high-dimensional regime has been the subject of many recent works. Exact asymptotics has been derived for the binary classification case with diagonal covariances in \cite{Thrampoulidis2019, Couillet2019, Mai2020} for the logistic loss and in \cite{Dobriban2018, Thrampoulidis2020b} for the square loss, both with $\ell_2$ penalty. A similar analysis has been performed in \cite{Sifaou2019} for the hard-margin SVM. These works were generalised to generic convex losses and $\ell_2$ penalty in \cite{Mignacco2020}, where it has been also shown that the regularisation term can play an important role in reaching Bayes-optimal performances. 
Hinge regression with $\ell_1$ penalty and diagonal covariance was treated in \cite{Sur2020}. Recently, these asymptotic results were generalised to the case in which both clusters share the same covariance in \cite{Thrampoulidis2021}, and finite rate bounds were given in \cite{Long2021, Belkin2021} in the case of sub-Gaussian mixtures. Asymptotic results for the multiclass problem with diagonal covariance were derived in \cite{Thrampoulidis2020} for the restricted case of the square loss with $\ell_2$ penalty. Our result unifies all the aforementioned asymptotic formulas, and extends them to the general case of a multiclass problem with generic covariances and arbitrary convex losses and penalties.\\
From a technical standpoint, in \cite{Thrampoulidis2019, salehi2019impact, Thrampoulidis2020b, Thrampoulidis2020, Mignacco2020, Sur2020, Thrampoulidis2021} the authors use convex Gaussian comparison inequalities, see e.g.~\cite{thrampoulidis2018precise,stojnic2013framework}, to prove their result. In particular, the proof given in \cite{Thrampoulidis2020} for the multiclass problem harnesses the geometry of least-squares, and it is then stressed that this method breaks down for multiclass problems in which the risk does not factorise over the $K$ clusters (as for the cross-entropy, for example). We solve this problem using an innovative proof technique which has an interest in its own. Our approach is to capture
 the effect of non-linearity and generic covariances via the rigorous study of an approximate message-passing (AMP) sequence, a family of iterations that admit closed-form asymptotics at each step called \emph{state evolution equations} \cite{bayati2011dynamics}. Our proof relies on several refinements of AMP methods to handle the full complexity of the problem, notably spatial coupling with matrix valued variables \cite{krzakala2012probabilistic,donoho2013information,javanmard2013state} and non-separable update functions \cite{berthier2020state}, via a multi-layer approach to AMP \cite{manoel2017multi}.\\
The sparse Gaussian mixture model analysed in Section \ref{sec:examples:sparse} is closely related to the rare/weak features model introduced in \cite{Donoho2008} and widely studied in the context of sparse linear discriminant analysis \cite{Jiashun2009, Shao2011, Qing2012, Li2017}. It was recently revisited in \cite{Belkin2021, Long2021} in the context of ERM with max-margin classifiers. Here, we consider a correlated variation of the model and study the benefit of using a sparsity inducing $\ell_1$ penalty.\\
The separability transition is a classical topic \cite{Cover1965,gardner1988space} that has recently witnessed a renewal of interest thanks to its connection to overparametrization. It was studied in \cite{Candes2020} in the context of uncorrelated Gaussian data, in \cite{Gerace2020} in the random features model and in \cite{Thrampoulidis2019, Mignacco2020} for binary Gaussian mixtures. \\
Recently, \cite{jacot2020kernel, bordelon2020, Loureiro2021} showed that the performance of different regression tasks on real data are well-captured by a teacher-student Gaussian model in high-dimensions for ridge regression, but this turned not to be true for non-linear problems such as logistic classification \cite{Loureiro2021}. Authors of \cite{Couillet2020} showed instead that data from generative adversarial networks behave like Gaussian mixtures, motivating the modeling of such mixture for real-data in the present paper. 

\section{Technical results}\label{sec:main}
Our main technical result is an exact asymptotic characterization of the distribution of the estimator~$\bW^{\star}$. Informally, the estimator $\bW^{\star}$ and the quantity $\bW^{\star}\bX/\sqrt{d}$ behave asymptotically as non-linear transforms of multivariate Gaussian distributions. These transforms are directly linked to the proximal operators \cite{parikh2014proximal,bauschke2011convex} associated to the loss and regularisation functions, summarizing the effect of the cost function landscape on the estimator. The parameters of these Gaussian distributions and proximals can then be computed from the fixed point of a self-contained set of equations. We start by presenting the most generic form of our result in a concentration of measure-like statement in Theorem~\ref{the:1}, and discuss an intuitive interpretation of the different quantities involved. Theorem \ref{th:2} then states how the training and generalisation errors can be computed. All results presented in the experiments section can be obtained from Theorem \ref{the:1}. In Corollary \ref{corr:1} we discuss a particular case where explicit simplifications can be obtained. But first, let's summarise the required assumptions for our result to hold.
\begin{description}[wide = 0pt,noitemsep]
\label{set:assump}
	\item[(A1)] \label{assu:conv} The functions $\ell$ (as a function of its second argument) and $r$ are proper, closed, lower semi-continuous convex functions. We assume additionally
	that either the cost function $\ell(\by,\bullet\bX)+r(\bullet)$ is strictly convex, or that $\ell(\by,\bullet)$ is strictly convex in its second argument and $r$ is the $\ell_{1}$ norm. We also assume that the cost function $\ell(\by,\bullet\bX)+r(\bullet)$ is coercive.
	\item[(A2)] The covariance matrices are positive definite and their spectral norms are bounded (with probability one). 
	\item[(A3)] The mean vectors $\bmu_{k}$ are distributed according to some density $P_\bmu(\bMM)$ such that the following quantity is finite
	\begin{equation}
	    \forall d\qquad  \mathbb{E}\left[\norm{\bMM^\top\bMM}_{\rm F}\right] < +\infty,
	\end{equation}
	where  $\|\bullet\|_{\rm F}$ denotes the Frobenius norm.
	\item[(A4)] The number of samples $n$ and dimension $d$ both go to infinity with fixed ratio $\alpha = \nicefrac{n}{d}$, called hereafter the sample complexity. The number of clusters $K$ is finite.
	\item[(A5)] The fixed point of the set of self-consistent equations Eq.\eqref{spgeneral} exists and is unique.
\end{description}
As specified by assumption \textbf{(A1)}, our proof does not apply to any convex problem. We discuss this assumption further in Appendix \ref{app:strct}. We also comment on the existence and uniqueness of the solution to the set of self consistent equations Eq.\eqref{spgeneral} in Appendix \ref{app:unique}.
Before proceeding further, let us specify a useful notation. Suppose that the matrix $\bG=(G_{ki})_{ki}\in\mathbb R^{K\times d}$ is given, alongside the four-index tensor $\bsA=(A_{ki\,k'i'})_{ki\,k'i'}\in\mathbb R^{K\times d}\otimes\mathbb R^{K\times d}$. We will use the notation $\bG\odot \bsA=\sum_{ki}G_{ki}A_{ki\,k'i'}\in\mathbb R^{K\times d}$. Similarly, given a four-index tensor $\bsA$, we will define $\sqrt{\bsA}$ as the tensor such that $\bsA=\sqrt\bsA\odot \sqrt\bsA$. We are now in a position to state our main result.
\begin{theorem1}[Concentration properties of the estimator]\label{the:1} Let $\bxi_{k \in [K]}\sim \mathcal N(\mathbf 0,\bI_K)$ be  collection of $K$-dimensional standard normal vectors independent of other quantities. Let also be $\{\bXi_k\}$ a set of $K$ matrices, $\bXi_k\in\mathbb R^{K\times d}$, with i.i.d.~standard normal entries, independent of other quantities. Under the set of assumptions (A1--A5), for any pseudo-Lispchitz functions of finite order $\phi_{1}:\mathbb{R}^{K\times d} \to \mathbb{R}, \phi_{2}: \mathbb{R}^{K\times n}\to \mathbb{R}$, the estimator $\bW^{\star}$ and the matrix $\boldsymbol{Z}^{\star} = \frac{1}{\sqrt{d}}\bW^{\star}\bX$ verify:
\begin{align}
    \phi_{1}(\bW^{\star}) \xrightarrow[n,d \to +\infty]{P}\mathbb{E}_{\bXi}\left[\phi_{1}(\bG)\right], && \phi_{2}(\boldsymbol{Z}^{\star})\xrightarrow[n,d \to +\infty]{P}\mathbb{E}_{\bxi}\left[\phi_{2}(\bH)\right]\, ,
\end{align}
where we have introduced the proximal for the loss:
\begin{equation}
    \bh_k=\bV_k^{1/2}\Prox_{\ell(\be_k,\bV_k^{1/2}\bullet)}(\bV^{-1/2}_k\bomega_{k})\in\mathbb R^{K}\, , \qquad
\boldsymbol{\omega}_k\equiv\bM_k+\bbb+\bQ^{1/2}_k\bxi_{k}\, ,
\end{equation}
and $\boldsymbol{H} \in \mathbb{R}^{K \times n}$ is obtained by concatenating each $\bh_k$, $\rho_{k}n$ times. 
We have also introduced the matrix proximal $\bG\in\mathbb R^{K\times d}$:
\begin{equation}
\bG=\bsA^{\frac{1}{2}}\odot \Prox_{r(\bsA^{\frac{1}{2}}\odot \bullet)}(\bsA^{\frac{1}{2}}\odot \bbB),\ \ \ \ 
\bsA^{-1}\equiv\sum_k \bhV_k\otimes\bSigma_k,\ \  \bbB\!\equiv\sum_k\!\left(\!\bmu_k\bhM_k^\top\!+\!\bXi_k\odot \sqrt{\bhQ_k\!\otimes\!\bSigma_k }\!\right). \nonumber
\end{equation}
The collection of parameters $(\bQ_{k},\bM_{k},\bV_{k},\bhQ_{k},\bhM_{k},\bhV_{k})_{k\in[K]}$ is given by the fixed point of the following self-consistent equations:
\begin{equation}\label{spgeneral}
\begin{cases}
\bQ_k\!=\!\frac{1}{d}\mathbb E_{\bXi}[\bG\bSigma_k\bG^\top]\\
\bM_k\!=\!\frac{1}{\sqrt d}\mathbb E_{\bXi}[\bG\bmu_k]\\
\bV_k\!=\!\frac{1}{d}\mathbb{E}_{\bXi}\!\!\left[\left(\bG\odot \!\left(\bhQ_k\otimes\bSigma_k\right)^{-\frac{1}{2}}\!\!\!\odot (\bI_K\otimes\bSigma_k) \!\right)\bXi_k^\top\right]
\end{cases} \!
\begin{cases}
\bhQ_k\!=\alpha \rho_k\mathbb E_{\bxi}\left[\bff_k\bff_k^\top\right]\\
\bhV_k\!=-\alpha \rho_k\bQ_k^{-\frac{1}{2}}\mathbb E_{\bxi}\!\left[\bff_k\bxi^\top\right]\\
\bhM_k\!=\alpha \rho_k\mathbb E_{\bxi}\left[\bff_k\right]
\end{cases}
\end{equation}
where $\bff_k\equiv \bV_k^{-1}(\bh_k-\bomega_k)$, and the vector $\bbb^\star$ is such that $\sum_k \rho_k\mathbb E_{\bxi}\left[\bV_k\bff_k\right]=\mathbf 0$ holds.
\end{theorem1}
The purpose of this statement is to have an asymptotically exact description of the distribution of the estimator, where the dimensions going to infinity are effectively summarized as averages over simple, independent distributions. Those distributions are parametrised by the set of finite-size parameters $(\bQ_{k},\bM_{k},\bV_{k},\bhQ_{k},\bhM_{k},\bhV_{k})_{k\in[K]}$ that can be exactly evaluated and have a clear interpretation. Indeed, the parameters $(\bM_{k},\bhM_{k})$ and $(\bQ_{k},\bhQ_{k})$ respectively represent means and covariances of multivariate Gaussians (combined with the original $\bmu_{k},\bSigma_{k}$), and the $(\bV_{k},\bhV_{k})$ parametrise the deformations that should be applied to these Gaussians to obtain the distribution of $\bW^{\star},\boldsymbol{Z}^{\star}$. The distribution is characterized in a weak sense with concentration of pseudo-Lipschitz (i.e., sufficiently regular) functions, whose definition is reminded in the Appendix \ref{sec:app:proof}. From this result one can work out a number of properties of the weights $\bW^{\star}$, e.g., training and generalisation error, but also hypothesis tests as done in \cite{celentano2020lasso} for the LASSO. Due to the generality of the statement, no direct simplification is possible. However, we will see that in certain specific cases all quantities can be greatly simplified. This is notably the case for diagonal covariance matrices and separable estimators and observables $\phi_{1},\phi_{2}$, where the sums over high-dimensional Gaussians concentrate explicitly to one-dimensional expectations. For instance the results of \cite{Thrampoulidis2020, Mignacco2020} can be recovered as special cases of the present work. Theorem \ref{the:1} then allows to obtain the asymptotic values of the generalisation error, of the training loss and of the training error. Their explicit expression is given in the following Theorem.

\begin{theorem2}[generalisation error and training loss]
\label{th:2}
In the hypotheses of Theorem \ref{the:1}, the training loss, the training error and the generalisation error are given by
\begin{equation}\label{eq:errors}
\epsilon_\ell=\sum_{k=1}^K\rho_k\mathbb E_{\bxi}[\ell(\be_k,\bh_k)],\qquad 
\epsilon_t=1-\sum_{k=1}^K\rho_k\mathbb E_{\bxi}\left[\hat y_k(\bh_k)\right],
\qquad \epsilon_g=1-\sum_{k=1}^K\rho_k\mathbb E_{\bxi}\left[\hat y_k(\bomega_k)\right].
\end{equation}
\end{theorem2}
\paragraph{The case of ridge regularisation and diagonal $\bSigma_k$} The general formulas given above can be remarkably simplified under some assumptions about the choice of the regularisation and about the structure of the covariance matrices $\bSigma_k$. This is the case for instance for the ridge regularisation $r(\bW)=\|\bW\|^2_{\rm F}/2$ and jointly diagonalizable covariances. In this case, Theorem \ref{the:1} simplifies as follows.

\begin{corollary1}
\label{corr:1}
Under the hypotheses of Theorem~\ref{the:1}, let us further assume that a ridge regularisation is adopted, $r(\bW)=\|\bW\|^2_{\rm F}/2$, and that the covariance matrices $\bSigma_k$ have a common set of orthonormal eigenvectors $\{\bv_i\}_{i=1}^d$, so that, for each $\bSigma_k=\sum_{i=1}^d\sigma_i^k\bv_i\bv_i^\top$. Let us also introduce, in the $d\to+\infty$ limit, the joint distribution for the $K$-dimensional vectors $\bsigma=(\sigma^1,\dots,\sigma^K)$ and $\bmu=(\mu^1,\dots,\mu^K)$,
\begin{equation}
\frac{1}{d}\sum_{i=1}^d\prod_{k=1}^K\delta(\sigma^k-\sigma^k_i)\delta(\mu^k-\sqrt d\bmu_k^\top\bv_i)\xrightarrow{d\to+\infty} p(\bsigma,\bmu),
\end{equation}
Then, the first three saddle point equations in eqs.~\eqref{spgeneral} take the form
\begin{equation}\label{eq:corr1}
\begin{cases}
\bQ_k=\mathbb E_{\bsigma,\bmu}\left[\sigma^k\left(\lambda\bI_K+\sum_{\kappa=1}^K\sigma^\kappa\bhV_k\right)^{-2}\left(\sum_{\kappa\kappa'}\mu^\kappa\mu^{\kappa'}\bhM_{\kappa}\bhM^\top_{\kappa'}+\sum_{\kappa=1}^K\sigma^\kappa\bhQ_k\right)\right],\\
\bM_k=\mathbb E_{\bsigma,\bmu}\left[\mu^k\left(\lambda\bI_K+\sum_{\kappa=1}^K\sigma^\kappa\bhV_k\right)^{-1}\sum_{\kappa=1}^K\mu^\kappa\bhM_\kappa\right],\\
\bV_k=\mathbb E_{\bsigma,\bmu}\left[\sigma^k\left(\lambda\bI_K+\sum_{\kappa=1}^K\sigma^\kappa\bhV_k\right)^{-1}\right].
\end{cases}
\end{equation}
\end{corollary1}

\paragraph{Narrative of the proof} The proof is detailed in Appendix \ref{sec:app:proof}.
It overcomes problems that existing methods, notably convex Gaussian comparison inequalities \cite{Thrampoulidis2020}, have yet to be adapted to. The first main technical difficulty resides in the estimator of interest being a matrix learned with non-linear functions. This makes it impossible to decompose the problem on each row of the estimator, which must be characterized in a probabilistic sense directly as a matrix. The second main difficulty is brought by the mixture of arbitrary covariances. Intuitively, the covariances correlate the estimator with the individual clusters, and therefore the correlation function cannot be represented by a single quantity. In our proof, these points are handled using the AMP and related state-evolution techniques \cite{bolthausen2014iterative,bayati2011dynamics,bayati2011lasso,gerbelot2020asymptotic}. The main idea of the proof is to express the estimator $\bW^{\star}$ as the limit of a convergent sequence whose structure enables the decomposition of all correlations and distributions in closed form. AMP iterations can handle matrix valued variables \cite{aubin2019committee,javanmard2013state}, correlations in block-structure \cite{javanmard2013state}, non-separable functions \cite{manoel2017multi,berthier2020state} and compositions of the previous three, leaving a large choice of possibilities in their design. We thus reformulate the problem in a way that makes the interaction between the estimator and each cluster explicit, effectively introducing a block structure to the problem, and isolate the overlaps with the means $\left\{\bmu_{k}\right\}$. We then design a matrix-valued sequence that obeys the update rule of an AMP sequence, in order to benefit from its exact asymptotics, and whose fixed point condition matches the optimality condition of the ERM problem, Eq.~\eqref{ERM}. Our proof builds on the spatial coupling framework in the AMP literature \cite{krzakala2012statistical,javanmard2013state}, which shows that the effect of random matrices defined with non-identically distributed blocks can be embedded in an AMP iteration while explicitly keeping the effect of each block. The non-linearities are then obtained by a block decomposition of the proximal operators defined on sets of matrices, acting on different variables of the AMP sequence and representing the effect of each cluster. The convergence analysis is made possible by the convexity of the problem: the sequence is defined with proximal operators of convex functions which are roughly contractions, and results in converging sequences when combined with the high-dimensional properties of the iteration. It is also interesting to note that the replica method, although heuristic, yet again gives the correct prediction without any hindering from the aforementioned main difficulties, as detailed in Appendix~\ref{sec:app:replicas}.

\paragraph{Universality} AMP-type proofs are amenable to both finite sample size analysis and universality proofs. For instance, in \cite{rush2018finite} it is shown that simpler instances of AMP for the LASSO exhibit exponential concentration in the system size, and the i.i.d.~Gaussian assumption can be relaxed to independently sampled sub-Gaussian distributions, as shown in \cite{bayati2015universality,chen2021universality}. Although these results do not formally encompass our case, their proof method contains most of the required technicalities, and it should be possible to prove similar results in the present setting. Indeed, recent results in \cite{Couillet2020} suggest that the formula of Theorem \ref{the:1} and \ref{th:2} should be universal for all mixtures of concentrated distribution in high-dimension, not only Gaussian ones. As we discuss Sec.~\ref{sec:examples:mnist}, even real data learning curves are empirically found to follow the behavior of the mixture of Gaussians.

\section{Results on synthetic and real datasets}
\label{sec:examples}
In this section we exemplify how Theorem \ref{the:1} can be employed to compute quantities of interest in different empirical risk minimisation tasks in high-dimensions. In all cases discussed below, eqs.~\eqref{spgeneral} have been solved numerically. A repository with a polished version of the code we used to solve the equations is available on GitHub \cite{github} (see also Appendix \ref{app:numint}).

\subsection{Correlated sparse mixtures}
\label{sec:examples:sparse}
\begin{figure}[ht]
    \centering
    \includegraphics[width=0.45\textwidth]{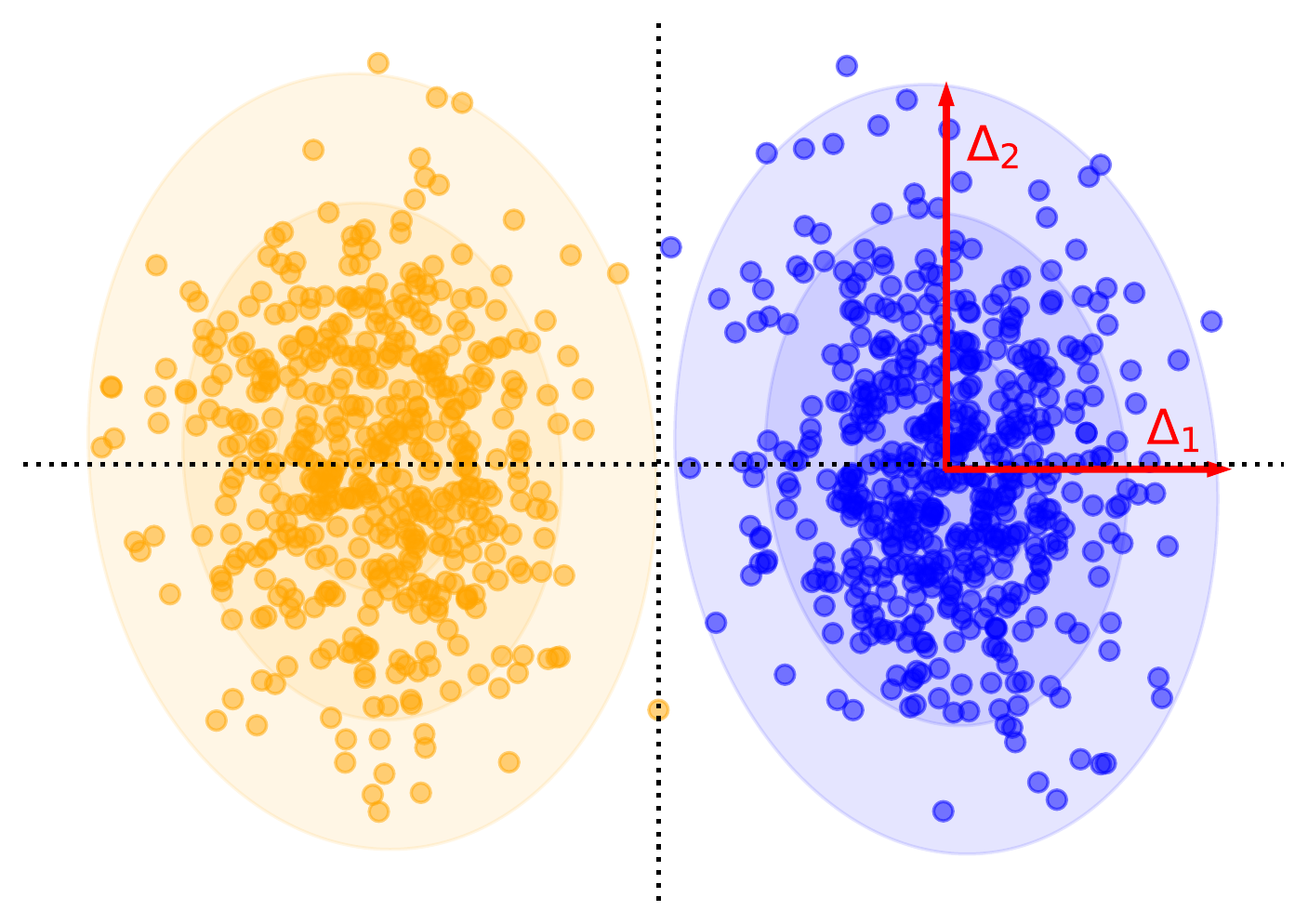}
    \includegraphics[width=0.45\textwidth]{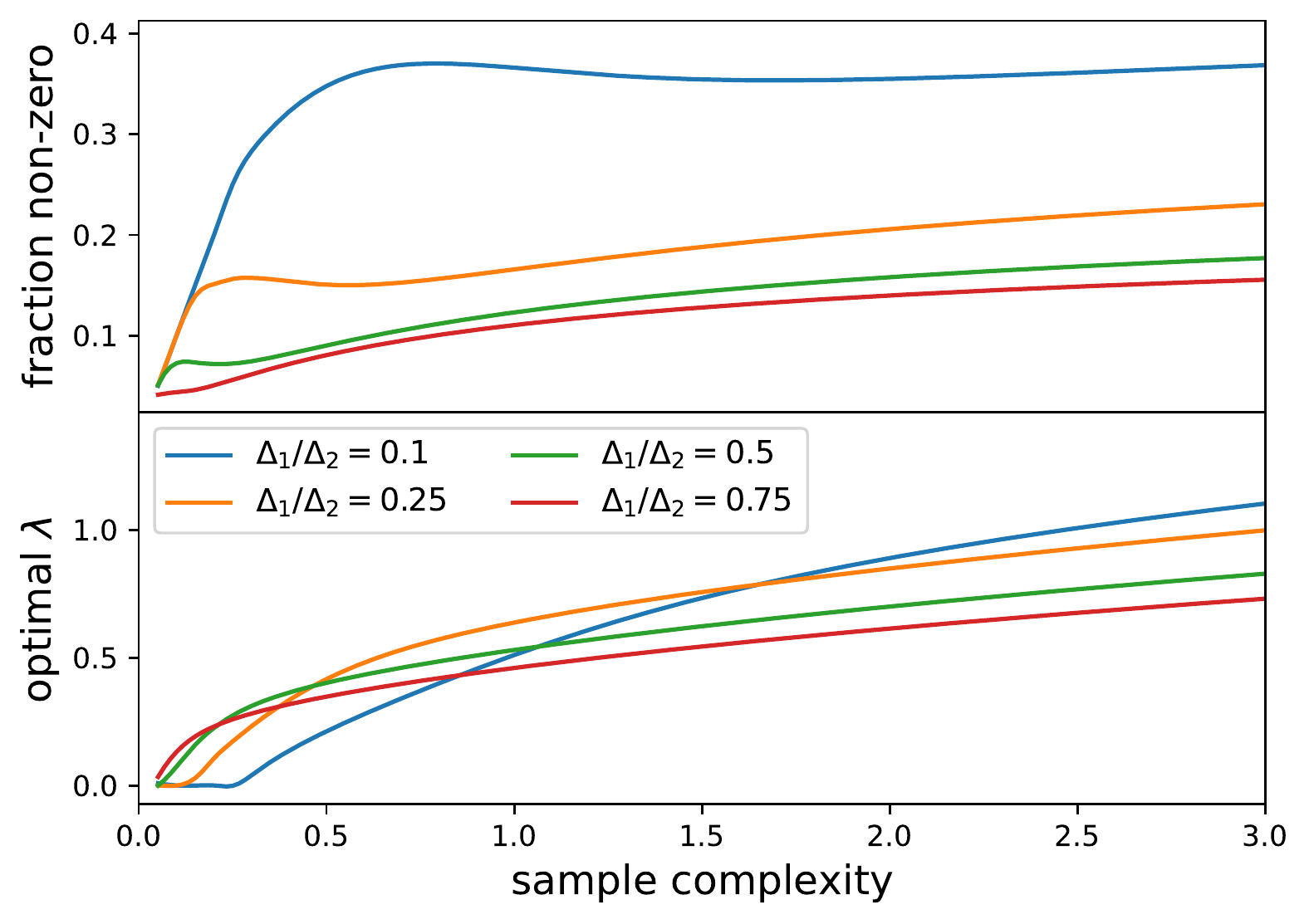}
    \caption{(\textbf{Left}) Two-dimensional projection of the Gaussian mixture introduced via Eq.~\eqref{eq:sparse} in which the sparse directions of the means are correlated with the weak/strong directions in the data. (\textbf{Right}) Fraction of non-zero elements of the lasso estimator (\textit{top}) and optimal regularisation strength (\textit{bottom}) as a function of the sample complexity $\alpha=\nicefrac{n}{d}$ for different anisotropy ratios and fixed sparsity $\rho=0.1$. Note that for $\Delta_1/\Delta_2 \lesssim 1$ and for low $\alpha$ the optimal error is achieved for vanishing regularisation, which corresponds to the \emph{basis pursuit} algorithm \cite{Donoho2018}.}
    \label{fig:eg1:setup}
\end{figure}
\begin{figure}
    \centering
    \includegraphics[width=\textwidth]{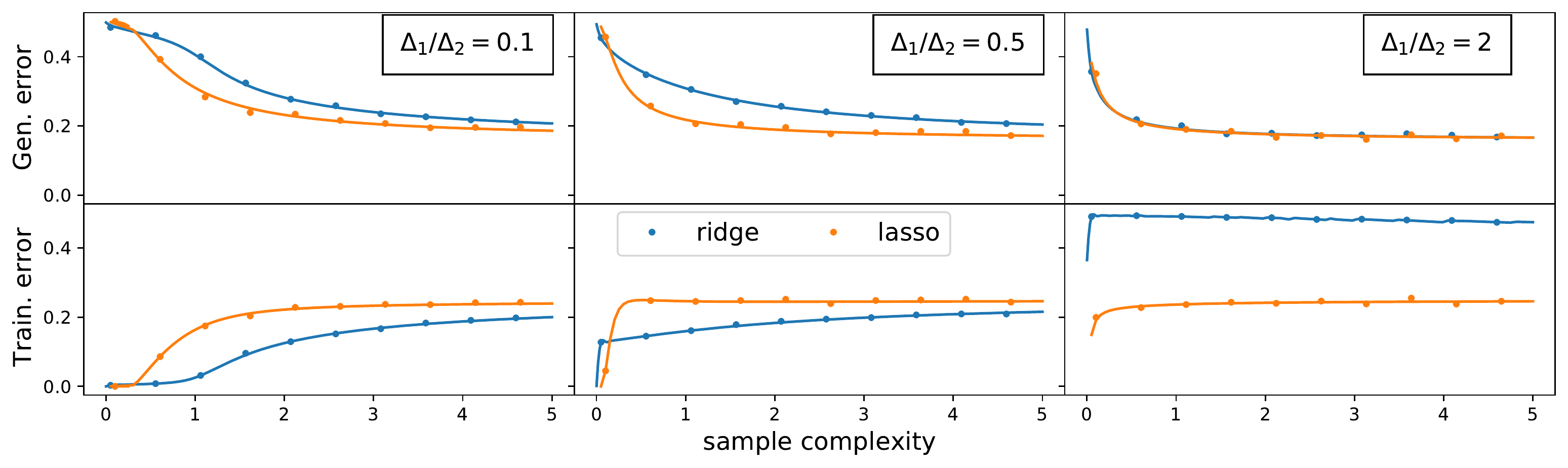}
    \caption{Learning curves for the sparse mixture model defined via Eq.~\eqref{eq:sparse} at fixed sparsity $\rho=0.1$, comparing the performance of the ridge (blue) and the lasso (orange) estimators at optimal regularisation strength $\lambda^{*}$ and for different anisotropy ratio $\Delta_{1}/\Delta_2$ (here $\Delta_1 =0.1$ and we vary $\Delta_{2}$). Full lines denote the theoretical prediction, and dots denote finite instance simulations with $d=1000$ using the \texttt{ElasticNet} module in the \texttt{Scikit-learn} package \cite{pedregosa2011scikit}. Above a certain sample complexity $\alpha$, we can identify two regimes: a) a $\Delta_1/\Delta_2 \lesssim 1$ regime in which the $\ell_1$ penalty improves significantly over $\ell_2$; b) a  $\Delta_1/\Delta_2 \gtrsim 1$ regime in which the performance is similar. Interestingly, even though the generalisation error of lasso is considerably better in a), the training loss (i.e. the mse on the labels) is higher, \& vice-versa in b).}
    \label{fig:eg1:lasso}
        \vspace{-0.3cm}
\end{figure}

As a first example, consider a binary classification problem in which the most relevant features live in a subspace of $\mathbb{R}^{d}$, and can be either weaker or stronger with respect to the irrelevant features. This problem can be modelled with a Gaussian mixture model with sparse means, and where the strong/weak directions of the covariance matrix are correlated with the non-zero components of the means. Mathematically, we consider a data set with $n$ independent samples $(\bx^{\nu}, y^{\nu})\in\mathbb{R}^{d}\times \{-1,1\}$ drawn from a Gaussian mixture $\bx^{\nu}\sim\mathcal{N}(y^{\nu}\bmu, \bSigma)$ with diagonal covariance $\Sigma_{ij} = \sigma_{i}\delta_{ij}$ which is correlated with the sparse means:
\begin{equation}
    P(\bmu, \bsigma) = \prod\limits_{i=1}^{d}\left\{\rho \mathcal{N}(\mu_{i}|0,1) \delta_{\sigma_{i}, \Delta_{1}} + (1-\rho) \delta_{\mu_{i},0} \delta_{\sigma_{i}, \Delta_{2}}\right\}\label{eq:sparse}
\end{equation}
\noindent where $\rho>0$ is the fraction of non-zero entries in $\bmu$. This model is closely related to the rare/weak features model introduced by Donoho and Jin in \cite{Donoho2008}. Indeed, in the case $\Delta_1 = \Delta_2 \equiv \Delta$ the signal-to-noise ratio of the model is proportional to $\rho/\sqrt{\Delta}$, with $\rho$ and $\Delta^{-1/2}$ playing the roles of the parameters $\epsilon$ and $\mu_{0}$ setting the "rareness" and "strength" of the features in \cite{Donoho2008}. 

The formulas given in Theorem \ref{the:1} simplify considerably for this model (see Appendix \ref{sec:app:simpli} for details), and therefore can be readily used to characterise the learning performance of different losses and penalties. For instance, one fundamental question we can address is when learning a sparse solution with the $\ell_1$ regularization is advantageous over the usual $\ell_2$. Figure \ref{fig:eg1:lasso} compares the learning curves computed from Theorem \ref{the:1} for the lasso and ridge estimators, with optimal regularisation strength $\lambda^{\star}(\alpha) =\text{argmin}~\epsilon_{g}(\alpha, \lambda)$ at fixed sparsity $\rho=0.1$. We can see that lasso performs considerably better than ridge in the regime where $\Delta_1 / \Delta_2 \lesssim 1$, while it achieves a similar performance when $\Delta_1 / \Delta_2  \gtrsim 1$. This is quite intuitive: the sparse directions are uninformative, and therefore learning the relevant features is better when they are stronger. Figure \ref{fig:eg1:setup} (right) shows how the sparsity of the learned estimator $\bW^{\star}$ and the optimal regularisation $\lambda^{\star}$ depends on the sample complexity $\alpha = n/d$. Interestingly, for $\Delta_1/\Delta_2 =0.1$ or lower there is a region of small $\alpha$ in which basis pursuit ($\lambda = 0^{+}$) \cite{Donoho2018} is optimal, and the sparsity of the estimator has a curious non-monotonic behaviour with $\alpha$.

\subsection{Separability transition for the cross-entropy loss}\label{sec:examples:multi}

We now  consider the problem of classifying points of $K$ Gaussian clusters using a cross-entropy loss
\begin{equation}
\ell(\by,\bx)=-\sum_{k=1}^K y_k\ln\frac{e^{x_k}}{\sum_{\kappa=1}^{K} e^{x_\kappa}}.
\end{equation}
Using the results of Theorem \ref{th:2}, we estimate the dependence of the generalisation error $\epsilon_g$ on the sample complexity $\alpha$ and on the regularisation $\lambda$. We assume Gaussian means $\bmu_k\sim\mathcal N(\mathbf 0,\bI_d/d)$ and diagonal covariances $\bSigma_k\equiv\bSigma=\Delta\bI_d$. Finally, we adopt a ridge penalty, $r(\bW)\equiv\|\bW\|^2_{\rm F}/2$, and we focus on the case of balanced clusters, i.e., $\rho_k=\nicefrac{1}{K}$ for the sake of simplicity.

\begin{figure}
\centering
    \includegraphics[scale=0.41]{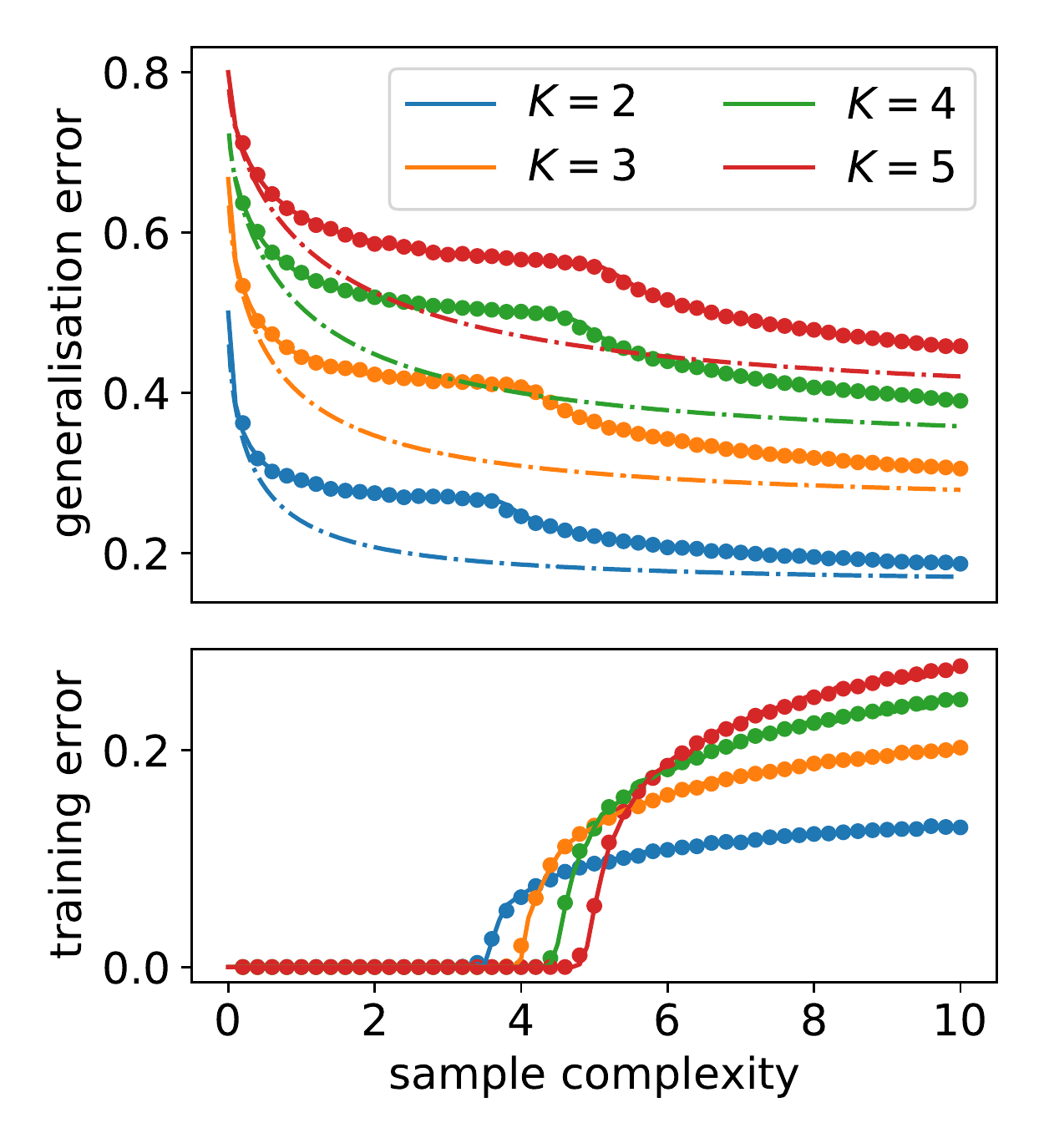}
    \includegraphics[scale=0.41]{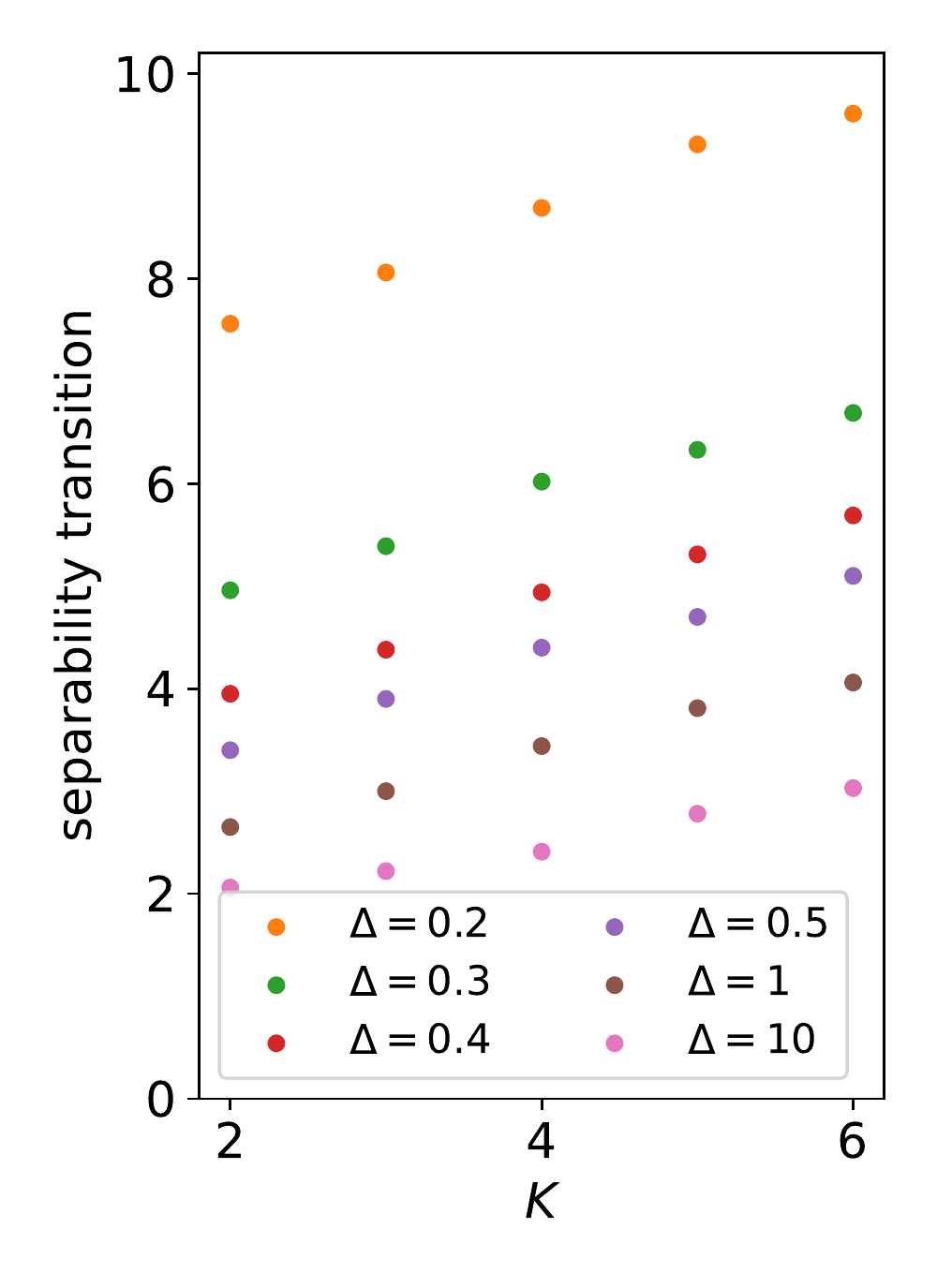}
    \includegraphics[scale=0.41]{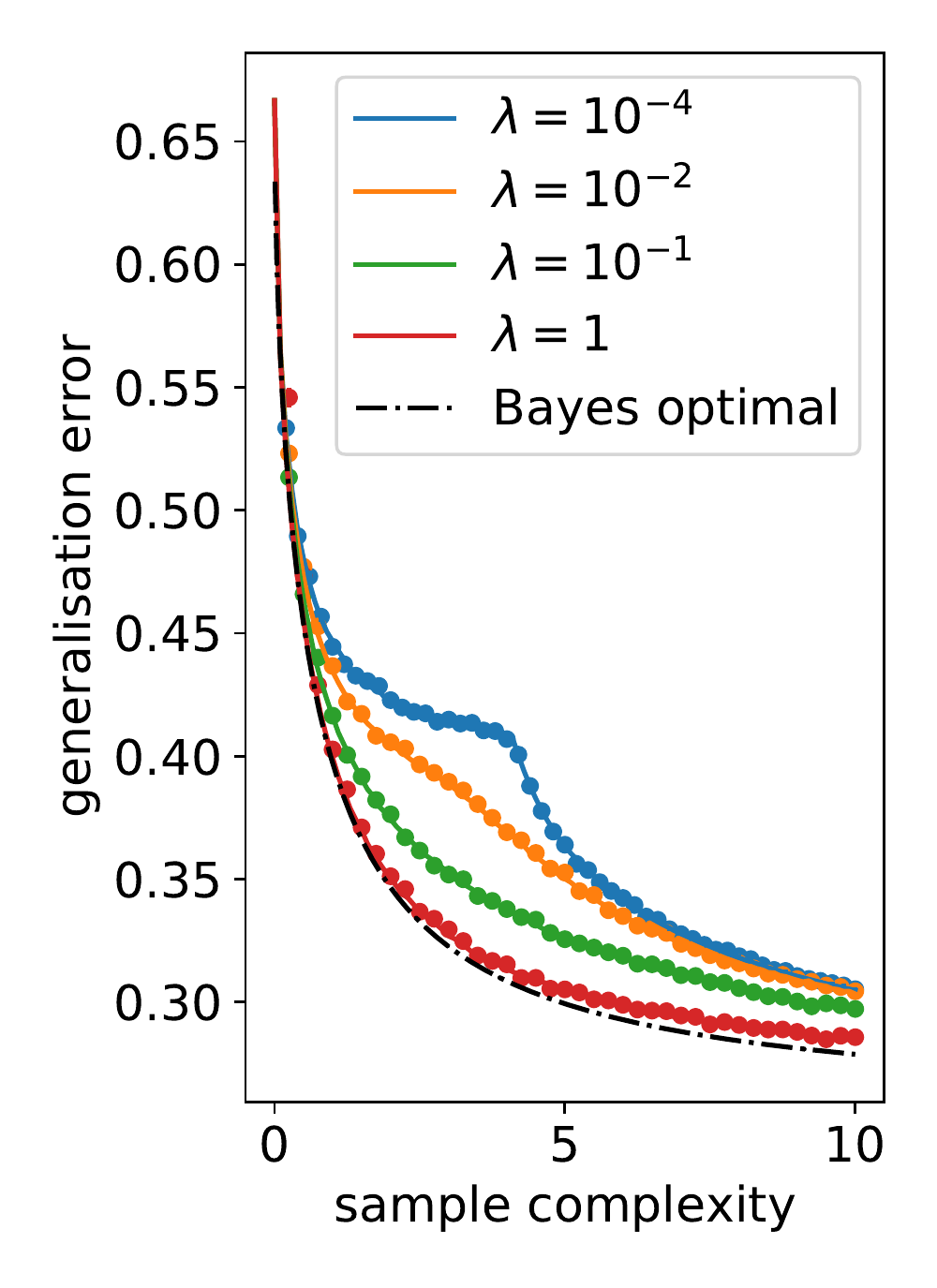}
    \caption{Classification of $K$ Gaussian clusters in $d$ dimensions, having Gaussian means and $\bSigma_k\equiv \bSigma=\Delta\bI_d$ with $\Delta=\nicefrac{1}{2}$. In all presented cases, a quadratic regularisation has been adopted. Numerical experiments have been performed using $d=10^3$. (\textbf{Left}) Generalisation error $\epsilon_g$ (\textit{top}) and training error $\epsilon_t$ (\textit{bottom}) as function of $\alpha$ at $\lambda=10^{-4}$. Theoretical predictions (full lines) are compared with the results of numerical experiments (dots). Dash-dotted lines of the corresponding color represent, for comparison, the Bayes-optimal error. The results of numerical experiments are in agreement with the theoretical predictions in all cases. (\textbf{Center}) Separability transition $\alpha_K^\star$ as a function of $K$ in the same setting for different values of $\Delta$. (\textbf{Right}) Dependence of the generalisation error on the regularization $\lambda$ for $K=3$ and $\Delta=\nicefrac{1}{2}$ in the balanced case, $\rho_k=\nicefrac{1}{K}$. 
    }
    \label{fig:manyk}
        \vspace{-0.3cm}
\end{figure}

\paragraph{Separability transition} In Fig.~\ref{fig:manyk} (left top) we plot the generalisation error $\epsilon_g$ as function of $\alpha$ for $2\leq K\leq 5$ and $\lambda=10^{-4}$. The smooth curve is obtained solving the fixed point equations in Theorem~\ref{the:1} and plugging the results in the formulas in Theorem \ref{th:2}. The results of numerical experiments are obtained averaging over $10^2$ instances of the problem solved using the \texttt{LogisticRegression} module in the \texttt{Scikit-learn} package \cite{pedregosa2011scikit}. An excellent agreement is observed. For each pair $(K,\Delta)$ and for vanishing regularisation $\lambda\to 0^{+}$ we observe a double-descent-like behaviour in the generalisation error. Indeed, the cusp $\alpha^\star_K(\Delta)$ in the generalisation error corresponds to the point in which the cross-entropy estimator ceases to perfectly interpolate the data, revealing the existence of a separability transition of the type discussed in \cite{Candes2020} for Gaussian i.i.d.~data. As stressed therein, a phase of perfect separability of the data points corresponds to a regime in which the maximum-likelihood estimate does not exist with probability one. This is visible, in the same figure (left bottom), from the training error $\epsilon_t$ that is identically zero for $\alpha<\alpha^\star_K$, and strictly positive otherwise. Our result extends the observations in \cite{Thrampoulidis2019, Mignacco2020}, where an analytic expression for $\alpha^\star_2$ has been given in the case of for $K=2$, $\bmu_1=-\bmu_2$ Gaussian vector, generalising the classical result of Cover \cite{Cover1965}. The separability transition point $\alpha^\star_K$ decreases with $\Delta$ and increases with $K$, showing that for larger $K$ it is easier to separate the different clusters: this intuitively follows from the fact that, at fixed $\alpha$ and $\Delta$, each cluster is given by $\nicefrac{\alpha d}{K}$ points, i.e., fewer for increasing $K$ and therefore easier to classify, see Fig.~\ref{fig:manyk} (center). 
\begin{wrapfigure}{r}{0.47\textwidth}
    \centering
    \includegraphics[scale=0.41]{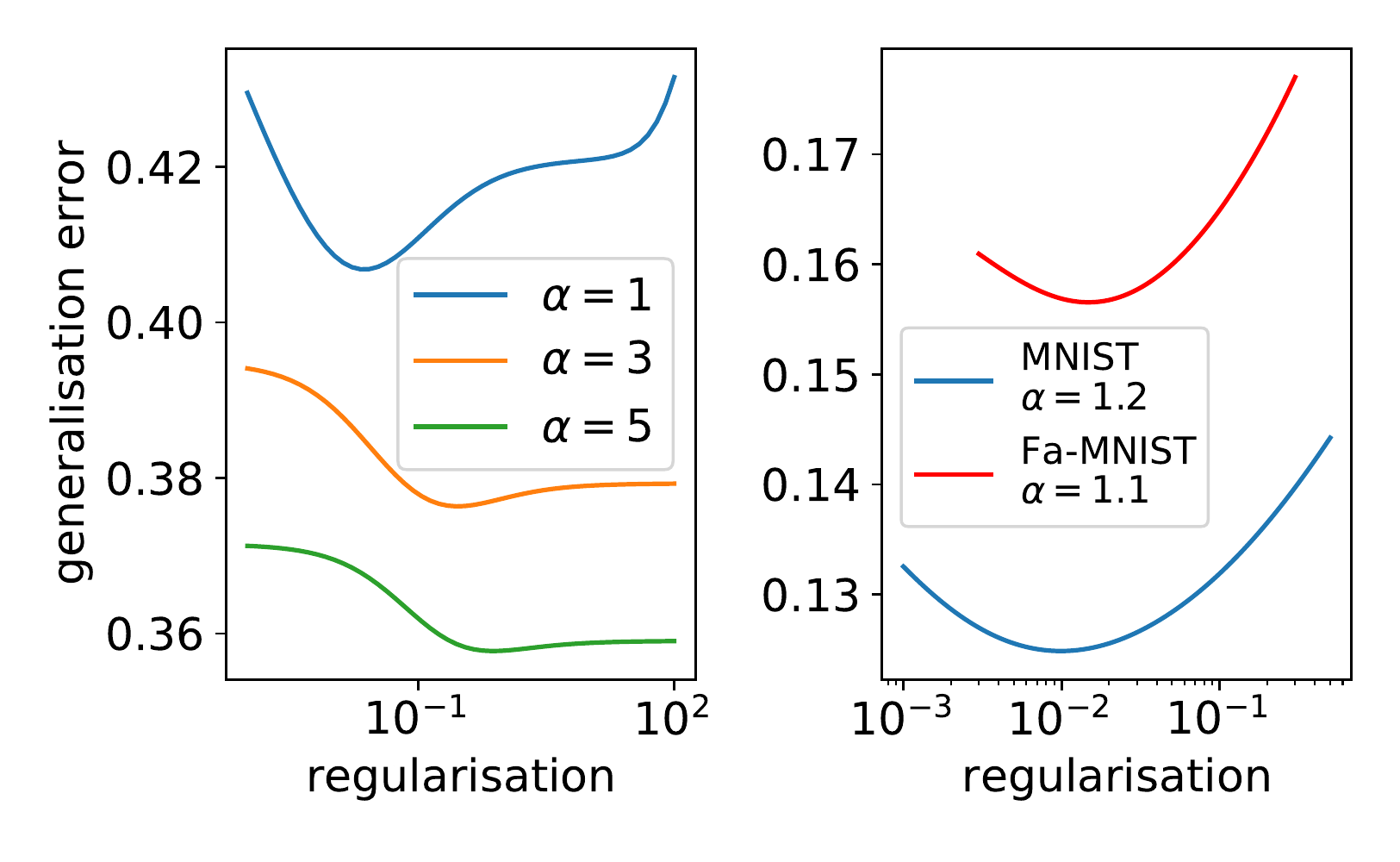}
    \caption{(\textbf{Left.}) Generalisation error obtained using ridge regression in the case of two balanced Gaussian clusters having $\bSigma_1=\frac{1}{10}\bI_d$ and $\bSigma_2=\frac{1}{100}\bI_d$ as function of $\lambda$ for different values of the sample complexity $\alpha$. (\textbf{Right}) Generalisation error $\epsilon_g$ as a function of $\lambda$ at fixed $\alpha$ in the binary classification of MNIST and in the FashionMNIST via logistic regression (see Sec.~\ref{sec:examples:mnist} for details).
    }\vspace{-0.4cm}
    \label{fig:EvsLambda}
\end{wrapfigure} 

\paragraph{The role of regularisation} In Fig.~\ref{fig:manyk} (right) we compare the performances of the cross-entropy loss with respect to the Bayes-optimal error (detailed in Appendix \ref{sec:app:bayes}) for different strength $\lambda$ of the regularisation, assuming all identical diagonal covariances $\bSigma_k\equiv\bSigma=\Delta\bI_d$. In the case of balanced clusters (i.e., $\rho_k=\nicefrac{1}{K}$ for all $k$) it is observed that the generalisation error approaches the Bayes-optimal error for $\lambda\to+\infty$. The same phenomenology has been observed in \cite{Dobriban2018, Mignacco2020} in the $K=2$ case with opposite means and generic loss, and in \cite{Thrampoulidis2020} for $K>2$ for the square loss. 
Using the concentration results of Section~\ref{sec:main}, we investigated the robustness of this result in the case of balanced clusters but with different covariances and various losses. First, we considered two opposite \textit{balanced} clusters with $\bSigma_1=\Delta_1\bI_d$ and $\bSigma_2=\Delta_2\bI_2$, $\Delta_1\neq\Delta_2$, and we estimated the generalisation error at fixed sample complexity as function of $\lambda\in[10^{-4},10^2]$ using ridge regression. As shown in Fig.~\ref{fig:EvsLambda} (left), the regularisation strength optimising the error is finite, and in particular depends on the sample complexity. This situation is closer to what is observed in real problems with balanced data analysed using logistic regression. Indeed, using the covariances from real data sets such as MNIST or Fashion-MNIST yields a similar behaviour, see Fig.~\ref{fig:EvsLambda} (right), with an optimal $\lambda$ that is found to be finite.

\subsection{Binary classification with real data}
\label{sec:examples:mnist}

A recent line of works has reported that the asymptotic learning curves of simple regression tasks on real data sets can be well approximated by a surrogate Gaussian model matching the first two moments of the data \cite{bordelon2020, jacot2020kernel, Loureiro2021}. However, this analysis was fundamentally restricted to least-squares regression, and considerable deviation from the Gaussian model was observed for classification tasks~\cite{Loureiro2021}. 
Authors of \cite{Couillet2020} have shown that realistic-looking data from trained generative adversarial networks behave like Gaussian mixtures. Here, we pursue these observations and investigate whether Theorem \ref{th:2} can be used to capture the learning curves of classification tasks on two popular data sets: MNIST \cite{lecunMNIST} and Fashion-MNIST \cite{xiao2017fashion}. Our goal is to compare the performances of some classification tasks on them with the predictions provided by the theory for the Gaussian mixture model. 
\begin{figure}
    \centering
    \includegraphics[scale=0.41]{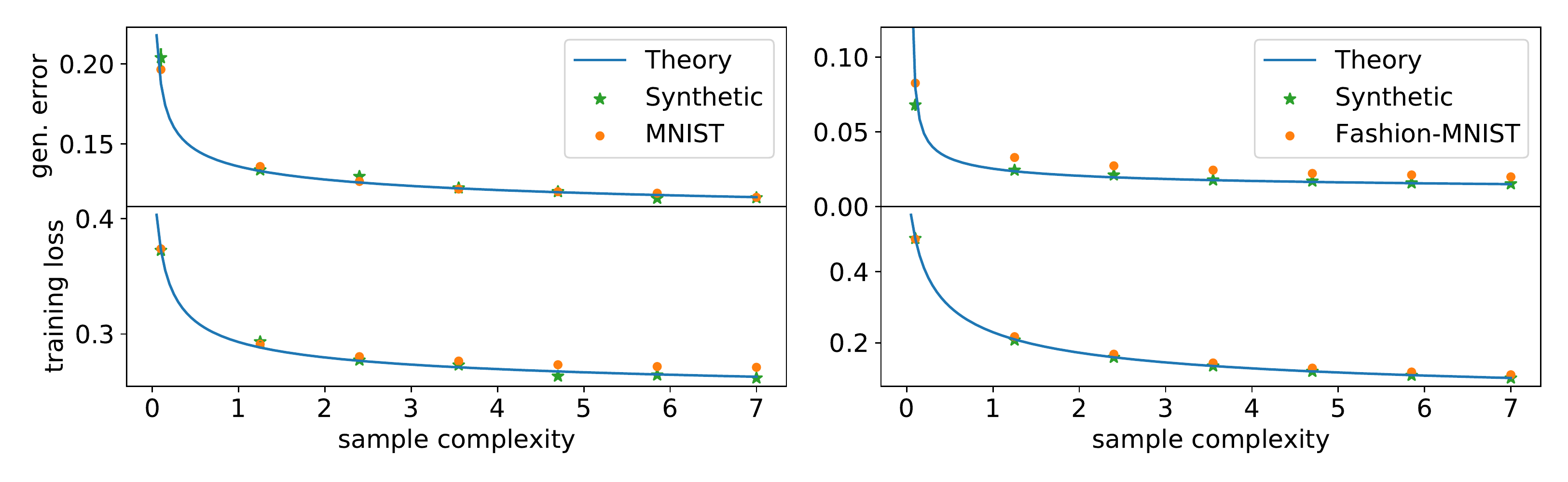}
    \caption{Generalisation error and training loss for the binary classification using the logistic loss on MNIST with $\lambda=0.05$ (\textbf{left}) and on Fashion-MNIST with $\lambda=1$ (\textbf{right}). The results are compared with synthetic data produced from the corresponding Gaussian mixture, and the theoretical prediction.}
    \label{fig:mnist}
        \vspace{-0.2cm}
\end{figure}

Both data sets consist of $n_{\rm tot}=7\times 10^4$ images $\hat\bx^\mu\in\mathbb R^d$, $d=784$. Each image $\hat\bx^\mu$ is associated to a label $\hat y^\mu=\{0,1,\dots,9\}$ specifying the type of represented digit (in the case of MNIST) or item (in the case of Fashion-MNIST). In both cases, we divided the database into two balanced classes (even vs odd digits for MNIST, clothes vs accessories for Fashion-MNIST), relabelling the elements $\hat\bx^\mu$ with $y^\mu\in\{-1,1\}$ depending on their class, and we selected $n<n_{\rm tot}$ elements to perform the training, leaving the others for the test of the performances. We adopted a logistic loss with $\ell_2$ regularisation. First, we performed logistic regression on the training real data set, then we tested the learned estimators on the remaining $n_{\rm tot}-n$ images. At the same time, for each class $k$ of the training set, we empirically estimated the corresponding mean $\bmu_k\in\mathbb R^{d}$ and covariance matrix $\bSigma_k\in\mathbb R^{d\times d}$. We then assumed that the classification problem on the real database corresponds to a Gaussian mixture model of $K=2$ clusters with means $\{\bmu_k\}_{k\in[2]}$ and covariances $\{\bSigma_k\}_{k\in[2]}$. Under this assumption, we computed the generalisation error and the training loss predicted by the theory inserting the empirical means and covariances in our general formulas. The results are given in Fig.~\ref{fig:mnist}, showing a good agreement between the theoretical prediction and the results obtained on MNIST and Fashion-MNIST. In Fig.~\ref{fig:mnist} we also plot, as reference, the results of a classification task performed on synthetic data, obtained generating a genuine Gaussian mixture with the means and covariances of the real data set.

\begin{figure}[ht]
    \centering
    \includegraphics[scale=0.41]{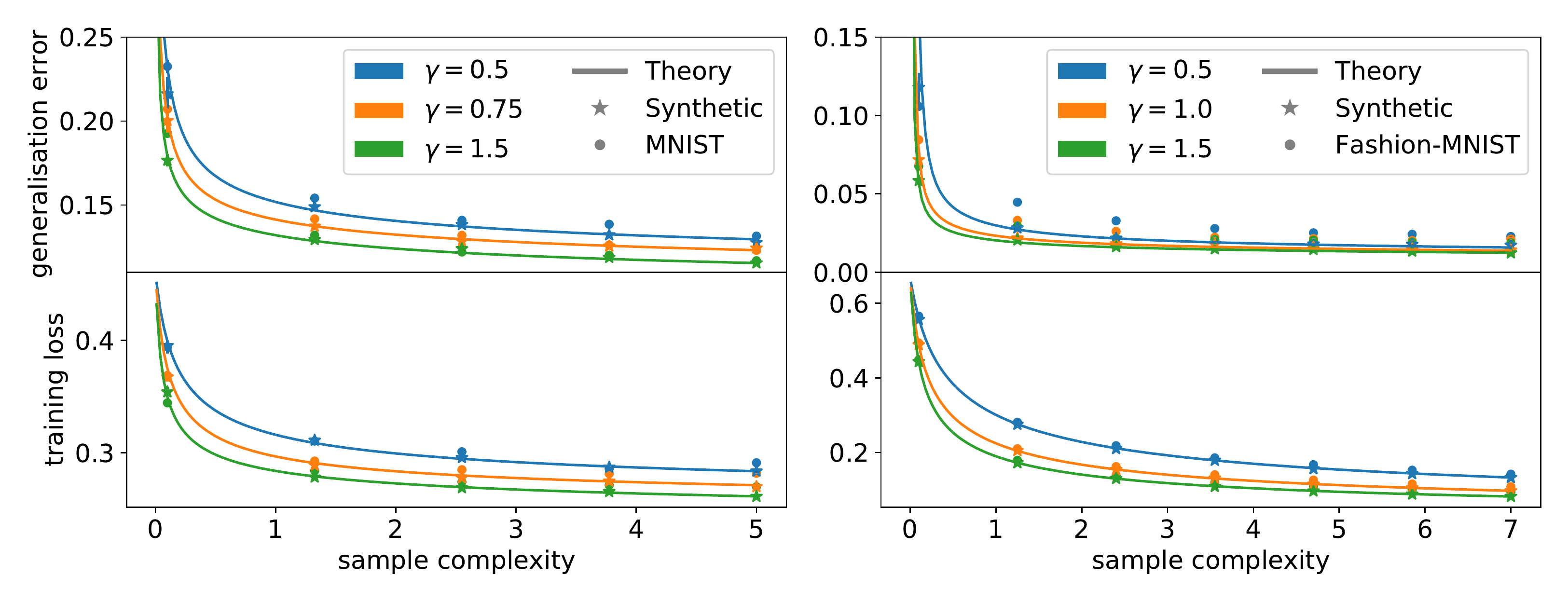}
    \caption{Generalisation error and training loss for the binary classification using the logistic on MNIST at $\lambda=0.05$ (\textbf{left}) and on Fashion-MNIST at $\lambda=1$ (\textbf{right}) in the random feature setting, for different values of $\gamma$, ratio between the number of parameters and the dimensionality of the data. The results are compared with synthetic data produced with the same $\gamma$, and the theoretical prediction.}
    \label{fig:mnistrf}
    \vspace{-0.2cm}
\end{figure}

Interestingly, this construction can also be used to analyse the learning curves of classification problems with non-linear feature maps \cite{Loureiro2021}, e.g. random features \cite{rahimi2007random}. In this case, we first apply to our data set a feature map $\bx^\mu=\mathrm{erf}(\bF\hat\bx^\mu)$, where $\bF\in\mathbb R^{p\times d}$ has i.i.d.~Gaussian entries and the $\mathrm{erf}$ function is applied component wise. The classification task is then performed on the new data set $\{(\bx^\nu, y^\nu)\}_{\nu\in[n]}$, the new data points $\bx^\nu$ living in a $p$-dimensional space. We denote $\gamma=\nicefrac{p}{d}$. We repeat the analysis described above in this new setting. Our results are in Fig.~\ref{fig:mnistrf} for different values of $\gamma$. Once again, the generalisation error and the training loss are shown to be in a good agreement with both the theoretical prediction and the synthetic data sets obtained plugging in our formulas the real data means and the real data covariance matrices. 
\section*{Acknowledgements}
We thank Rapha\"{e}l Berthier and Francesca Mignacco for discussions. We acknowledge funding from the ERC under the European Union’s Horizon 2020 Research and Innovation Program Grant Agreement 714608- SMiLe, and from the French National Research Agency grants ANR-17-CE23-0023-01 PAIL. GS is grateful to EPFL for its generous hospitality during the finalization of the project.

\appendix
\part*{Appendix}
\section{Proof}
\label{sec:app:proof}
This appendix presents the proof of the main technical result, Theorem \ref{the:1}. Throughout the whole proof, we assume that the set of conditions from Sec.~\ref{set:assump} is verified.

\subsection{Required background}
\label{sec:req_back}
In this Section, we give an overview of the main concepts and tools on approximate message passing algorithms which will be required for the proof.\vspace{0.1cm}

We start with some definitions that commonly appear in the approximate message-passing literature, see e.g. \cite{bayati2011dynamics,javanmard2013state,berthier2020state}.
The main regularity class of functions we will use is that of pseudo-Lipschitz functions, which roughly amounts to functions with polynomially bounded first derivatives. We include the required scaling w.r.t.~the dimensions in the definition for convenience.
\begin{definition}[Pseudo-Lipschitz function]
For $k,K \in \mathbb{N}^{*}$ and any $n,m \in \mathbb{N}^{*}$, a function $\bphi \colon \mathbb{R}^{n \times K} \to \mathbb{R}^{m \times K}$ is called a \emph{pseudo-Lipschitz of order $k$} if there exists a constant $L(k,K)$ such that for any $\bx,\by \in \mathbb{R}^{n \times K}$, 
\begin{equation}
    \frac{\norm{\bphi(\bx)-\bphi(\by)}_{\rm F}}{\sqrt{m}} \leqslant L(k,K) \left(1+\left(\frac{\norm{\bx}_{\rm F}}{\sqrt{n}}\right)^{k-1}+\left(\frac{\norm{\by}_{\rm F}}{\sqrt{n}}\right)^{k-1}\right)\frac{\norm{\bx-\by}_{\rm F}}{\sqrt{n}}
\end{equation}
where $\norm{\bullet }_{\rm F}$ denotes the Frobenius norm.
Since $K$ will be kept finite, it can be absorbed in any of the constants.
\end{definition}
For example, the function $f:\mathbb{R}^{n}\to \mathbb{R}, \bx \mapsto \frac{1}{n}\norm{\bx}_{2}^{2}$ is pseudo-Lipshitz of order 2.

\paragraph{Moreau envelopes and Bregman proximal operators} In our proof, we will also frequently use the notions of Moreau envelopes and proximal operators, see e.g. \cite{parikh2014proximal,bauschke2011convex}. These elements of convex analysis are often encountered in recent works on high-dimensional asymptotics of convex problems, and more detailed analysis of their properties can be found for example in \cite{thrampoulidis2018precise,Loureiro2021}. For the sake of brevity, we will only sketch the main properties of such mathematical objects, referring to the cited literature for further details. In this proof, we will mainly use proximal operators acting on sets of real matrices endowed with their canonical scalar product. Furthermore, proximals will be defined with matrix valued parameters in the following way: for a given convex function $f \colon\mathbb{R}^{d \times K}\to\mathbb R$, a given matrix $\bX \in \mathbb{R}^{d \times K}$ and a given symmetric positive definite matrix $\bV \in \mathbb{R}^{K \times K}$ with bounded spectral norm, we will consider operators of the type
\begin{equation}
    \Argmin_{\bT \in \mathbb{R}^{d \times K}} \left\{f(\bT)+\frac{1}{2}\mathrm{tr}\left((\bT-\bX)\bV^{-1}(\bT-\bX)^{\top}\right)\right\}
\end{equation}
This operator can either be written as a standard proximal operator by factoring the matrix $\bV^{-1}$ in the arguments of the trace:
\begin{equation}
    \Prox_{f(\bullet\bV^{1/2})}(\bX\bV^{-1/2})\bV^{1/2} \in \mathbb{R}^{d \times K}
\end{equation}
or as a Bregman proximal operator \cite{bauschke2003bregman} defined with the Bregman distance induced by the strictly convex, coercive function (for positive definite $\bV$)
\begin{equation}
\label{breg_dist}
    \bX \mapsto \frac{1}{2}\mathrm{tr}(\bX\bV^{-1}\bX^\top)
\end{equation}
which justifies the use of the Bregman resolvent 
\begin{equation}
\label{bregman_res}
   \Argmin_{\bT \in \mathbb{R}^{d \times K}} \left\{f(\bT)+\frac{1}{2}\mathrm{tr}\left((\bT-\bX)\bV^{-1}(\bT-\bX)^{\top}\right)\right\}=\left(\mathrm{Id}+\partial f(\bullet)\bV\right)^{-1}(\bX)
\end{equation}
Many of the usual or similar properties to that of standard proximal operators (i.e.~firm non-expansiveness, link with Moreau/Bregman envelopes,\dots) hold for Bregman proximal operators defined with the function \eqref{breg_dist}, see e.g. \cite{bauschke2003bregman,bauschke2018regularizing}. In particular, we will be using the equivalent notion to firmly nonexpansive operators for Bregman proximity operators, called $\emph{D-firm}$ operators. Consider the Bregman proximal defined with a differentiable, strictly convex, coercive function $g : \mathcal{X} \to \mathbb{R}$, where $\mathcal{X}$ is a given input Hilbert space. Let $T$ be the associated Bregman proximal of a given convex function $f : \mathcal{X} \to \mathbb{R}$, i.e., for any $\mathbf{x} \in \mathcal{X}$
\begin{equation}
    T(\mathbf{x}) = \Argmin_{\mathbf{y} \in \mathcal{X}} \left\{f(\mathbf{x})+D_{g}(\mathbf{x},\mathbf{y})\right\}
\end{equation}
Then $T$ is \emph{D-firm}, meaning it verifies
\begin{equation}
\label{eq:D-firm}
    \langle T\bx-T\by, \nabla g(T\bx)-\nabla g(T\by) \rangle \leqslant  \langle T\bx-T\by, \nabla g(\bx)-\nabla g(\by) \rangle
\end{equation}
for any $\mathbf{x},\mathbf{y}$ in $\mathcal{X}$.
\paragraph{Gaussian concentration}
Gaussian concentration properties are at the root of this proof. Such properties are reviewed in more detail, for example, in \cite{berthier2020state,Loureiro2021}.

\paragraph{Notations}
For any set of matrices $\{\bA_{k}\in \mathbb{R}^{n_{k}\times d_{k}}\}_{k\in[K]}$ we will use the following notation:
\begin{equation}
	\begin{bmatrix}\bA_{1}&&&\\ &\bA_{2}&(*)&\\ &(*)&\ddots& \\ &&&\bA_{K}\end{bmatrix} \equiv \left[\bA_{k}\right]_{k=1}^{K} \in \mathbb{R}^{(\sum_{k=1}^{K}n_{k})\times (\sum_{k=1}^{K}d_{k})}
\end{equation}
where the terms denoted by $(*)$ will be zero most of the time. \\
For a given function $\bphi\colon \mathbb{R}^{d \times K}\to \mathbb{R}^{d \times K}$, we write :
\begin{equation}
    \bphi(\bX) = \begin{bmatrix}
    \bphi^{1}(\bX) \\
    \vdots \\
    \bphi^{d}(\bX)\end{bmatrix} \in \mathbb{R}^{d \times K}
\end{equation}
where each $\bphi^{i} \colon \mathbb{R}^{d \times K} \to \mathbb{R}^{K}$. We then write the $K \times K$ Jacobian 
\begin{equation}
\label{eq:jacob}
    \frac{\partial \bphi^{i}}{\partial \bX_{j}}(\bX) = \begin{bmatrix}\frac{\partial \phi^{i}_1(\bX)}{\partial X_{j1}} & \cdots & \frac{\partial \phi^{i}_1(\bX)}{\partial X_{jK}} \\
    \vdots&\ddots&\vdots \\
    \frac{\partial \phi^{i}_K(\bX)}{\partial X_{j1}} & \cdots & \frac{\partial \phi^{i}_K(\bX)}{\partial X_{jK}}
    \end{bmatrix} \in \mathbb{R}^{K \times K}
\end{equation}
For a given matrix $\bQ\in \mathbb{R}^{K \times K}$, we write $\bZ \in \mathbb{R}^{n \times K} \sim \mathcal{N}(\boldsymbol 0,\bQ\otimes \bI_{n})$ to denote that the lines of $\bZ$ are sampled i.i.d.~from $\mathcal{N}(\boldsymbol 0,\bQ)$. Note that this is equivalent to saying that $\bZ = \tilde{\bZ}\bQ^{1/2}$ where $\tilde{\bZ} \in \mathbb{R}^{n \times K}$ is an i.i.d.~standard normal random matrix. The notation $\stackrel{\rm P}\simeq$ denotes convergence in probability.

\paragraph{Approximate message-passing}
Approximate message-passing algorithms are a statistical physics inspired family of iterations which can be used to solve high dimensional inference problems \cite{zdeborova2016statistical}. One of the central objects in such algorithms are the so called \emph{state evolution equations}, a low-dimensional recursion equations which allow to exactly compute the high dimensional distribution of the iterates of the sequence.
In this proof we will use a specific form of matrix-valued approximate message-passing iteration with non-separable non-linearities. In its full generality, the validity of the state evolution equations in this case is an extension of the works of \cite{javanmard2013state,berthier2020state} included in \cite{gerbelot2021graph}. Consider a sequence Gaussian matrices $\bA(n) \in \mathbb{R}^{n\times d}$ with i.i.d.~Gaussian entries, $A_{ij}(n)\sim\mathcal{N}(0,\nicefrac{1}{d})$. For each $n,d \in \mathbb{N}$, consider two sequences of pseudo-Lipschitz functions
\begin{equation}
\{\bh_{t} : \mathbb{R}^{n \times K}\to \mathbb{R}^{n \times K}\}_{t \in \mathbb{N}}\qquad \{\be_{t} : \mathbb{R}^{d \times K}\to \mathbb{R}^{d \times K}\}_{t \in \mathbb{N}}
\end{equation}
initialized on $\bu^{0} \in \mathbb{R}^{d\times K}$ in such a way that the limit
\begin{equation}
    \lim_{d \to \infty}\frac{1}{d}\norm{\be_{0}(\bu^{0})^\top\be_{0}(\bu^{0})}_{\rm F}
\end{equation}
exists and it is finite, and recursively define:
\begin{align}
\label{canon_AMP}
	&\hspace{1cm} \bu^{t+1} = \bA^{\top}\bh_{t}(\bv^{t})-\be_{t}(\bu^{t})\langle \bh_{t}'\rangle^\top \\
	&\hspace{1cm} \bv^{t} = \bA\be_{t}(\bu^{t})-\bh_{t-1}(\bv^{t-1})\langle \be_{t}'\rangle^\top
\end{align}
where the dimension of the iterates are $\bu^{t} \in \mathbb{R}^{d \times K}$ and $\bv^{t} \in \mathbb{R}^{n \times K}$. The terms in brackets are defined as:
\begin{equation}
\label{eq:mat_ons}
 \langle \bh_{t}' \rangle = \frac{1}{d}\sum_{i=1}^{n} \frac{\partial \bh_{t}^{i}}{\partial \bv_{i}}(\bv^{t})\in \mathbb{R}^{K \times K} \quad \langle \be_{t}' \rangle = \frac{1}{d}\sum_{i=1}^{d} \frac{\partial \be_{t}^{i}}{\partial \bu_{i}}(\bu^{t}) \in \mathbb{R}^{K \times K}
\end{equation}
We define now the \emph{state evolution recursion} on two sequences of matrices $\{\bQ_{r,s}\}_{s,r\geqslant0}$ and $\{\hat{\bQ}_{r,s}\}_{s,r\geqslant1}$ initialized with $\bQ_{0,0} = \lim_{d \to \infty}\frac{1}{d}\be_{0}(\bu^{0})^\top \be_{0}(\bu^{0})$:
\begin{align}
   &\bQ_{t+1,s} = \bQ_{s,t+1} = \lim_{d \to \infty} \frac{1}{d}\mathbb{E}\left[\be_{s}(\hat{\bZ}^{s})^{\top}\be_{t+1}(\hat{\bZ}^{t+1})\right] \in \mathbb{R}^{K \times K} \\
    &\hat{\bQ}_{t+1,s+1} = \hat{\bQ}_{s+1,t+1} = \lim_{d \to \infty} \frac{1}{d}\mathbb{E}\left[\bh_{s}(\bZ^{s})^{\top}\bh_{t}(\bZ^{t})\right] \in \mathbb{R}^{K \times K}
\end{align}
where $(\bZ^{0},\dots,\bZ^{t-1}) \sim \mathcal{N}(\boldsymbol 0,\{\bQ_{r,s}\}_{0\leqslant r,s \leqslant t-1} \otimes \bI_{n}),(\hat{\bZ}^{1},\dots,\hat{\bZ}^{t}) \sim \mathcal{N}(\boldsymbol 0,\{\bhQ_{r,s}\}_{1\leqslant r,s \leqslant t} \otimes \bI_{d})$. Then the following holds
\begin{theorem1}
\label{th:SE}
In the setting of the previous paragraph, for any sequence of pseudo-Lipschitz functions $\phi_{n}:(\mathbb{R}^{n \times K}\times \mathbb{R}^{d \times K})^{t} \to \mathbb{R}$, for $n,d \to +\infty$:
\begin{equation}
\phi_{n}(\bu^{0}, \bv^{0},\bu^{1},\bv^{1},\dots,\bv^{t-1},\bu^{t}) \stackrel{\rm P}\simeq \mathbb{E}\left[\phi_{n}\left(\bu^{0},\bZ^{0},\hat{\bZ}^{1},\bZ^{1},\dots,\bZ^{t-1},\hat{\bZ}^{t}\right)\right]
\end{equation}
where $(\bZ^{0},\dots,\bZ^{t-1}) \sim \mathcal{N}(\boldsymbol 0,\{\bQ_{r,s}\}_{0\leqslant r,s \leqslant t-1} \otimes \bI_{n}),(\hat{\bZ}^{1},\dots,\hat{\bZ}^{t}) \sim \mathcal{N}(\boldsymbol 0,\{\bhQ_{r,s}\}_{1\leqslant r,s \leqslant t} \otimes \bI_{n})$.
\end{theorem1}

\paragraph{Spatial coupling} As a final premise to our proof, we give the intuition on how to handle a specific form of block random matrix in an AMP sequence.
Consider the iteration \eqref{canon_AMP}, but this time with a Gaussian matrix defined as:
\begin{equation}
    \bA = \begin{bmatrix}\bA_{1}&&&\\ & \bA_{2}&(0)&\\ &(0)&\ddots & \\ &&&\bA_{K}\end{bmatrix} \in \mathbb{R}^{n \times Kd}
\end{equation}
where $\bA_{k} \in \mathbb{R}^{n_{k}\times d}$ and $\sum_{k=1}^{K}n_{k} = n$, which leads to the following form for the products between matrices and non-linearities:
\begin{equation}
    \bA^{\top}\bh_{t}(\bv^{t}) = \begin{bmatrix}\bA_{1}^{\top}\bh_{1,t}(\bv^{t}) \\ \bA_{2}^{\top}\bh_{2,t}(\bv^{t}) \\ \vdots\\\bA_{K}^{\top}\bh_{K,t}(\bv^{t})\end{bmatrix} \in \mathbb{R}^{Kd \times K} \quad \bA\be_{t}(\bu_{t}) = \begin{bmatrix}\bA_{1}\be_{1,t}(\bu^{t}) \\ \bA_{2}\be_{2,t}(\bu^{t}) \\ \vdots\\\bA_{K}\be_{K,t}(\bu^{t})\end{bmatrix} \in \mathbb{R}^{n \times K}
\end{equation}
where the blocks $\bh_{k,t}(\bv^{t}) \in \mathbb{R}^{n_{k} \times K}, \be_{k,t}(\bu_{t}) \in \mathbb{R}^{d \times K}$ may depend on their full arguments or only the corresponding blocks depending on their separability. This iteration can be embedded as a subset of the iterates of a larger sequence defined with the full version of the matrix $\bA$ and non-linearities defined as:
\begin{align}
	&\be_{t} : \mathbb{R}^{Kd\times K^{2}} \to \mathbb{R}^{Kd\times K^{2}} \notag \\
	&\mbox{generates} \quad \begin{bmatrix}\be_{1,t}\left(\bullet\right)&&&\\ & \be_{2,t}\left(\bullet\right)&(0)&\\ &(0)&\ddots& \\ &&& \be_{K,t}\left(\bullet\right)\end{bmatrix} \in \mathbb{R}^{Kd \times K^{2}} \\
	&\bh_{t} : \mathbb{R}^{n\times K^{2}} \to \mathbb{R}^{n \times K^{2}} \notag \\
	&\mbox{generates} \quad \begin{bmatrix}\bh_{1,t}\left(\bullet\right)&&&\\ & \bh_{2,t}\left(\bullet\right)&(0)&\\ &(0)&\ddots& \\ &&& \bh_{K,t}\left(\bullet\right)\end{bmatrix} \in \mathbb{R}^{n \times K^{2}} 
\end{align} 
The original iteration is recovered on the block diagonal of the variables of the iteration. This new setting, however, introduces a richer correlation structure, since each block will be described by a different $K\times K$ covariance according to the state evolution equations. Formally, the new covariance will be a $K^{2}\times K^{2}$ block diagonal matrix. Also, the shape of the Onsager term changes from a matrix of size $K \times K$ to one of size $K^{2} \times K^{2}$ with a $K\times (K\times K)$ block diagonal structure. 

\subsection{Reformulation of the problem}
We start by reformulating problem \eqref{ERM} in a way that can be treated efficiently using an AMP iteration. With respect to the main part of this paper, we will consider the estimator $\bW \in \mathbb{R}^{d \times K}$ instead of $\mathbb{R}^{K \times d}$. The normalized (so that the cost does not diverge with the dimension) problem \eqref{ERM} then reads:
\begin{equation}
\label{ERM_bis}
    \min_{\bW\in \mathbb{R}^{d \times K},\bbb\in \mathbb{R}^{K}} \frac{1}{d}\left(L\left(\bY,\frac{1}{\sqrt{d}}\bX\bW+\bbb\right)+r(\bW)\right)
\end{equation}
where we have introduced the function $L: \mathbb{R}^{n\times K}\times\mathbb{R}^{n\times K}\to \mathbb{R}$ acting as
\begin{equation}
\left(\bY,\frac{1}{\sqrt{d}}\bX\bW+\bbb\right) \mapsto \sum_{\nu=1}^n\ell\left(\by^\nu,\frac{\bW\bx^\nu}{\sqrt d}+\bbb\right),
\end{equation}
the matrix $\bY \in \mathbb{R}^{n \times K}$ of concatenated one-hot encoded labels, and the matrix of concatenated means $\bMM \in \mathbb{R}^{K\times d}$ (in the main we took the transpose $\bMM \in \mathbb{R}^{d \times K}$). Until further notice, we will drop the scaling $\frac{1}{d}$ for convenience and study the problem
\begin{equation}
    \min_{\bW\in \mathbb{R}^{d \times K},\bbb\in \mathbb{R}^{K}} L\left(\bY,\frac{1}{\sqrt{d}}\bX\bW+\bbb\right)+r(\bW)
\end{equation}
We will write $L_{k}$ the application of $\ell$ on each row of a sub-block in $\mathbb{R}^{n_{k}\times K}$.
Without loss of generality, we can assume that the samples are grouped by clusters in the data matrix, giving the following form for $\bX \in \mathbb{R}^{n \times d}$,
separating the mean part $\bY\bMM$ and centered Gaussian part :
\begin{equation}
	\bX = 
	\bY\bMM+\tilde{\bZ}\bSigma
	\in \mathbb{R}^{n\times d}
\end{equation}
where we have introduced the block-diagonal matrix $\tilde{\bZ}$ and the $Kd\times d$ full-column-rank matrix $\bSigma$
\begin{equation}\label{Ztilde}
\tilde{\bZ}=\begin{bmatrix}\bZ_{1}&&&\\ & \bZ_{2}&(0)&\\ &(0)&\ddots& \\ &&&\bZ_{K}\end{bmatrix}\in \mathbb R^{n\times Kd}\qquad \bSigma=\begin{bmatrix} \bSigma_{1}^{1/2}\\\bSigma^{1/2}_{2}\\\vdots\\\bSigma_{K}^{1/2}\end{bmatrix}\in\mathbb R^{Kd\times d}.
\end{equation}
Here $(\bZ_{1},\dots,\bZ_{K}) \in \mathbb{R}^{n_{1}\times d}\times\cdots\times \mathbb{R}^{n_{K}\times d}$ are independent, i.i.d. standard normal matrices.

The product between the data matrix and the weights $\bW \in \mathbb{R}^{d \times K}$ then reads:
\begin{equation}
	\bX\bW =\bY\bMM\bW+\tilde{\bZ}\bSigma\bW =\begin{bmatrix}\bY_1\bMM\bW+\bZ_{1}\bSigma_{1}^{1/2}\bW\\\vdots\\\bY_K\bMM\bW+\bZ_{K}\bSigma_{K}^{1/2}\bW\end{bmatrix} \in \mathbb{R}^{n\times K}
\end{equation}
where each $\bY_{k}\in \mathbb{R}^{n_{k}\times d}$ is a $n_{k}$ copy of the same label vector. 
Defining now $\tilde{\bW}=\bSigma\bW$, observe that
\begin{equation}
	\tilde{\bW} = \bSigma\bW \quad\implies \quad  \bW = \bSigma^{+}\tilde{\bW},
\end{equation}
where
\begin{equation}
\bSigma^{+}\equiv  \left(\sum_{k=1}^{K}\bSigma_{k}\right)^{-1}\bSigma^\top
\end{equation}
is the pseudo-inverse of the matrix $\bSigma$. The optimization problem \eqref{ERM} is thus equivalent to
\begin{equation}
\label{inter-ERM}
	\inf_{\substack{\tilde{\bW}\in \mathbb{R}^{Kd\times K}\\ \bbb\in \mathbb{R}^{K}}} \sum_{k=1}^{K} L_{k}\left(\frac{1}{\sqrt{d}}\bY_{k}\bMM\bW+\frac{1}{\sqrt{d}}\bZ_{k}\tilde{\bW}_{k},\bbb\right)+r\left(\bSigma^{+}\tilde{\bW}\right)
\end{equation}
Introducing the order parameter $\bM = \frac{1}{\sqrt{d}}\bMM\bW \in \mathbb{R}^{K \times K}$, we reformulate Eq.(\ref{inter-ERM}) as a constrained optimization problem :
\begin{align}
	\inf_{\bM,\tilde{\bW},\bbb} &\sum_{k=1}^{K} L_{k}\left(\frac{1}{\sqrt{d}}\bY_{k}\bM+\frac{1}{\sqrt{d}}\bZ_{k}\tilde{\bW}_{k}\right)+r\left(\bSigma^{+} \tilde{\bW}\right) \\ 
	&\mbox{s.t.} \quad \frac{1}{\sqrt{d}}\bMM\bSigma^{+} \tilde{\bW} = \bM \notag
\end{align}
whose Lagrangian form, with dual parameters $\hat{\bM}\in\mathbb{R}^{K \times K}$, reads
\begin{equation}\comprimi
	\inf_{\bM,\tilde{\bW},\bbb}\sup_{\hat{\bM}} \sum_{k=1}^{K} L_{k}\left(\bY_{k}\bM+\frac{1}{\sqrt{d}}\bZ_{k}\tilde{\bW}_{k}\right)+r\left(\bSigma^{+} \tilde{\bW}\right)
	+\mathrm{tr}\left(\hat{\bM}^{\top}\left(\bM-\frac{1}{\sqrt{d}}\bMM\bSigma^{+} \tilde{\bW}\right)\right).
\end{equation}
This is a proper, closed, convex, strictly feasible optimization problem, thus strong duality holds and we can invert the order of the inf-sup to focus on the minimization problem in $\tilde{\bW}$ for fixed $\bM,\hat{\bM},\bbb$:
\begin{equation}
	\label{student_1}
	\inf_{\tilde{\bW}\in \mathbb{R}^{Kd\times K}} \tilde{L}\left(\frac{1}{\sqrt{d}}\tilde{\bZ}\tilde{\bW}\right)+\tilde{r}(\tilde{\bW})
\end{equation}
where we defined the loss term
\begin{subequations}
\label{new_functions}
\begin{equation}
\begin{split}
	\tilde{L}: \mathbb{R}^{n\times K} &\to \mathbb{R} \\
	\frac{1}{\sqrt{d}}\tilde{\bZ}\tilde{\bW} &\mapsto \sum_{k=1}^{K} L_{k}\left(\bY_{k}\bM+\frac{1}{\sqrt{d}}\bZ_{k}\tilde{\bW}_{k}\right) = \sum_{k=1}^{K}\sum_{i=1}^{n_{k}}\ell\left(\left[\bY_{k}\bM+\frac{1}{\sqrt{d}}\bZ_{k}\tilde{\bW}_{k}\right]_{i}\right)
\end{split}
\end{equation} 
and the regularisation term
\begin{equation}
\begin{split}
	\tilde{r}: \mathbb{R}^{Kd\times K} &\to \mathbb{R} \\
	\tilde{\bW} &\mapsto r\left(\bSigma^{+} \tilde{\bW}\right)+\mathrm{tr}\left(\hat{\bM}^{\top}\left(\bM-\frac{1}{\sqrt{d}}\bMM\bSigma^{+} \tilde{\bW}\right)\right)
\end{split}
\end{equation}
\end{subequations}
where $\bSigma^\top \tilde{\bW} = \sum_{k=1}^{K}\bSigma_{k}^{1/2}\bW_{k}$ and $\tilde{\bZ} = \left[\bZ_{k}\right]_{k=1}^{K} \in \mathbb{R}^{n \times Kd}$ is an i.i.d. standard normal block diagonal matrix as in Eq.~\eqref{Ztilde}.
\subsection{Finding the AMP sequence}
\label{build_AMP}
We now need to find an AMP iteration relating to $\tilde{\bW}$ that solve the optimization problem in Eq.~\eqref{student_1}. Although this section is not written as a formal proof, all steps are rigorous. The aim is to give the reader the core intuition on how the AMP iteration is found, otherwise the solution may feel ``parachuted''. The reader uninterested in the underlying intuition may directly skip to the next section.
In order to find the appropriate sequence two key points must be considered :
\begin{itemize}
	\item the fixed point of the sequence has to match the optimality condition of Eq.~\eqref{student_1};
	\item the update rule of the sequence should have the form Eq.~\eqref{canon_AMP} for the state evolution equations to hold.
\end{itemize}
These two points completely determine the form of the iteration. In the subsequent derivation, we absorb the scaling $\frac{1}{\sqrt{d}}$ in the matrix $\tilde{\bZ}$, such that the $\bZ_{k} \in \mathbb{R}^{n_{k}\times d}$ have i.i.d. $\mathcal{N}(0,\nicefrac{1}{d})$ elements.
\paragraph{Resolvent of the loss term}
Going back to problem Eq.~\eqref{student_1}, its optimality condition will look like :
\begin{align}
	&\tilde{\bZ}^\top \partial \tilde{L}(\bZ\tilde{\bW})+\partial\tilde{r}(\tilde{\bW}) = 0 \iff \begin{bmatrix}\bZ_{1}^\top &&&\\ & \bZ_{2}^\top &(0)&\\ &(0)&\ddots& \\ &&&\bZ_{K}^\top \end{bmatrix}\begin{bmatrix}\partial \tilde{L}_{1}(\bZ_{1}\tilde{\bW}_{1})\\\partial \tilde{L}_{2}(\bZ_{2}\tilde{\bW}_{2}))\\\vdots\\\partial \tilde{L}_{K}(\bZ_{K}\tilde{\bW}_{K}))\end{bmatrix}+\partial \tilde{r}(\tilde{\bW}) = 0
\end{align}
where each $\bZ_{k} \in \mathbb{R}^{n_{k}\times d}$, and the subdifferential of $\tilde{L}$ is separable across blocks of size $n_{k}\times d$, and $\partial \tilde{r}(\tilde{\bW}) \in \mathbb{R}^{Kd \times K}$. Following the intuition of spatial coupling, we introduce the \textit{full} matrix $\bZ \in \mathbb{R}^{n \times Kd}$, with i.i.d.~$\mathcal{N}(0,\nicefrac{1}{d})$ entries. The optimality condition can then be written on the diagonal of a $Kd \times K^{2}$ matrix:
\begin{multline}\bZ^\top \begin{bmatrix}\partial \tilde{L}_{1}(\bZ_{1}\tilde{\bW}_{1})&&&\\ & \partial \tilde{L}_{2}(\bZ_{2}\tilde{\bW}_{2})&(0)&\\ &(0)&\ddots& \\ &&&\partial \tilde{L}_{K}(\bZ_{K}\tilde{\bW}_{K})\end{bmatrix}\\
+\begin{bmatrix}\partial \tilde{r}(\tilde{\bW})_{1}&&&\\ & \partial \tilde{r}(\tilde{\bW})_{2}&(0)&\\ &(0)&\ddots& \\ &&&\partial \tilde{r}(\tilde{\bW})_{K}\end{bmatrix} = \boldsymbol 0
\end{multline}
where $\partial \tilde{r}(\tilde{\bW})_{k}$ represents the $k$-th block of the subdifferential of $\tilde{r}$ which is non-separable across the blocks of $\tilde{\bW}$. To make the resolvents/proximals appear, we add the argument of the subdifferentials on both sides weighted by a (symmetric) positive definite matrix $\bS_{k} \in \mathbb{R}^{K \times K}$ which will be used to allow for Onsager correction while respecting the fixed point condition. Using the notation defined in section \ref{sec:req_back}
\begin{align}
	&\left[\bZ_{k}^\top \partial \tilde{L}_{k}(\bZ_{k}\tilde{\bW}_{k})\right]_{k=1}^{K}+\left[\partial \tilde{r}(\tilde{\bW})\right]_{k=1}^{K} = 0  \notag\\
	& \iff \left[\bZ_{k}^\top \partial \tilde{L}_{k}(\bZ_{k}\tilde{\bW}_{k})+\bZ_{k}^\top \bZ_{k}\tilde{\bW}_{k}\bS_{k}^{-1}\right]_{k=1}^{K}+\left[\partial \tilde{r}(\tilde{\bW})\right]_{k=1}^{K} = \left[\bZ_{k}^\top \bZ_{k}\tilde{\bW}_{k}\bS_{k}^{-1}\right]_{k=1}^{K} 
\end{align}
for a given set of positive definite matrices $\{\bS_{k}\}_{k\in[K]}$. Again, the reason for introducing different $\bS_{k}$ on each block is to match the expected structure of the Onsager term. We can introduce the resolvent, formally Bregman resolvent/proximal operator: 
\begin{equation}
	\bU_{k} \equiv \partial \tilde{L}_{k}(\bZ_{k}\tilde{\bW}_{k})\bS_{k}+\bZ_{k}\tilde{\bW}_{k} \iff \bZ_{k}\tilde{\bW}_{k} = \bR_{\tilde{L}_{k},\bS_{k}}(\bU_{k})
\end{equation}
where
\begin{equation}
\comprimi
\label{block_res_loss}
\begin{split}
	\bR_{\tilde{L}_{k},\bS_{k}}(\bU_{k}) &= (\mathrm{Id}+\partial \tilde{L}_{k}(\bullet)\bS_{k})^{-1}(\bU_{k}) \\
	&=\Argmin_{\bT \in \mathbb{R}^{n_{k}\times K}}\left\{\tilde{L}_{k}(\bT)+\frac{1}{2}\mathrm{tr}\left((\bT-\bU_{k})\bS_{k}^{-1}(\bT-\bU_{k})^\top \right)\right\} \\
	&= \Argmin_{\bT \in \mathbb{R}^{n_{k}\times K}}\left\{L_{k}(\bT)+\frac{1}{2}\mathrm{tr}\left((\bT-(\bY_{k}\bM+\bU_{k}))\bS_{k}^{-1}(\bT-(\bY_{k}\bM+\bU_{k}))^\top \right)\right\}-\bY_{k}\bM.
\end{split}
\end{equation}
In the previous expressions $\partial \tilde{L}_{k} \in \mathbb{R}^{n_{k}\times K}$ and $\bV_{k} \in \mathbb{R}^{K \times K}$. 
The following formulation of the optimality condition is reached:
\begin{align}
	&\left[\bZ_{k}^\top \bU_{k}\bS_{k}^{-1}\right]_{k=1}^{K}+\left[\partial\tilde{r}(\tilde{\bW})_{k}\right]_{k=1}^{K} = \left[\bZ_{k}^\top \bR_{\tilde{L}_{k},\bS_{k}}(\bU_{k})\bS_{k}^{-1}\right]_{k=1}^{K} \notag\\
	&\iff\left[\bZ_{k}^\top \left(\bU_{k}-\bR_{\tilde{L}_{k},\bS_{k}}(\bU_{k})\right)\bS_{k}^{-1}\right]_{k=1}^{K}+\left[\partial \tilde{r}(\tilde{\bW})_{k}\right]_{k=1}^{K} = 0
\end{align}
\paragraph{Resolvent of the regularization term}
Determining the block decomposition of the subdifferential of the regularization term is 
less simple. We would like a block expression in the flavour of:
\begin{equation}
	\left[\partial \tilde{r}(\tilde{\bW})_{k}\right]_{k=1}^{K}+\left[\tilde{\bW}_{k}\bhS^{-1}_{k}\right]_{k=1}^{K} = \left[\tilde{\bW}_{k}\bhS^{-1}_{k}\right]_{k=1}^{K}
\end{equation}
At this point it becomes clear that we cannot consider the resolvent as acting on $\tilde{\bW}\in \mathbb{R}^{Kd \times K}$ otherwise there could be only one $\bhS \in \mathbb{R}^{K \times K}$ and there would be a mismatch with the expected form of the Onsager terms. As specified by the definitions Eq.(\ref{new_functions}),
the subdifferential of $\tilde{r}$ is acting on the whole block diagonal matrix $[\tilde{\bW}_{k}]_{k=1}^{K}$, by way of summation due to the action of the pseudo-inverse $\bSigma^{+}$. We can thus consider its proximal acting on $\mathbb{R}^{d \times K^{2}}$ as $[\tilde{\bW}_{1}\tilde{\bW}_{2}...\tilde{\bW}_{K}]$ (note that we could have also worked directly with a block diagonal matrix in $\mathbb{R}^{Kd \times K^{2}}$).
Proceeding in this way, we can directly write our expression as an application parametrized by another set of positive definite matrices $\{\bhS_{k}\}_{k\in[K]}$.
\begin{equation}
\label{eq:res_reg}
\bhU = \left(\mathrm{Id}+\partial \tilde{r}(\bullet)\bhS\right)(\tilde{\bW}) \qquad \tilde{\bW} = \bR_{\tilde{r},\bhS}(\bhU)
\end{equation}
where
\begin{equation}
\begin{split}
\bR_{\tilde{r},\bhS}(\bhU) &= \left(\mathrm{Id}+\partial \tilde{r}(\bullet)\bhS\right)^{-1}(\bhU) \\
	&= \Argmin_{\bT \in \mathbb{R}^{d \times K^{2}}}\left\{\tilde{r}(\bT)+\frac{1}{2}\mathrm{tr}\left((
	\bT-\bhU)\bhS^{-1}(\bT-\bhU)^\top \right)\right\}
\end{split}
\end{equation}
where $\bhS \in \mathbb{R}^{K^{2}\times K^{2}}$ block diagonal, and $\bhU\in \mathbb{R}^{d\times K^{2}}$.
This would lead to the equivalent optimality condition for the regularization part:
\begin{equation}
	\bhU\bhS^{-1} = \bR_{\tilde{r},\bhS}(\bhU)\bhS^{-1} \iff \left[\bhU_{k}\bhS_{k}^{-1}\right]_{k=1}^{K} = \left[\bR_{\tilde{r},\bhS,k}(\bhU)\bhS_{k}^{-1}\right]_{k=1}^{K}
\end{equation}
We now need to figure out the block structure of this resolvent since we want to spread it across a block diagonal matrix. Let $\bC = \sum_{k=1}^{K}\bSigma_{k}$, so that $\bSigma^{+} = \bC^{-1}\bSigma^{\top}$, and the blocks $\bT_{k}\in\mathbb{R}^{d \times K}$ are the solution to the minimization problem
\begin{align}
	\min_{\{\bT_{k}\}_{{k\in [K]}}\in (\mathbb{R}^{d \times K})^{K}} r(\bC^{-1}\sum_{k=1}^{K}\bSigma^{1/2}_{k}\bT_{k})&+\frac{1}{2}\mathrm{tr}\left((\bT-\bhU)\bhS^{-1}(\bT-\bhU^\top )\right) \notag \\
	&+\mathrm{tr}\left(\hat{\bM}^{\top}\left(\bM-\frac{1}{\sqrt{d}}\bMM\bSigma^+\bT\right)\right)
\end{align}
Let $\tilde{\bT} = \bC^{-1}\sum_{k=1}^{K}\bSigma_{k}^{1/2}\bT_{k} \in \mathbb{R}^{d \times K}$, and the equivalent reformulation as a constraint optimization problem:
\begin{align}
	\min_{\substack{\bT_{k\in [K]}\in \mathbb{R}^{d \times K}\\\tilde{\bT}\in \mathbb{R}^{d \times K}}} &r(\tilde{\bT})+\frac{1}{2}\mathrm{tr}\left((\bT-\bhU)\bhS^{-1}(\bT-\bhU^\top )\right)+\mathrm{tr}\left(\hat{\bM}^{\top}\left(\bM-\frac{1}{\sqrt{d}}\bMM\tilde{\bT}\right)\right) \\
	&\mbox{s.t.} \quad \tilde{\bT} = \bC^{-1}\sum_{k=1}^{K}\bSigma_{k}^{1/2}\bT_{k} \notag
\end{align}
This is a feasible convex problem under convex constraint with a strongly convex term, it thus has a unique solution and strong duality holds.
Introducing the Lagrange multiplier $\blambda \in \mathbb{R}^{d \times K}$, we get the equivalent representation:
\begin{align}
	\min_{\substack{\bT_{k\in [K]}\in \mathbb{R}^{d \times K}\\\tilde{\bT}\in \mathbb{R}^{d \times K}}}\max_{\blambda \in \mathbb{R}^{d\times K}} r(\tilde{\bT})&+\sum_{k=1}^{K}\mathrm{tr}\left((\bT_{k}-\bhU_{k})\bhS_{k}^{-1}(\bT_{k}-\bhU_{k})^{\top}\right) \notag
	\\&+\mathrm{tr}\left(\blambda^\top \left(\tilde{\bT}- \bC^{-1}\sum_{k=1}^{K}\bSigma_{k}^{1/2}\bT_{k}\right)\right)
	+\mathrm{tr}\left(\hat{\bM}^{\top}\left(\bM-\frac{1}{\sqrt{d}}\bMM\tilde{\bT}\right)\right).
\end{align}
The optimality condition for this problem reads:
\begin{align}
&\partial_{\tilde{\bT}}: \quad \partial r(\tilde{\bT})+\blambda-\frac{1}{\sqrt{d}}\bMM^\top \bhM = 0 \\
&\partial_{\bT}: \quad (\bT_{k}-\bU_{k})\bhS_{k}^{-1} = \bSigma_{k}^{1/2}\bC^{-1}\blambda \qquad \forall k \in [K] \\
&\partial_{\blambda}: \quad \tilde{\bT}=\bC^{-1}\sum_{k=1}^{K}\bSigma_{k}^{1/2}\bT_{k}
\end{align}
Using the 
gradient condition on $\bT$, we get 
\begin{equation}
	\sum_{k=1}^{K}\bSigma^{1/2}_{k}(\bT_{k}-\bhU_{k})\bhS_{k}^{-1} = \blambda
\end{equation}
The constraint $\tilde{\bT}=\bC^{-1}\sum_{k=1}^{K}\bSigma_{k}^{1/2}\bT_{k}$ is solved by $\bT_{k} = \bSigma_{k}^{1/2}\tilde{\bT}$ which gives the solution for $\blambda$
\begin{equation}
	\blambda = \sum_{k=1}^{K}\bSigma_{k}^{1/2}(\bSigma^{1/2}_{k}\tilde{\bT}-\bhU_{k})\bhS_{k}^{-1} = \sum_{k=1}^{K}\bSigma_{k}\tilde{\bT}\bhS_{k}^{-1}-\sum_{k=1}^{K}\bSigma_{k}^{1/2}\bhU_{k}\bhS_{k}^{-1}
\end{equation}
and prescribes the following form for $\tilde{\bT}$, as solution to the problem
\begin{align}
\label{block_res_reg}
	&\partial r(\tilde{\bT})+\sum_{k=1}^{K}\bSigma_{k}\tilde{\bT}\bhS_{k}^{-1}-\sum_{k=1}^{K}\bSigma_{k}^{1/2}\bhU_{k}\bhS_{k}^{-1}-\frac{1}{\sqrt{d}}\bMM^\top \hat{\bM} = 0 \notag \\
	&\iff \Argmin_{\tilde{\bT}} r(\tilde{\bT})+\frac{1}{2}\sum_{k=1}^{K}\bSigma_{k}\tilde{\bT}\bhS_{k}^{-1}\tilde{\bT}-\left(\sum_{k=1}^{K}\bSigma_{k}^{1/2}\bhU_{k}\bhS_{k}^{-1}+\frac{1}{\sqrt{d}}\bMM^\top \hat{\bM}\right)\tilde{\bT}
\end{align}
We then recover $\bT$ from $\bT = \boldsymbol{\Sigma}\tilde{\bT}$. Thus, defining the function
\begin{align}
\label{def_eta}
    &\beeta:\mathbb{R}^{d\times K^{2}}\to \mathbb{R}^{d \times K} \notag\\
    &\bhU \mapsto \Argmin_{\tilde{\bT}} r(\tilde{\bT})+\frac{1}{2}\sum_{k=1}^{K}\bSigma_{k}\tilde{\bT}\bhS^{-1}_{k}\tilde{\bT}-\left(\sum_{k=1}^{K}\bSigma_{k}^{1/2}\bhU_{k}\bhS_{k}^{-1}+\frac{1}{\sqrt{d}}\bMM^\top \hat{\bM}\right)\tilde{\bT}
\end{align}
the block decomposition of the resolvent for the regularizer reads:
\begin{equation}
    \bR_{\tilde{r},\bhS,k}(\bhU) = \bSigma_{k}^{1/2}\beeta(\bhU)
\end{equation}

\paragraph{Matching the optimality condition with the AMP fixed point}
The global optimality condition then reads:
\begin{align}
\label{prox-opti}
	&\left[\bZ_{k}^\top \left(\bR_{\tilde{L}_{k},\bS_{k}}(\bU_{k})-\bU_{k}\right)\bS_{k}^{-1}\right]_{k=1}^{K} = \left[(\bhU_{k}-\bR_{\tilde{r},\bhS,k}(\bhU))\bhS_{k}^{-1}\right]_{k=1}^{K} \\
	&\left[\bZ_{k}\bR_{\tilde{r},\bhS,k}(\bhU)\right]_{k=1}^{K} = \left[\bR_{\tilde{L}_k,\bS_{k}}(\bU_{k})\right]_{k=1}^{K}
\end{align}
where both equations should be satisfied. We can now define update functions based on the previously obtained block decomposition. The fixed point of the matrix-valued AMP Eq.(\ref{canon_AMP}) reads:
\begin{align}
	\mathrm{Id}+\be(\bu)\langle \bh' \rangle^\top  &= \bZ^\top \bh(\bv) \\
	\mathrm{Id}+\bh(\bv)\langle \be' \rangle^\top  &= \bZ\be(\bu) 
\end{align}
Matching this fixed point with the optimality condition Eq.(\ref{prox-opti}) suggests the following mapping:
\begin{equation}
\begin{split}
\bh_{k}(\bU_{k}) &= \left(\bR_{\tilde{L}_{k},\bS_{k}}(\bU_{k})-\bU_{k}\right)\bS_{k}^{-1}, \\
\be_{k}(\bhU) &= \bR_{\tilde{r},\bhS,k}(\bhU\bhS),
\end{split}\qquad
\begin{split}
\bS_{k} &= \langle \be'_{k} \rangle,\\
\bhS_{k} &= -\langle \bh_{k}'\rangle^{-1},
\end{split}
\end{equation}
where we redefined $\bhU\equiv \bhU\bhS$ in \eqref{eq:res_reg}, and the subscripts on the non-linearities are block indexes.
\subsection{Proof of Theorem \ref{the:1} using the AMP sequence}
Following the analysis carried out in the previous section, define the following two sequences of non-linearities, for fixed values of the parameters $\hat{\bM},\bM,\bbb$ and any $\bu \in \mathbb{R}^{d\times K^{2}}, \bv \in \mathbb{R}^{n \times K}$ :
\begin{align}
	&\be_{t} : \mathbb{R}^{Kd\times K^{2}} \to \mathbb{R}^{Kd\times K^{2}} \notag \\
	&\bu\mapsto \quad \begin{bmatrix}\be_{1,t}\left(\bu\right)&&&\\ & \be_{2,t}\left(\bu\right)&(0)&\\ &(0)&\ddots& \\ &&& \be_{K,t}\left(\bu\right)\end{bmatrix} \in \mathbb{R}^{Kd \times K^{2}} \\
	&\bh_{t} : \mathbb{R}^{n\times K^{2}} \to \mathbb{R}^{n \times K^{2}} \notag \\
	&\bv \mapsto \quad \begin{bmatrix}\bh_{1,t}\left(\bv_{1}\right)&&&\\ & \bh_{2,t}\left(\bv_{2}\right)&(0)&\\ &(0)&\ddots& \\ &&& \bh_{K,t}{t}\left(\bv_{K}\right)\end{bmatrix} \in \mathbb{R}^{n \times K^{2}} 
\end{align} 
where $\bY_{k} \in \mathbb{R}^{n_{k} \times K}$ and 
\begin{align}
    \bh_{k,t} &: \mathbb{R}^{n_{k}\times K}\to \mathbb{R}^{n_{k}\times K}\notag\\ &\bv_{k} \mapsto \left(\bR_{\tilde{L}_{k},\bV^{k,t}}(\bv_{k})-\bv_{k}\right)(\bV^{k,t})^{-1} \notag \\
    &= \left(\Argmin_{\bT \in \mathbb{R}^{n_{k}\times K}}\left\{\tilde{L}_{k}(\bT)+\frac{1}{2}\mathrm{tr}\left((\bT-\bv_{k})(\bV_{k,t})^{-1}(\bT-\bv_{k})^\top \right)\right\}-\bv_{k}\right)(\bV_{k,t})^{-1} \notag \\
    &=\left(\mbox{Prox}_{L_{k}(\bullet(\bV_{k,t})^{1/2})}((\bY_{k}\bM+\bv_{k})(\bV_{k,t})^{-1/2})(\bV_{k,t})^{1/2}-(\bY_{k}\bM+\bv_{k})\right)(\bV_{k,t})^{-1} \\
    \be_{k,t} &: \mathbb{R}^{d\times K^{2}}\to \mathbb{R}^{d\times K}\notag\\
    \bu &\mapsto \bSigma^{1/2}_{k}\Argmin_{\tilde{\bT}\in \mathbb{R}^{d \times K}} r(\tilde{\bT})+\frac{1}{2}\sum_{k=1}^{K}\bSigma_{k}\tilde{\bT}\hat{\bV}_{k,t}\tilde{\bT}-\left(\sum_{k=1}^{K}\bSigma_{k}^{1/2}\bu_{k}+\frac{1}{\sqrt{d}}\bMM^\top \hat{\bM}\right)\tilde{\bT} \notag \\
    &=\bSigma^{1/2}_{k}\beeta(\bu(\hat{\bV}^{t})^{-1})
\end{align}
where $(\bV_{t},\hat{\bV}_{t}) \in \mathbb{R}^{K^{2} \times K^{2}}$, are defined as the block diagonal matrices $\left[\bV_{k,t}\right]_{k\in [K]},\left[\hat{\bV}_{k,t}\right]_{k\in [K]}$ such that
\begin{align}
    \bV_{k,t} = \langle (\be_{k,t-1})'\rangle \quad \hat{\bV}_{k,t} = -\langle (\bh_{k,t})'\rangle
\end{align}
using the notation from Eq.~\eqref{eq:mat_ons}. Now define the following sequence, initialized with 
\begin{align}
    &\bu^{0},\bh^{-1}\equiv 0, \hat{\bV}_{0} \\
    &\mbox{such that} \lim_{d \to \infty} \frac{1}{d}\norm{\be_{0}(\bu^{0})^\top \be_{0}(\bu^{0})}_{\rm F}<+\infty \thickspace \mbox{and} \thickspace \hat{\bV}_{0} \in \mathbb{S}_{K}^{++} \notag\\
&\mbox{and recursively define} \notag \\
\label{final_AMP}
	&\hspace{1cm} \bu^{t+1} = \boldsymbol{Z}^{\top}\bh_{t}(\bv^{t})-\be_{t}(\bu^{t})\langle \bh_{t}'\rangle^\top \\
	&\hspace{1cm} \bv^{t} = \boldsymbol{Z}\be_{t}(\bu^{t})-\bh_{t-1}(\bv^{t-1})\langle \be_{t}'\rangle^\top 
\end{align}
where $\boldsymbol{Z} \in \mathbb{R}^{n \times Kd}$ has i.i.d. $\mathcal{N}(0,\nicefrac{1}{d})$ elements, and in the Jacobians defining $\hat{\bV},\bV$, we used the notation from Eq.~\eqref{eq:jacob}.
\paragraph{State evolution equations}
The results from section \ref{build_AMP} show that the functions $\be^{t},\bh^{t}$ are proximals operators, and thus are Lipschitz continuous for all $t \in \mathbb{N}$, along with their block restrictions. Therefore the conditions of Theorem \ref{th:SE} are verified and we have the following lemma:
\begin{Lemma}
\label{block_SE}
Consider the sequence defined by Eq.(\ref{final_AMP}), for any fixed $\bM,\hat{\bM},\bbb$. For any sequences of pseudo-Lipschitz functions $\phi_{1,n}:\mathbb{R}^{d \times K^{2}}\to \mathbb{R},\phi_{2,n}:\mathbb{R}^{n\times K^{2}} \to \mathbb{R}$, for any $t \in \mathbb{N}^{*}$:
\begin{align}
    \phi_{1,n}(\bu_{1}^{t},\dots,\bu_{K}^{t})\stackrel{\rm P}\simeq\mathbb{E}\left[\phi_{1,n}(\bH_{1}(\hat{\bQ}_{1,t})^{1/2},\dots,\bH_{K}(\hat{\bQ}_{K,t})^{1/2})\right] \\
    \phi_{2,n}(\bv^{t}_{1},\dots,\bv^{t}_{K})\stackrel{\rm P}\simeq\mathbb{E}\left[\phi_{1,n}(\bG_{1}(\bQ_{1,t})^{1/2},\dots,\bG_{K}(\bQ_{K,t})^{1/2})\right]
\end{align}
where the matrices $\bH_{k} \in \mathbb{R}^{d \times K},\bG_{k} \in \mathbb{R}^{n_{k} \times K}$ are independent matrices with i.i.d. standard normal elements, and at each time step $t\geqslant 1$
\begin{align}
    \bQ_{k,t} &= \lim_{d \to  +\infty}\frac{1}{d}\mathbb{E}\left[\be_{k,t}(\{\bH_{k}(\hat{\bQ}_{k,t})^{1/2}(\hat{\bV}_{k,t})^{-1}\}_{k\in[K]})^{\top}\be_{k,t}(\{\bH_{k}(\hat{\bQ}_{k,t})^{1/2}(\hat{\bV}_{k,t})^{-1}\}_{k\in[K]})\right] \\ &\in \mathbb{R}^{K \times K} \notag\\
    \hat{\bQ}_{k,t} &= \lim_{d \to +\infty} \frac{1}{d}\mathbb{E}\left[\bh_{k,t-1}(\bG_{k}(\bQ_{k,t-1})^{1/2})^{\top}\bh_{k,t-1}(\bG_{k}(\bQ_{k,t-1})^{1/2})\right] \in \mathbb{R}^{K \times K} \\
    \bV_{k,t} &= \lim_{d\to +\infty}\frac{1}{d} \sum_{i=1}^{d}\frac{\partial \be_{k,t-1}(\{\bH_{k}(\hat{\bQ}_{k,t-1})^{1/2}\}_{k\in[K]})}{\partial (\bH_{k}(\hat{\bQ}_{k,t-1})^{1/2})_{i} } \in \mathbb{R}^{K \times K} \\
    \hat{\bV}_{k,t} &= -\lim_{d\to +\infty}\frac{1}{d} \sum_{i=1}^{n_{k}}\frac{\partial \bh_{k,t}(\bG_{k}(\bQ_{k,t})^{1/2})}{\partial (\bG_{k}(\bQ_{k,t})^{1/2})_{i}} \in \mathbb{R}^{K \times K}
\end{align}
where the sequence is initialized with $\hat{\bV}_{0},\be_{0},\bQ_{0,0} = \lim_{d \to \infty} \frac{1}{d}\norm{\be_{0}(\bu^{0})^\top \be_{0}(\bu^{0})}_{\rm F}$.
\end{Lemma}
\begin{proof}
Lemma \ref{block_SE} is a consequence of Theorem \ref{th:SE} whose assumptions have been verified in the paragraph.
\end{proof}
Note that in Lemma \ref{block_SE}, we have directly written the block decomposition of the state evolution corresponding to the iteration Eq.~\eqref{final_AMP}, which involves the block diagonal matrices $\bQ_{t},\hat{\bQ}_{t},\bV_{t},\hat{\bV}_{t}$ which are all in $\mathbb{R}^{K^{2} \times K^{2}}$. Using the notations introduced in section \ref{sec:req_back}
\begin{equation}
    \bV = \left[\bV_{k}\right]_{k=1}^{K} \thickspace  \hat{\bV} = \left[\hat{\bV}_{k}\right]_{k=1}^{K} \thickspace \bQ = \left[\bQ_{k}\right]_{k=1}^{K} \thickspace \hat{\bQ} = \left[\hat{\bQ}_{k}\right]_{k=1}^{K}
\end{equation}
Also note that we do not use the full state evolution giving the correlations across all time steps, but only use those at equal times $t$.
\paragraph{Trajectories and fixed point of the AMP sequence}
Now that we have a sequence with state evolution equations, the following two lemmas link the fixed points of this iteration to any optimal solution of problem Eq.(\ref{student_1}).
\begin{Lemma}
\label{lemma:fixed_point}
Consider any fixed point $\bV,\hat{\bV},\bQ, \hat{\bQ}$ of the state evolution equations from Lemma \ref{block_SE}. For any fixed point $\bu^{*},\bv^{*}$ of iteration Eq.(\ref{final_AMP}), the quantity
\begin{align}
    \bR_{\tilde{r},\hat{\bV}}(\bu^{*}\hat{\bV}^{-1}) = \left(\mathrm{Id}+\partial \tilde{r}(\bullet)\hat{\bV}^{—1}\right)(\bu^{*}\hat{\bV}^{-1})
\end{align}
is an optimal solution $\tilde{\bW}^{\star}$ of problem Eq.( \ref{student_1}). Furthermore 
\begin{equation}
\bR_{\tilde{L},\bV}(\bv^{*}) = (\mathrm{Id}+\partial\tilde{L}(\bullet)\bV)(\bv^{*}) = \bZ\tilde{\bW}^{\star}
\end{equation}
where the block decompositions of each resolvents have been explicitly calculated in section \ref{build_AMP}.
\end{Lemma}
\begin{proof}
Lemma \ref{lemma:fixed_point} is a direct consequence of the analysis carried out in section \ref{build_AMP}.
\end{proof}
At this point we know the fixed points of the AMP iteration correspond to the optimal solutions of problem Eq.(\ref{student_1}). Note that the resolvents/proximals linking the fixed point of the AMP iteration with the solutions of Eq.(\ref{student_1}) are Lipschitz continuous, making them acceptable transforms for state evolution observables. However this does not guarantee that the optimal solution is characterized by the fixed point of the state evolution equations. Indeed, we need to show that a converging trajectory can be systematically found for any instance of the problem Eq.(\ref{student_1}). This is the purpose of the following lemma.
\begin{Lemma}
\label{lemma:conv_traj}
Consider iteration Eq.(\ref{final_AMP}), where the parameters $\bQ,\hat{\bQ},\bV,\hat{\bV}$ are initialized at any fixed point of the state evolution equations of Lemma \ref{block_SE}. For any sequence initialized with $\bhV_{0} = \bhV$ and  $\bu^{0}$ such that
\begin{equation}
    \lim_{d \to \infty}\frac{1}{d}\be_{0}({\bu^{0}})^{\top}\be_{0}(\bu^{0}) = \bQ
\end{equation}
the following holds
\begin{equation}
    \lim_{t\to \infty}\lim_{d \to \infty}\frac{1}{\sqrt{d}}\norm{\bu^{t}-\bu^{\star}}_{\rm F} = 0 \quad \lim_{t\to \infty}\lim_{d \to \infty}\frac{1}{\sqrt{d}}\norm{\bv^{t}-\bv^{\star}}_{\rm F} = 0
\end{equation}
\end{Lemma}
\begin{proof}
The proof of Lemma \ref{lemma:conv_traj} is deferred to subsection \ref{sec:proof_conv}.
\end{proof}
Note that the $\bG$ defined here is not the same as the $\bG$ in the replica computation.
Combining the lemmas \ref{block_SE}, \ref{lemma:fixed_point} and \ref{lemma:conv_traj} with the pseudo-Lipschitz property, we have reached the following lemma
\begin{Lemma}
\label{partial_sol}
For any fixed $\bM,\hat{\bM},\bbb$, consider the fixed point $(\bQ,\hat{\bQ},\bV,\hat{\bV})$ of the state evolution equations from Lemma. \ref{block_SE}. Then, for any sequences of pseudo-Lipschitz functions $\phi_{1,n}:\mathbb{R}^{d \times K^{2}}\to \mathbb{R},\phi_{2,n}:\mathbb{R}^{n\times K} \to \mathbb{R}$, for $n,d \to \infty$
\begin{align}
    \phi_{1,n}(\tilde{\bW}^{\star}) &\stackrel{\rm P}\simeq \mathbb{E}\left[\phi_{1,n}\left(R_{\tilde{r},\hat{\bV}}(\bH\hat{\bQ}^{1/2}\hat{\bV}^{-1})\right)\right] \\
    \phi_{2,n}(\bZ\tilde{\bW}^{\star}) &\stackrel{\rm P}\simeq \mathbb{E}\left[\phi_{2,n}\left(R_{\tilde{L},\bV}(\bG\bQ^{1/2})\right)\right]
\end{align}
where we remind that $\bG=\left[\bG_{k}\right]_{k=1}^{K},\bH=\left[\bH_{k}\right]_{k=1}^{K}$ are block diagonal i.i.d. standard normal matrices as in Lemma \ref{block_SE}, and $\thickspace \bQ = \left[\bQ_{k}\right]_{k=1}^{K} \thickspace \hat{\bQ} = \left[\hat{\bQ}_{k}\right]_{k=1}^{K}$ are the $K^{2}\times K^{2}$ block diagonal covariances.
\end{Lemma}
\begin{proof}
Lemma \ref{partial_sol} is a consequence of Lemmas \ref{block_SE},\ref{lemma:fixed_point},\ref{lemma:conv_traj} and applying the pseudo-Lipschitz property along with the fact that the iterates of the AMP have bounded norm using the state evolution and that the estimator also has bounded norm (feasibility assumption). Note that, for a generically non-strictly convex problem, being close to the zero gradient condition does not guarantee being close to the estimator. This is further discussed in Appendix \ref{app:strct}.
\end{proof}
Note that the resolvents are implicitly acting on the block diagonals of their arguments.
At this point we are quite close to Theorem \ref{the:1}(details for the exact matching will be given later), but we are missing the equations on $\bM,\hat{\bM},\bbb$.
\paragraph{Fixed point equations for $\bM,\hat{\bM},\bbb$}
We drop the dependence on the bias term $\bbb$ as its solution is very similar to the one for $\bM,\hat{\bM}$.
To obtain the equations for $\bM,\hat{\bM}$, we go back to the complete optimization problem
\begin{align}
	\inf_{\bM,\tilde{\bW},\bbb}\sup_{\hat{\bM}} L(\bY_{k}\bM+\bZ_{k}\tilde{\bW}_{k})&+r\left(\bSigma^+ \tilde{\bW}\right)\notag \\
	&+\mathrm{tr}\left(\hat{\bM}^{\top}\left(\bM-\frac{1}{\sqrt{d}}\bMM\bSigma^+ \tilde{\bW}\right)\right)
\end{align}
where we can use strong duality to write the equivalent form
\begin{align}
	\inf_{\bM,\bbb}\sup_{\hat{\bM}} L(\bY_{k}\bM+\bZ_{k}\tilde{\bW}^{\star}_{k})&+r\left(\bSigma^+ \tilde{\bW}\right)\notag \\
	&+\mathrm{tr}\left(\hat{\bM}^{\top}\left(\bM-\frac{1}{\sqrt{d}}\bMM\bSigma^+ \tilde{\bW}^{\star}\right)\right)
\end{align}
The gradients w.r.t. $\bM,\hat{\bM}$ then read:
\begin{align}
\label{grad_m}
    &\partial\hat{\bM} = \bM- \frac{1}{\sqrt{d}}\bMM\bSigma^+ \tilde{\bW}^{\star} \\
    &\partial\bM = \hat{\bM}+\partial_{\bM}L(\bY\bM+\bZ\tilde{\bW}^{\star})
\end{align}
Uniform convergence of derivatives and conditions for the dominated convergence theorem are verified using similar arguments as in \cite[Lemma 12]{Loureiro2021}. We can thus invert limits and derivatives, and expectations and derivatives.
To facilitate taking the derivative $\partial_{\bM}$, we use Lemma \ref{partial_sol} (assuming the normalized loss function is pseudo-Lipschitz, which is a very loose assumption verified by most machine learning losses) to obtain, reintroducing the scaling $\nicefrac{1}{d}$
\begin{equation}
     \frac{1}{d}L(\bY\bM+\bZ\tilde{\bW}^{\star}) \xrightarrow[d \to \infty]{P} \frac{1}{d}\mathbb{E}\left[ L(\bY\bM+\bR_{\tilde{L},\bV}(\bG\bQ^{1/2}))\right]
\end{equation}
Using the block decomposition from Eq.(\ref{block_res_loss}), the blocks $(\bR_{\tilde{L},\bV}(\bG\bQ^{1/2}))_{k} \in \mathbb{R}^{n_{k} \times K}$ are given by:
\begin{equation}\comprimi
      \Argmin_{\bT \in \mathbb{R}^{n_{k}\times K}}\left\{L_{k}(\bT)+\frac{1}{2}\mathrm{tr}\left((\bT-(\bY_{k}\bM+\bG_{k}\bQ^{1/2}_{k}))\bV_{k}^{-1}(\bT-(\bY_{k}\bM+\bG_{k}\bQ^{1/2}_{k}))^\top \right)\right\}-\bY_{k}\bM
\end{equation}
Using a block diagonal representation, we can write:
\begin{multline}
\frac{1}{d}L(\bY\bM+R_{\tilde{L},\bV}(\bG\bQ^{1/2})) = \frac{1}{d}L(R_{L,\bV}(\bY\bM+\bG\bQ^{1/2})) \\
= \frac{1}{d}\mathcal{M}_{L,\bV}(\bY\bM+\bG\bQ^{1/2})-\\
\comprimi\frac{1}{2d}\mathrm{tr}\left((\bR_{L,\bV}(\bY\bM+\bG\bQ^{1/2})-(\bY\bM+\bG\bQ^{1/2}))\bV^{-1}(\bR_{L,\bV}(\bY\bM+\bG\bQ^{1/2})-(\bY\bM+\bG\bQ^{1/2}))^{\top}\right)
\end{multline}
where we have introduced the Bregman-envelope \cite{bauschke2018regularizing} with respect to the distance Eq.~\eqref{breg_dist}
\begin{align}
    &\mathcal{M}_{L,\bV}(\bY\bM+\bG\bQ^{1/2}) =\notag\\ &\min_{\bT}\left\{L(\bT)+\frac{1}{2}\mathrm{tr}\left((\bT-(\bY\bM+\bG\bQ^{1/2}))\bV^{-1}(\bT-(\bY\bM+\bG\bQ^{1/2}))^{\top}\right)\right\}
\end{align}
Then, using the state evolution equations from Lemma \ref{block_SE} and Stein's lemma, we can write:
\begin{align}
    \frac{1}{d}L(\bY\bM+\bR_{\tilde{L},\bV}(\bG\bQ^{1/2})) = \frac{1}{d}\mathcal{M}_{L,\bV}(\bY\bM+\bG\bQ^{1/2})-\frac{1}{2}\mbox{tr}(\bV^{\top}\bQ)
\end{align}
Taking the gradient w.r.t.~$\bM$ using the expression for the derivative of a Bregman envelope \cite{bauschke2018regularizing}, we get:
\begin{equation}
    \partial_{\bM}L(\bY\bM+\bR_{\tilde{L},\bV}(\bG\bQ^{1/2})) = \frac{1}{d}\bY^{\top}\left(\bY\bM+\bG\bQ^{1/2}-\bR_{L,\bV}(\bY\bM+\bG\bQ^{1/2})\right)\bV^{-1}
\end{equation}
which prescribes, using Lemma \ref{partial_sol}
\begin{equation}
    \hat{\bM} \stackrel{\rm P}\simeq \frac{1}{d}\bY^{\top}\left(\bR_{L,\bV}(\bY\bM+\bG\bQ^{1/2})-\bY\bM+\bG\bQ^{1/2}\right)\bV^{-1}
\end{equation}
For $\bM$, we use the block decomposition from Eq.(\ref{block_res_reg}), which simplifies the pseudo-inverse $\bSigma^+$ in Eq.~\eqref{grad_m} to give, using Lemma \ref{partial_sol} again
\begin{equation}
    \bM \stackrel{\rm P}\simeq \frac{1}{\sqrt{d}}\bMM\boldsymbol{\eta}(\bH\hat{\bQ}^{1/2}\hat{\bV}^{-1}) 
\end{equation}
where the function $\boldsymbol{\eta}$ acts on the block diagonal and is defined by Eq.(\ref{def_eta}). Using those results and the definition of $\tilde{\bW}$, the solution $\bW^{\star}$ and the quantity $\bX\bW^{\star}$ are characterized, in the pseudo-Lipschitz sense of Theorem \ref{the:1}, by the fixed point of the system of equations (the first four equations are meant for all $1\leqslant k \leqslant K$):
\begin{align}
\label{final_AMP_SE}
    \bQ_{k} &= \lim_{d \to  +\infty}\frac{1}{d}\mathbb{E}\left[\be_{k}(\{\bH_{k}(\hat{\bQ}_{k})^{1/2}\hat{\bV}_{k}^{-1}\}_{k\in[K]})^{\top}\be_{k}(\{\bH_{k}(\hat{\bQ}_{k})^{1/2}\hat{\bV}_{k}^{-1}\}_{k\in[K]})\right]\in \mathbb{R}^{K \times K} \\
    \hat{\bQ}_{k} &= \lim_{d \to +\infty} \frac{1}{d}\mathbb{E}\left[\bh_{k}(\bG_{k}\bQ_{k}^{1/2})^{\top}\bh_{k}(\bG_{k}\bQ_{k}^{1/2})\right]\in \mathbb{R}^{K \times K} \\
    \bV_{k} &= \lim_{d\to +\infty}\frac{1}{d} \sum_{i=1}^{d}\mathbb{E}\left[\frac{\partial \be_{k}(\{\bH_{k}(\hat{\bQ}_{k})^{1/2}\}_{k\in[K]})}{\partial (\bH_{k}(\hat{\bQ}_{k})^{1/2})_{i} }\right] \in \mathbb{R}^{K \times K} \\
    \hat{\bV}_{k} &= -\lim_{d\to +\infty}\frac{1}{d} \sum_{i=1}^{n_{k}}\mathbb{E}\left[\frac{\partial \bh_{k,t}(\bG_{k}(\bQ_{k,t})^{1/2})}{\partial (\bG_{k}(\bQ_{k})^{1/2})_{i}}\right] \in \mathbb{R}^{K \times K} \\
    \bM &= \frac{1}{\sqrt{d}}\mathbb{E}\left[\bMM\boldsymbol{\eta}(\bH\hat{\bQ}^{1/2}\hat{\bV}^{-1})\right] \in \mathbb{R}^{K \times K} \\
     \hat{\bM} &= \frac{1}{d} \bY^{\top}\left(\bR_{L,\bV}(\bY\bM+\bG\bQ^{1/2})-\bY\bM+\bG\bQ^{1/2}\right)\bV^{-1} \in \mathbb{R}^{K \times K}
\end{align}
Using the explicit form of the different functions given in section \ref{build_AMP} and Stein's lemma for the derivatives, these equations match those of Theorem \ref{the:1}. This completes the proof.
\subsection{On the strict convexity assumption}
\label{app:strct}
If the optimization problem defining $\bW^{\star}$ is strictly convex, there is only one minimizer and the provided proof is enough. Additionally it is shown in \cite{tibshirani2013lasso} that for any loss function that is strictly convex in its argument and penalized with the $\ell_{1}$ norm, provided the data is sampled from a continuous distribution, the solution is unique with probability one regardless of the rank of the design matrix. Thus finding a point verifying the optimality condition of \eqref{student_1} is also enough to complete the proof. For generic convex (non-strictly) problems a more careful analysis could be performed in the same spirit as the one of \cite{bayati2011lasso}. Empirically the result still holds.
\subsection{On the uniqueness of the solution to the fixed point equations (\ref{final_AMP_SE})}
\label{app:unique}
It is possible to reconstruct Bregman envelopes on problem \eqref{student_1} for the loss and regularization as we have done for the loss in the previous section. We can then show that the fixed point equations \eqref{final_AMP_SE} are the optimality condition of a convex-concave problem involving both Bregman envelopes and linear combinations of the order parameters. In the same spirit as \cite{celentano2020lasso,Loureiro2021}, this problem should be asymptotically strictly convex. This is supported by the simulations presented in the experiments sections but left as an assumption in the main paper.

\subsection{Proof of Lemma \ref{lemma:conv_traj}}
\label{sec:proof_conv}
This proof follows a similar argument to the one used to control the trajectory of the AMP studied in \cite{bolthausen2014iterative}.
Note that, because of the way the AMP is initialized using the fixed point of the state evolution equations, for any $t\geqslant 1$ the following holds:
\begin{align}
\lim_{d \to +\infty}\frac{1}{d}\mathbb{E}\left[ \be(\bu^{t})^{\top}\be(\bu^{t})\right] = \bQ \in \mathbb{R}^{K^{2}\times K^{2}}\\
\lim_{d \to +\infty} \frac{1}{d}\mathbb{E}\left[\bh(\bv^{t})^{\top}\bh(\bv^{t})\right] = \bhQ \in \mathbb{R}^{K^{2}\times K^{2}}
\end{align}

where
\begin{align}
    \be(\bu^{t}) = (Id+\partial\tilde{r}(\bullet)\bhV^{-1})^{-1}(\bu^{t}\hat{\bV}^{-1}) \quad \bh(\bv^{t}) = \left(\left(Id+\partial\tilde{L}(\bullet)\bV\right)^{-1}(\bv^{t})-\bv^{t}\right)\bV^{-1}
\end{align}
then the limit we are looking for reads:
\begin{align}
    \lim_{d\to \infty} \frac{1}{d}\norm{\bu^{t}-\bu^{t-1}}^{2}_{F}=\lim_{d \to \infty}2(\hat{\bQ}-\frac{1}{d}\mathrm{tr}((\bu^{t})^{\top}\bu^{t-1}) \notag \\
    \lim_{d\to \infty} \frac{1}{d}\norm{\bv^{t}-\bv^{t-1}}^{2}_{F}=2(\bQ-\frac{1}{d}\mathrm{tr}((\bv^{t})^{\top}\bv^{t-1}) 
\end{align}
We thus need to study the correlation between successive iterates. At each time step, denote $(\hat{\bC}_{t},\bC_{t})$ in $\mathbb{R}^{K^{2}\times K^{2}}$ the correlation matrices between iterates at times $t,t-1$ describing the Gaussian fields respectively associated to $\bu^{t},\bv^{t}$ i.e., 
\begin{equation}
    \lim_{d \to \infty}\frac{1}{d}\mathrm{tr}((\bu^{t})^{\top}\bu^{t-1} = \hat{\bC}_{t} \quad
    \lim_{d \to \infty}\frac{1}{d}\mathrm{tr}((\bv^{t})^{\top}\bv^{t-1} = \bC_{t}
\end{equation} 
we can then write the block diagonal Gaussian fields $\hat{\bZ}^{t},\hat{\bZ}^{t-1},\bZ^{t},\bZ^{t-1}$ in $\mathbb{R}^{Kd\times K^{2}}$ and  in the following way
\begin{align}
\label{gauss_interp1}
   \hat{\bZ}^{t}\sim \bH(\hat{\bC}_{t})^{1/2}+\bH^{'}(\hat{\bQ}-\hat{\bC}_{t})^{1/2} \\
    \hat{\bZ}^{t-1} \sim \bH(\hat{\bC}_{t})^{1/2}+\bH^{''}(\hat{\bQ}-\hat{\bC}_{t})^{1/2} \\
    \bZ^{t}\sim \bG(\bC_{t})^{1/2}+\bG^{'}(\bQ-\bC_{t})^{1/2} \\
   \bZ^{t-1}\sim \bG(\bC_{t})^{1/2}+\bG^{''}(\bQ-\bC_{t})^{1/2}
  \label{gauss_interp2}
\end{align}

where the matrices $\bH,\bH',\bH''$ are in $\mathbb{R}^{Kd \times K^{2}}$, $\bG,\bG',\bG''$ are in $\mathbb{R}^{n\times K^{2}}$ and all have i.i.d. standard normal elements.
The recursion describing the evolution of these correlations then reads :
\begin{align}
	\bC_{t+1} = \frac{1}{d}\mathbb{E}\left[\be(\bH\hat{\bC}_{t}^{1/2}+\bH^{'}(\hat{\bQ}-\hat{\bC}_{t})^{1/2})^{\top}\be(\bH\hat{\bC}_{t}^{1/2}+\bH^{''}(\hat{\bQ}-\hat{\bC}_{t})^{1/2})\right] \\
	\hat{\bC}_{t} = \frac{1}{d}\mathbb{E}\left[\bh(\bG\bC_{t}^{1/2}+\bG^{'}(\bQ-\bC_{t})^{1/2})^{\top}\bh(\bG\bC_{t}^{1/2}+\bG^{''}(\bQ-\bC_{t})^{1/2})\right]
\end{align}
Integrating out the independent $\bH^{'},\bH^{''}$ first, we get
\begin{equation}
    \bC_{t+1} = \int_{\mathbb{R}^{Kd \times K^{2}}}d\mu(\bH)\mathbf{I}(H)^{\top}\mathbf{I}(\bH)
\end{equation}
where $\mathbf{I}(\bH) = \int_{\mathbb{R}^{Kd \times K^{2}}}d\mu(\bH^{'})\be(\bH\hat{\bC}_{t}^{1/2}+\bH^{'}(\hat{\bQ}-\hat{\bC}_{t})^{1/2})$.
So $\bC^{t}$ is symmetric positive definite, assuming the resolvents aren't trivial.
The same argument applied to $\hat{\bC}^{t}$ shows it is also symmetric positive definite.
From \cite{bauschke2003bregman}, the operators 
\begin{align}
    (Id+\partial\tilde{r}(\bullet)\bhV^{-1})^{-1}(\bullet) \quad \left(Id+\partial\tilde{L}(\bullet)\bV\right)^{-1}(\bullet)
\end{align}
are \emph{D-firm} w.r.t. the Bregman distances induced by the differentiable, strictly convex functions $\frac{1}{2}\mbox{tr}(X\hat{\bV}X^{\top})$ and $\frac{1}{2}\mbox{tr}(\bX\bV^{-1}\bX^{\top})$ respectively. Recall
\begin{align}
    \be(\bu^{t}) = (Id+\partial\tilde{r}(\bullet)\bhV^{-1})^{-1}(\bu^{t}\hat{\bV}^{-1}) \quad \bh(\bv^{t}) = \left(\left(Id+\partial\tilde{L}(\bullet)\bV\right)^{-1}(\bv^{t})-\bv^{t}\right)\bV^{-1}
\end{align}
Then, using the definition of \emph{D-firm} 
\begin{align}
    \langle \be(\hat{\bZ}^{t})-\be(\hat{\bZ}^{t-1}),\left(\be(\hat{\bZ}^{t})-\be(\hat{\bZ}^{t-1})\right)\hat{\bV} \rangle \leqslant \langle \be(\hat{\bZ}^{t})-\be(\hat{\bZ}^{t-1}),(\hat{\bZ}^{t}-\hat{\bZ}^{t-1})\hat{\bV}^{-1}\hat{\bV} \rangle
\end{align}

Adding the normalization by $\frac{1}{d}$, using the representation Eq.(\ref{gauss_interp1}-\ref{gauss_interp2}), taking expectations and applying the matrix form of Stein's lemma, see for example \cite{gerbelot2021graph} Lemma 12, we get: 
\begin{equation}
    \mbox{tr}((\bQ-\bC_{t+1})\hat{\bV}) \leqslant \mbox{tr}((\hat{\bQ}-\hat{\bC}_{t})\bV)
\end{equation}
Using a similar argument on $\bh$, we get
\begin{equation}
    \mbox{tr}((\hat{\bQ}-\hat{\bC}_{t})\bV) \leqslant \mbox{tr}((\bQ-\bC_{t})\hat{\bV})
\end{equation}
and
\begin{equation}
    \mbox{tr}(\bC_{t+1}\hat{\bV}) \geqslant \mbox{tr}(\bC_{t}\hat{\bV})
\end{equation}
thus the sequence $\mbox{tr}(\bC_{t+1}\hat{\bV})$ is a bounded (above) monotone (increasing) sequence, and therefore converges. Since $\hat{\bV}$ is positive definite and given the iteration defining $\bC_{t+1}$ from $\bC_{t}$, any fixed point of this iteration is a fixed point of $\mbox{tr}(\bC_{t}\hat{\bV})$. Assuming there is only one fixed point to the set of self-consistent equations Eq.\eqref{spgeneral} (see previous section), the proof is complete. (A similar argument can be carried out on $\hat{\bC}_{t}$).

\section{Replica computation}
\label{sec:app:replicas}
\label{app:replica}
\subsection{Setting of the problem}
In this Section we give a full derivation of the results in Theorem \ref{the:1} and Theorem \ref{th:2} by means of the replica approach, a standard method developed in the realm of statistical physics of disordered systems \cite{mezard1987spin}. In the general computation, we will consider the classification problem of $K$ clusters, assuming a dataset $\{(\bx^\nu,\by^\nu)\}_{\nu\in[n]}$ of $n$ independent datapoints where, as in the main text, the labels $\by$ takes value in a set of $K$ elements, $\by^\nu\in\{\be_k\}_k$, with $\be_k\in\mathbb R^L$. The elements of the dataset are independently generated by a mixture density in the form
\begin{equation}
 P(\bx,\by)
 = \sum_{k=1}^{K}\mathbb I(\by=\be_k)\rho_k\mathcal{N}\left(\bx\left|\bmu_k,\bSigma_k\right.\right), \quad \sum_{k=1}^K\rho_k=1.
\end{equation}
We will perform our classification task searching for a set of parameters $(\bW^\star,\bbb^\star)$ that will allow us to construct an estimator. The parameters will be chosen by minimising an empirical risk function in the form
\begin{equation}
\mathcal R(\bW,\bbb)\equiv \sum_{\nu=1}^n\ell\left(\by^\nu,\frac{\bW\bx^\nu}{\sqrt d}+\bbb\right)+\lambda r(\bW),
\end{equation}
i.e., they are given by
\begin{equation}
(\bW^{\star},\bbb^{\star}) \equiv \Argmin_{\bW\in \mathbb{R}^{L\times d},\,\bbb \in \mathbb{R}^{L}}\mathcal R(\bW,\bbb).  
\end{equation}
We will say that $\bW\in\mathbb R^{L\times d}$ and $\bbb\in\mathbb R^{L}$ are the weights and bias to be learned respectively, $\ell$ is a convex loss function with respect to its second argument, and $r$ is a regularisation function whose strength is tuned by the parameter $\lambda\geq0$. Finally, we will assume that a classifier $\bvarphi\colon\mathbb R^L\to \{\be_k\}_k$ is given, such that, once $(\bW^{\star},\bbb^{\star})$ are obtained, a new point $\bx$ is assigned to the label
\begin{equation}
\bx\mapsto \bvarphi\left(\frac{\bW^\star\bx}{\sqrt d}+\bbb^\star\right)\in\{\be_k\}_k.
\end{equation}

The described setting is slightly more general than the one given in Theorem \ref{the:1}. As a consequence of the fact that we choose $L$-dimensional labels, the order parameters that appear in the computation are $L$ dimensional vectors or $L\times L$ matrices. A typical ``high-dimensional encoding'' is the one-hot encoding convention adopted in Theorem \ref{the:1}, where $L=K$ and $\{\be_k\}_k\subset\mathbb R^K$ is the canonical basis of $\mathbb R^K$. In this case, the adopted classifier is
\begin{equation}
\bvarphi(\bx)\equiv \hat\by(\bx), \quad \hat y_k(\bx)=\mathbb I(\max_\kappa x_\kappa=x_k).
\end{equation}
Assuming \textit{scalar} labels $\{e_k\}_k\in \mathbb R$, we deal with scalar order parameters. For example, in the case of binary classification ($K=2$) it is common to adopt $L=1$ and $\{e_1,e_2\}=\{+1,-1\}$. In this case $\varphi(x)=\mathrm{sign}(x)$, see also Section \ref{sec:app:simpli:k2}.
\subsection{Gibbs minimisation}
The problem stated in Section \ref{sec:intro} is formulated as an optimisation problem. We can tackle such optimisation problem introducing a Gibbs measure over the weights $(\bW,\bbb)$, namely
\begin{equation}
\mu_{\beta}(\bW,\bbb)\propto e^{-\beta\mathcal R(\bW,\bbb)}=
\underbrace{e^{-\beta r(\bW)}}_{P_{w}(\bW)}\prod\limits_{\nu=1}^{n}\underbrace{\exp\left[-\beta\ell\left(\by^{\nu},\frac{\bW\bx^\nu}{\sqrt d}+\bbb\right)\right]}_{P_{y}(\by|\bW,\bbb)}.
\end{equation}
The parameter $\beta> 0$ is introduced for convenience: in the $\beta\to+\infty$ limit, the Gibbs measure concentrates on the values $(\bW^\star,\bbb^\star)$ which minimize the empirical risk $\mathcal R(\bW,\bbb)$ and are therefore the goal of the learning process. The functions $P_{y}$ and $P_{w}$ can be interpreted as a (unnormalised) likelihood and prior distribution respectively. Our analysis will go through the computation of the average free energy density associated to such Gibbs measure, i.e.,
\begin{equation}
f_\beta=-\lim_{\substack{n,d\to+\infty\\\nicefrac{n}{d}=\alpha}}\mathbb E_{\{(\bx,\by)\}}\left[\frac{\ln \mathcal Z_\beta}{d\beta}\right],
\end{equation}
where $\mathbb E_{\{(\bx,\by)\}}[\bullet]$ is the average over the training dataset, and we have introduced the partition function
\begin{equation}
\mathcal Z_\beta\equiv \int e^{-\beta\mathcal R(\bW,\bbb)}\dd\bW
\end{equation}
To perform the computation of such quantity, we use the so-called replica method, i.e., we compute
\begin{equation}
-\lim_{\substack{n,d\to+\infty\\\nicefrac{n}{d}=\alpha}}\mathbb E_{\{(\bx,\by)\}}\left[\frac{\ln \mathcal Z_\beta}{d\beta}\right]=\lim_{\substack{n,d\to+\infty\\\nicefrac{n}{d}=\alpha}}\lim_{s\to 0}\frac{1-\mathbb E_{\{(\bx,\by)\}}[\mathcal Z_\beta^s]}{sd\beta},
\end{equation}
\subsection{Replica approach}
We proceed in our calculation considering the bias vector assuming no prior on $\bbb$, which will play a role of an extra parameter. The equations for the bias $\bbb$ will be derived extremising with respect to it the final result for the free energy. We need to evaluate
\begin{equation}
\mathbb E_{\{(\bx,\by)\}}[\mathcal Z_\beta^s]=\prod_{a=1}^s\int \dd \bW^a P_w(\bW^a)\left(\sum_k \rho_k\mathbb E_{\bx|\by=\be_k}\left[\prod_{a=1}^s P_y\left(\be_k\left|\frac{\bW^a\bx}{\sqrt d}+\bbb\right.\right)\right]\right)^n.
\end{equation}
Let us take the inner average introducing a new variable $\beeta$,
\begin{multline}
\mathbb E_{\bx|\by=\be_k}\left[\prod_{a=1}^s P_y\left(\be_k\left|\frac{\bW^a\bx}{\sqrt d}+\bbb\right.\right)\right]=\prod_{a=1}^s\int\dd\beeta^a P_y(\be_k|\beeta^a)\mathbb E_{\bx}\left[\prod_{a=1}^s\delta\left(\beeta^a-\frac{\bW^a\bx}{\sqrt d}+\bbb\right)\right]\\
=\prod_{a=1}^s\int\dd\beeta^a P_y(\be_k|\beeta^a)\mathcal N\left(\beeta\Big|\frac{\bW^a\bmu_k}{\sqrt d}-\bbb;\frac{\bW^a\bSigma_k{\bW^b}^\top}{d}\right).
\end{multline}
We can write then
\begin{multline}
\mathbb E_{\{(\bx,\by)\}}[\mathcal Z_\beta^s]=\\
=\prod_{a=1}^n\int \dd \bW^a P_w(\bW^a)\left(\sum_k \rho_k\prod_{a=1}^s\int\dd\beeta^a P_y(\be_k|\beeta^a)\mathcal N\left(\beeta;\frac{\bW^a\bmu_k}{d}+\bbb;\frac{\bW^a\bSigma_k{\bW^b}^\top}{d}\right)\right)^n
\\
=\left(\prod_{k=1}^K\prod_{a\leq b}\iint\frac{\dd {\bQ}^{ab}_k\dd\hat{\bQ}_k^{ab}}{(2\pi)^{L^2/2}}\right)
\left(\prod_{k}\prod_{a}\int\frac{\dd\bM_k^a\dd\hat{\bM}_k^a}{(2\pi)^{L/2}}\right)e^{-d\beta\Phi^{(s)}}.
\end{multline}
where we introduced the \textit{order parameters}
\begin{align}
\bQ_k^{ab}&=\frac{\bW^a\bSigma_k\bW^{b\top}}{d}\in \mathbb R^{L\times L},\quad a,b=1,\dots, s,\\
\bM^{a}_k&=\frac{\bW^a\bmu_k}{\sqrt d}\in \mathbb R^{L},\quad a=1,\dots,s,
\end{align}
and the replicated free-energy
\begin{multline}
\beta \Phi^{(s)}(\bQ,\bM,\bhQ,\bhM,\bbb)=\sum_{k=1}^K\sum_{a}\hat {\bM}_k^{a\top} \bM_k^{a}+\sum_{k=1}^K\sum_{a\leq b}\tr[\hat{\bQ}_k^{ab\top}{\bQ}_k^{ab}]\\
-\frac{1}{d}\ln\prod_{a=1}^s\int P_w(\bW^a)\dd\bW^a \prod_k    \left(\prod_{a\leq b}e^{\tr[\hat {\bQ}_k^{ab\top} \bW^a\bSigma_k \bW^{b\top}]}\prod_a e^{\sqrt d{\hat{\bM}^{a\top}_k}\bW^a\bmu_k}\right)\\ -\alpha\ln\sum_k \rho_k\prod_{a=1}^s\int\dd\beeta^a P_y(\be_k|\beeta^a)\mathcal N\left(\beeta\big|\bM_k^a+\bbb,\bQ_k^{ab}\right).
\end{multline}
At this point, the free energy $f_\beta$ should be computed extremisizing with respect to all the order parameters by virtue of the Laplace approximation (in addition to $\bbb$),
\begin{equation}
f_\beta=\lim_{s\to 0}\Extr_{\{\bM,\bQ,\bhM,\bhQ\},\bbb}\frac{\Phi^{(s)}(\bQ,\bM,\bhQ,\bhM,\bbb)}{s}.
\end{equation}
However, the convexity of the problem allows us to make an important simplification.
\paragraph{Replica symmetric ansatz} Before taking the $s\to 0$ limit we make the assumptions
\begin{equation}
\begin{split}
{\bQ}_{k}^{aa}&=\begin{cases}\bR_k,&a=b\\ \bQ_k&a\neq b\end{cases}\\
\bM_{k}^{a}&=\bM_k
\end{split}\qquad
\begin{split}
{\bhQ}_{k}^{aa}&=\begin{cases}-\frac{1}{2}\bR_k,&a=b\\ \bhQ_k&a\neq b\end{cases}\\
\hat{\bM}_k^{a}&=\hat{\bM}_k\quad \forall a
\end{split}
\end{equation}
This ansatz is justified by the fact that we are assuming $\ell$ and $r$ to be convex, and $\lambda>0$. In this case, the problem admit one solution only that, therefore, coincide with the replica symmetric solution, in which overlaps between two replicas do not depend on the chosen replicas. By means of the replica symmetric hypotesis, we can write
\begin{equation}
\bQ^{ab}_k\mapsto \bsQ_k\equiv\bI_{s,s}\otimes (\bR_k-\bQ_k)+\bOne_s\otimes \bQ_k.
\end{equation}
The inverse matrix is therefore
\begin{equation}
\bsQ_k^{-1}=  \bOne_s\otimes(\bR_k-\bQ_k)^{-1}-\bI_{s,s}\otimes[(\bR_k+(s-1)\bQ_k)^{-1}\bQ_k(\bR_k-\bQ_k)^{-1}],
\end{equation}
whereas
\begin{equation}
\begin{split}
\det\bsQ_k&=\det(\bR_k-\bQ_k)^{s-1}\det(\bR_k+(s-1)\bQ_k)\\
&=1+s\ln\det(\bR_k-\bQ_k)+s\tr\left[(\bR_k-\bQ_k)^{-1}\bQ_k\right]+o(s).
\end{split}
\end{equation}
If we denote $\bV_k\equiv \bR_k-\bQ_k$
\begin{multline}
\ln\sum_k \rho_k\prod_{a=1}^s\int\dd\beeta^a P_y(\be_k|\beeta^a)\mathcal N\left(\beeta\big|\bM_k^a+\bbb,\bQ_k^{ab}\right)\\
=s\sum_k \rho_k\mathbb E_{\bxi}\ln\left(\int\frac{\dd\beeta P_y(\be_k|\beeta)}{\sqrt{\det(2\pi \bV_k)}}e^{-\frac{1}{2}(\beeta-\bM_k-\bbb-\bQ_k^{1/2}\bxi)^\top\bV^{-1}_k (\beeta-\bbb-\bM_k-\bQ_k^{1/2}\bxi)}\right)+o(s)\\
=s\sum_k \rho_k\mathbb E_{\bxi}\left[\ln Z\left(\be_k,\bM_k+\bbb+\bQ_k^{1/2}\bxi,\bV_k\right)\right]+o(s),
\end{multline}
with $\bxi\sim\mathcal N(\mathbf 0,\bI_L)$ is a normally distributed vector and we have introduced the function
\begin{equation}
Z\left(\be_k,\bM,\bV\right)\equiv \int\frac{\dd\beeta P_y(\be_k|\beeta)}{\sqrt{\det(2\pi \bV)}}e^{-\frac{1}{2}(\beeta-\bM)^\top\bV^{-1} (\beeta-\bM)}
\end{equation}
On the other hand, denoting by $\bhV_k=\bhR_k+\bhQ_k$,
\begin{multline}\comprimi
    \frac{1}{d}\ln\prod_{a=1}^s\left(\int P_w(\bW^a)\dd \bW^a \prod_k e^{-\frac{1}{2}\tr[\hat{\bV}_k^\top \bW^a\bSigma_k(\bW^a)^\top]+\sqrt d{\hat{\bM}}_k^\top\bW^a\bmu_k}\prod_{b,k}e^{\frac{1}{2}\tr[\hat {\bQ}_k \bW^a\bSigma_k(\bW^b)^\top]}\right)=\\
    \comprimi
    =\frac{s}{d}\mathbb E_{\bXi}\ln\left[\int P_w(\bW)\dd\bW \prod_k\exp\left(-\frac{\tr[\hat{\bV}_k^\top \bW\bSigma_k\bW^\top]}{2}+\sqrt d{\hat{\bM}}_k^\top\bW\bmu_k+\Xi_k\odot \sqrt{\bhQ_k\otimes\bSigma_k}\odot \bW\right)\right]\\ +o(s).  
\end{multline}
In the expression above we have used the tensorial product $\bhQ\otimes\bSigma=(\hat Q_{kk'}\Sigma_{ij})_{ki,k'j'}$. Given a matrix $\bbB\in\mathbb R^{L\times d}$ and the tensors $\bsA,\bsA'\in\mathbb R^{L\times d}\otimes\mathbb R^{L\times d}$, we denote $(\bbB\odot \bsA)_{ki}\equiv \sum_{k'i'}B_{k'i'}A_{k'i'\,ki}\in\mathbb R^{L\times d}$, $(\bsA\odot \bbB)_{ki}\equiv \sum_{k'i'}A_{ki\,k'i'}B_{k'i'}\in\mathbb R^{L\times d}$ and $(\bsA\odot\bsA')_{ki\,k'i'}=\sum_{\kappa j}A_{ki\,\kappa j}A_{\kappa j\,k'i'}$. In this way, we define $\sqrt{\bsA}$ as the tensor such that $\bsA=\sqrt{\bsA}\odot \sqrt{\bsA}$. Finally, we have also introduced a set of $k$ matrices $\bXi_k\in\mathbb R^{L\times d}$ with i.i.d.~random Gaussian entries with zero mean and variance $1$, and the average over them $\mathbb E_{\bXi}[\bullet]$. Therefore, the (replicated) \textit{replica symmetric} free-energy is given by
\begin{align}
\begin{split}
\lim_{s\to 0}\frac{\beta}{s}\Phi^{(s)}_{\rm RS}=&\sum_{k=1}^K\hat {\bM}_k^\top \bM_k+\frac{1}{2}\sum_{k=1}^K\tr[\bhV_k^\top\bQ_k]-\frac{1}{2}\sum_{k=1}^K\tr[\bhQ_k^\top\bV_k]-\frac{1}{2}\sum_{k=1}^K\tr[\bhV_k^\top\bV_k]\\
&-\alpha\beta\Psi_{\text{out}}(\bM,\bQ,\bV)-\beta\Psi_w(\bhM,\bhQ,\bhV)
\end{split}
\end{align}
where we have defined two contributions
\begin{align}
\Psi_{\text{out}}(\bM,\bQ,\bV)&\equiv \beta^{-1}\sum_k \rho_k\mathbb E_{\bxi_k}\ln Z\left(\be_k,\bomega_k,\bV_k\right)\\
\Psi_w(\bhM,\bhQ,\bhV)&\equiv \frac{1}{\beta d}\mathbb E_{\bxi}\ln\left(\int P_w(\bW)\dd\bW \prod_ke^{-\frac{\tr[\bhV_k^\top\bW\bSigma_k\bW^\top]}{2}+\sqrt d\bhM_k^\top \bW\bmu_k+\bXi_k\odot\sqrt{\bhQ_k\otimes\bSigma_k}\odot \bW}\right)
\end{align}
and introduced, for future convenience,
\begin{equation}
\bomega_k\equiv \bM_k+\bbb+\bQ_k^{1/2}\bxi_k.
\end{equation}
Note that we have separated the contribution coming from the chosen loss (the so-called \textit{channel} part $\Psi_{\rm out}$) from the contribution depending on the regularisation (the \textit{prior} part $\Psi_w$). To write down the saddle-point equations in the $\beta\to+\infty$ limit, let us first rescale our order parameters as $\bhM_k\mapsto \beta\bhM_k$, $\bhQ_k\mapsto \beta^2\bhQ_k$, $\bhV_k\mapsto \beta\bhV_k$ and $\bV_k\mapsto\beta^{-1}\bV_k$. For $\beta\to+\infty$ the channel part is
\begin{equation}
\Psi_{\rm out}(\bM,\bQ,\bV)=-\sum_k \rho_k\mathbb E_{\bxi}\left[\mathcal M_{\ell(\be_k,\bV_k^{1/2}\bullet)}\left(\bV_k^{-1/2}\bomega_k\right)\right].    
\end{equation}
Here and in the following the quantity
\begin{equation}
\mathcal{M}_{f(\bullet )}(\mathbf u)\equiv\min_{\mathbf v\in{\rm domain}(\mathbf v)}\left[\frac{1}{2}\|\mathbf v-\mathbf u\|^2_{\rm F}+f(\mathbf v)\right]
\end{equation}
is the Moreau envelope of $f\colon{\rm domain}(\mathbf v)\to \mathbb R$, whereas $\|\bullet\|_{\rm F}$ is the Frobenius norm. We can write the contribution $\Psi_{\rm out}$ in terms of a proximal
\begin{equation}\comprimi
\bh_k=\bV_k^{1/2}\Prox_{\ell(\be_k,\bV_k^{1/2}\bullet)}(\bV_k^{-1/2}\bomega_k)\equiv \bV_k^{1/2}\arg\min_{\mathbf u\in\mathbb R^L}\left[\frac{1}{2}\|\mathbf u-\bV_k^{-1/2}\bomega_k\|_{\rm F}^2+\ell(\be_k,\bV^{1/2}_k\mathbf u)\right].
\end{equation}
as
\begin{equation}
\Psi_{\rm out}(\bM,\bQ,\bV)=-\sum_k\rho_k \mathbb E_{\bxi}\left[\frac{1}{2}\|\bV_k^{-1/2}\bh_k-\bV_k^{-1/2}\bomega_k\|_{\rm F}^2+\ell(\be_k,\bh_k)\right]
\end{equation}

A similar expression can be obtained for $\Psi_w$. Defining
\begin{equation}
\bsA=\left(\sum_k\bhV_k\otimes \bSigma_k\right)^{-1},\qquad \bbB=\sqrt d\sum_k\bmu_k\bhM_k^\top+\sum\limits_k\bXi_k\odot \sqrt{\bhQ_k\otimes\bSigma_k}.
\end{equation}
$\Psi_w$ can be written as
\begin{multline}
\Psi_w(\bhM,\bhQ,\bhV)
=\frac{1}{2d}  \mathbb E_{\bxi}\left[\bbB\odot \bsA\odot \bbB\right]\\+\frac{1}{\beta d}  \mathbb E_{\bxi}\ln\left[\int\dd\bW \exp\left(-\frac{\beta}{2}\|\bsA^{-1/2}\odot \bW-\bsA^{1/2}\odot \bbB\|_{\rm F}^2-\beta r(\bW)\right)\right].  
\end{multline}
It follows that, for $\beta\to+\infty$,
\begin{equation}
\Psi_w(\bhM,\bhQ,\bhV)
=\frac{1}{2d}  \mathbb E_{\bxi}\left[\bbB\odot \bsA\odot \bbB\right]-\frac{1}{d}\mathbb E_{\bxi}\left[\mathcal{M}_{r(\bsA^{1/2}\odot \bullet)}(\bsA^{1/2}\odot\bbB)\right].
\end{equation}
As before, let us introduce the proximal
\begin{equation}
\bG=\bsA^{1/2}\odot \Prox_{r(\bsA^{1/2}\odot \bullet)}(\bsA^{1/2}\odot \bbB)\in\mathbb R^{L\times d}  
\end{equation}
We can rewrite the prior contribution $\Psi_w$ as
\begin{equation}
\Psi_w(\bhM,\bhQ,\bhV)
=\frac{1}{2d}  \mathbb E_{\bXi}\left[\bbB\odot \bsA\odot \bbB\right]-\frac{1}{d}\mathbb E_{\bXi}\left[\frac{\|\bsA^{-1/2}\odot \bG-\bsA^{1/2}\odot \bbB\|_{\rm F}^2}{2}+r(\bG)\right].
\end{equation}
The parallelism between the two contributions is evident, aside from the different dimensionality of the involved objects. The replica symmetric free energy in the $\beta\to+\infty$ limit is computed extremising with respect to the introduced order parameters, 
\begin{multline}\label{eq:app:rsfree}
f_{\rm RS}=\Extr_{\substack{\bM,\bQ,\bV,\bbb\\\bhM,\bhQ,\bhV}}\left[\sum_{k=1}^K\hat {\bM}_k^\top \bM_k+\frac{1}{2}\sum_{k=1}^K\tr[\bhV_k^\top\bQ_k]-\frac{1}{2}\sum_{k=1}^K\tr[\bhQ_k^\top\bV_k]\right.\\\left.-\frac{1}{2}\sum_{k=1}^K\tr[\bhV_k^\top\bV_k]-\alpha\Psi_{\text{out}}(\bM,\bQ,\bV)-\Psi_w(\bhM,\bhQ,\bhV)\right].
\end{multline}
To do so, we have to write down a set of saddle-point equations and solve them.

\paragraph{Saddle-point equations} The saddle-point equations are derived straightforwardly from the obtained free energy extremising with respect to all parameters. A first set of equations is obtained from $\Psi_{\rm out}$ as\footnote{To obtain the equation for $\bhV$ it is convenient to use Stein's lemma, so that $\mathbb E[\partial_{\xi}\bff_k]=\mathbb E[\bff_k\bxi^\top]$.}
\begin{subequations}\label{app:eq:sphat}
\begin{align}
\bhQ_k&=\alpha \rho_k\mathbb E_{\bxi}\left[\bff_k\bff_k^\top\right],\\
\bhV_k&=\displaystyle-\alpha \rho_k\bQ_{k}^{-1/2}\mathbb E_{\bxi}\left[\bff_k\bxi^\top\right],\\
\bhM_k&=\alpha \rho_k\mathbb E_{\bxi}\left[\bff_k\right],\\
\bbb&=\sum_k \rho_k\mathbb E_{\bxi}\left[\bh_k-\bM_k\right]\Longleftrightarrow \sum_{k}\rho_k \mathbb E_{\bxi}\left[\bV_k\bff_k\right]=\mathbf 0.
\end{align}
\end{subequations}
where for brevity we have denoted
\begin{equation}
\bff_k\equiv \bV^{-1}_{k}(\bh_k-\bomega_k).
\end{equation}
Similarly, the saddle-point equations from $\Psi_{\rm out}$ are
\begin{subequations}\label{app:eq:sp}
\begin{align}
	\bV_k &=\frac{1}{d}\mathbb{E}_{\bXi}\left[\left(\bG\odot \left(\bhQ_k\otimes\bSigma_k\right)^{-1/2}\odot (\bI_k\otimes\bSigma_k)\right)\bXi_k^\top\right]\\
	\bQ_k &=\frac{1}{d} \mathbb{E}_{\bxi}\left[\bG\bSigma_k\bG^\top\right]\\
	\bM_k &= \frac{1}{\sqrt d}\mathbb{E}_{\bxi}\left[\bG\bmu_k\right].
\end{align}
\end{subequations}
To obtain the replica symmetric free energy, therefore, the given set of equation has to be solved, and the result then plugged in Eq.~\eqref{eq:app:rsfree}. No further simplification can be obtained in the most general setting. We will explore however some simple (but important) applications in Appendix \ref{sec:app:simpli}. Before going on, however, it is important to express the relevant quantities for learning, i.e., the training and generalization errors, in terms of the obtained order parameters.

\subsection{Training and test errors}\label{app:replica:errors}
The order parameters introduced to solve the problem allow us to reach our ultimate goal of computing the average errors of the learning process. We will start from the estimation of the training loss. The complication in computing this quantity is that the order parameters found in the learning process are, of course, correlated with the dataset used for the learning itself. We need to compute
\begin{equation}
\epsilon_\ell\equiv \frac{1}{n}\sum\limits_{\nu=1}^{n}\ell\left(\by^{\nu}, \frac{\bW^\star\bx^\nu}{\sqrt d}+\bbb^\star\right)
\end{equation}
in the $n\to+\infty$ limit. Denoting for brevity $\ell_k(\bx)\equiv \ell(\be_k,\bx)$, the best way to proceed is to observe that $\mathbb E_{\{(\by^\nu,\bx^\nu)\}_{\nu}}[\mathcal R(\bW^\star,\bbb^\star)]=-\lim_{\beta\to+\infty}\mathbb E_{\{(\by^\nu,\bx^\nu)\}_{\nu}}[\partial_\beta\ln\mathcal Z_\beta]=\lambda\mathbb E_{\{(\by^\nu,\bx^\nu)\}_{\nu}}[r(\bW^\star)]+\epsilon_\ell$, where
\begin{multline}
\epsilon_\ell=-\lim_{\beta\to+\infty}\partial_\beta(\beta\Psi_{\rm out})=\lim_{\beta\to+\infty}\sum_k\rho_k\int \ell_k(\beeta)\frac{e^{-\frac{\beta}{2}(\beeta-\bM^\star_k)^\top{\bV^\star_k}^{-1}(\beeta-\bM^\star_k)-\beta\ell_k(\beeta)}}{\sqrt{\det(2\pi\beta^{-1}\bV^\star)}Z(\be_k,\bomega_k^\star,\beta^{-1}\bV_k^\star)}
\dd\beeta.
\end{multline}
In the $\beta\to+\infty$ limit, the integral concentrates on the minimizer of the exponent, that is, by definition, the proximal $\bh_k$. In conclusion, $\epsilon_\ell=\sum_k\rho_k\mathbb E[\ell(\bh_k)]$. By means of the same concentration result, the training error is
\begin{equation}
\epsilon_t=\frac{1}{n}\sum_{\nu=1}^n\mathbb I\left(\bvarphi\left(\frac{\bW^\star\bx^\nu}{\sqrt d}+\bbb^\star\right)\neq\by^\nu\right)\xrightarrow{n\to+\infty}\sum_{k=1}^K \rho_k\mathbb E_{\bxi}\left[\mathbb I(\bvarphi(\bh_k)\neq \be_k)\right].
\end{equation}
The expressions above hold in general, but, as anticipated, important simplifications can occur in the set of saddle-point equations \eqref{app:eq:sphat} and \eqref{app:eq:sp} depending on the choice of the loss $\ell$ and of the regularization function $r$.

The generalisation (or test) error can be written instead as
\begin{equation}
\epsilon_g=\mathbb E_{\by^{\rm new},\bx^{\new}}\left[\mathbb I\left(\bvarphi\left(\frac{\bW^\star\bx^{\rm new}}{\sqrt d}+\bbb^\star\right)\neq \by^{\rm new}\right)\right].
\end{equation}
This expression can be rewritten as
\begin{equation}
\epsilon_g=\sum_k\rho_k \int\mathbb I(\bvarphi(\beeta)=\be_k)\mathbb E_{\bx^{\rm new}}\left[\delta\left(\beeta-\frac{\bW^\star\bx^{\rm new}}{\sqrt d}-\bbb^\star\right)\right]\dd\beeta
\end{equation}
Once again, we write
\begin{equation}
    \mathbb E_{\bx^{\rm new}}\left[\delta\left(\beeta-\frac{\bW^\star\bx^{\rm new}}{\sqrt d}-\bbb^\star\right)\right]\xrightarrow{d\to+\infty}\mathcal N(\beeta|\bM^\star_k+\bbb^\star,\bQ_k^\star)
\end{equation}
so that
\begin{equation}
\epsilon_g= \sum_{k=1}^K\rho_k\mathbb E_{\bxi}\left[\mathbb I\left(\bvarphi\left(\bM_k^\star+{\bQ^\star_k}^{1/2}\bxi+\bbb^\star\right)\neq \be_k\right)\right].
\end{equation}
This can be easily computed numerically once that the order parameters are given.

\subsection{A note on the numerical integration of the saddle-point equations}\label{app:numint}
To estimate $\epsilon_g$, $\epsilon_t$ and $\epsilon_\ell$ we first need to find the fixed-point solutions of the saddle-point equations \eqref{app:eq:sphat} and \eqref{app:eq:sp}. The simplest numerical strategy consists in updating, in a self-consistent way, the order parameters until their variation according to, e.g., the Frobenius norm is smaller than a given threshold value (that we adopted to be $10^{-5}$). In the simplest setting, i.e., the one discussed in Corollary \ref{corr:1}, the update of $(\bM_k,\bQ_k,\bV_k)_{k\in[K]}$ is performed explicitly using eq.~\eqref{eq:corr1}, where $\mathbb E_{\boldsymbol\sigma,\bmu}[\bullet]$ is a shorthand for the sum over the eigenvalues and eigenvectors of the assigned covariance matrices. The update of $(\bhM_k,\bhQ_k,\bhV_k)_{k\in [K]}$
 (right hand side of eq.~\eqref{spgeneral}) is more involved, as it requires the computation of the proximal 
 followed by a Gaussian average. Such average has been performed using a Monte Carlo strategy, i.e., by solving the equation for the proximal for a large number ($10^4-10^5$) of instances of $\bxi$ and averaging the solution. We remark that in the case of the square loss, the proximal can be computed analytically and the integration can be performed explicitly, highly simplifying the fixed-point equations (see below eq.~\eqref{eq:spsquare}). We have found that in practice fluctuations due to the adopted Monte Carlo pool were small enough to be negligible compared with the outcomes of direct numerical experiments.

The convergence to the the correct fixed point is guaranteed (in principle) by the convexity of the problem. However, a few delicate aspects have to be taken into account in the update process described above.
\begin{enumerate}
    \item The update requires the computation of the proximals $\bG$ and $\bh_k$. Such computations can be performed analytically in some specific cases only (for example, in the case of ridge regression). The existence of a unique solution is guaranteed by the strong convexity of the problem defining the proximal. In our study of the cross-entropy loss function, for example, we computed the proximals $\bh_k$ numerically solving Eq.~\eqref{app:eq:proxce}. In this problem, however, additional numerical instabilities emerged in the $\lambda\to 0$ limit, due the fact that the discontinuity in the gradient appear, see Eq.~\eqref{app:eq:lam0}. We solved this issue performing an annealing in $\lambda$, i.e., solving for the proximal for decreasing values of the regularization strength.
    \item The numerical solution of the saddle-point equations might suffer numerical instabilities due to the operations of inversion involved, see, e.g., the equation for $\bhV_k$ in \eqref{app:eq:sphat}, which requires the inversion of $\bQ_k$. It is convenient, in such cases, to rewrite the equation in an equivalent form which is numerically more stable. For example, in the aforementioned equation, we can observe that $\bff_k$ satisfies the equation $\bff_k+\partial_{\vec x}\ell_k(\bV_k\bff_k+\bomega_k)=\mathbf 0$ so that $\partial_{\bomega_k}\bff_k=-(\bI_K+\partial^2_{\vec x}\ell_k(\bV_k\bff_k+\bomega_k)\bV_k)^{-1}\partial^2_{\vec x}\ell_k(\bV_k\bff_k+\bomega_k)$. Using Stein's lemma,
\begin{equation}
\bhV_k=-\alpha \rho_k\mathbb E_{\bxi}\left[\partial_\bxi\bff_k\right]=\alpha \rho_k\mathbb E_{\bxi}\left[\left(\bI_K+\partial^2_{\vec x}\ell_k(\bV_k\bff_k+\bomega_k)\bV_k\right)^{-1}\partial^2_{\vec x}\ell_k(\bV_k\bff_k+\bomega_k)\right].
\end{equation}
We found this equation numerically more stable than the one given in \eqref{app:eq:sphat} when dealing with the cross-entropy loss.
\end{enumerate}
Our implementation can be found at \cite{github}.

\section{Some relevant particular cases}
\label{sec:app:simpli}
In this Appendix, we will specify the saddle-point equations for the multiclass classification problem for different choices of the loss function $\ell$ and of the regularisation function $r$. From the analysis developed in the previous Appendices, it is clear that the choices of $\ell$ and $r$ impact separately the set of equations \eqref{app:eq:sphat} and \eqref{app:eq:sp} respectively. Once the order parameters are found, it is possible to estimate the training and generalisation errors as, for example, in Section~\ref{app:replica:errors}.

{\subsection{\boldmath The case of $\ell_2$ regularization}
In this Section we consider the relevant case of quadratic regularization, $r(\bW)=\nicefrac{1}{2}\|\bW\|_{\rm F}^2$. In this case the computation of $\Psi_w$ can be performed explicitly via a Gaussian integration,
\begin{equation}
\frac{1}{\beta}\Psi_w(\bhM,\bhQ,\bhV)
=\frac{1}{2d}\tr\ln\bsS-\frac{K\ln\beta}{2\beta}
+\frac{1}{2}\tr\left[\bsS\odot \left(\sum_{kk'}\bhM_k\bhM^\top_{k'}\otimes\bmu_k\bmu^\top_{k'}+\frac{1}{d}\sum_k\bhQ_k\otimes\bSigma_k\right)\right].
\end{equation}
Here we have introduced, for notation compactness,
\begin{equation}
\bsS\equiv\left(\lambda\bI_K\otimes \bI_d+\sum_{\kappa}\bhV_{\kappa}\otimes\bSigma_{\kappa}\right)^{-1}.
\end{equation}
This form of $\Psi_w$ allows us to write in a simpler way the set of eqs.~\eqref{app:eq:sp}, that can be re-written as
\begin{equation}
\begin{split}
\bQ_k&=\tr_d\left[(\bI_K\otimes\bSigma_k)\odot \bsS\odot \left(\sum_{kk'}\bhM_k\bhM^\top_{k'}\otimes\bmu_k\bmu^\top_{\kappa'}+\frac{1}{d}\sum_{\kappa}\bhQ_{\kappa}\otimes\bSigma_{\kappa}\right)\odot \bsS\right] \\
\bM_k&=\sum_{k'}\tr_d\left[\bsS\odot  \left(\bhM_{k'}\otimes \bmu_{k'}\bmu_{k}^\top\right)\right]\\
\bV_k&=\frac{1}{d}\tr_d\left[(\bI_K\otimes\bSigma_k)\odot \bsS\right].
\end{split}
\end{equation}
In the previous equations, by $\tr_d$ we denoted the trace with respect to the components living in the $d$-dimensional space of the dataset.

\paragraph{Jointly diagonal covariances --- } Suppose now that $\bSigma_k=\sum_i\sigma^k_i\bv_i\bv_i^\top$ for all $k$, i.e., the covariance matrices share the same basis of eigenvectors $\{\bv_i\}_i$. Then, denoting $\mu_i^k\equiv \sqrt d\bmu_k^\top \bv_i$
\begin{equation}\comprimi
\begin{split}
\bQ_k&=\frac{1}{d}\sum_{i=1}^d\sigma_i^k\left(\lambda\bI_K+\sum_{\kappa}\sigma_i^\kappa\bhV_{\kappa}\right)^{-1}\left(\sum_{kk'}\mu^{k}_i\mu^{k'}_i\bhM_k\bhM^\top_{k'}+\sum_{\kappa}\sigma_i^\kappa\bhQ_{\kappa}\right)\left(\lambda\bI_K+\sum_{\kappa}\sigma_i^\kappa\bhV_{\kappa}\right)^{-1} \\
\bM_k&=\frac{1}{d}\sum_{i=1}^d\sum_{k'}\mu_i^k\mu_i^{k'}\left(\lambda\bI_K+\sum_{\kappa}\sigma_i^\kappa\bhV_{\kappa}\right)^{-1} \bhM_{k'}\\
\bV_k&=\frac{1}{d}\sum_{i=1}^d\sigma_i^k\left(\lambda\bI_K+\sum_{\kappa}\sigma_i^\kappa\bhV_{\kappa}\right)^{-1}.
\end{split}
\end{equation}
Introducing the joint density 
\begin{equation}
\frac{1}{d}\sum_{i=1}^d\prod_{\kappa=1}^K\delta(\sigma^\kappa-\sigma_i^\kappa)\delta(\mu^\kappa-\mu_i^\kappa)\xrightarrow{d\to+\infty}\rho(\bsigma,\bmu),
\end{equation}
then we can write the saddle-point equations given in Corollary \ref{corr:1}
\begin{equation}\comprimi
\begin{split}
\bQ_k&=\mathbb E_{\bsigma,\bmu}\left[\sigma^k\left(\lambda\bI_K+\sum_{\kappa}\sigma^\kappa\bhV_{\kappa}\right)^{-1}\left(\sum_{kk'}\mu^{k}\mu^{k'}\bhM_k\bhM^\top_{k'}+\sum_{\kappa}\sigma^\kappa\bhQ_{\kappa}\right)\left(\lambda\bI_K+\sum_{\kappa}\sigma^\kappa\bhV_{\kappa}\right)^{-1}\right] \\
\bM_k&=\mathbb E_{\bsigma,\bmu}\left[\mu^k\left(\lambda\bI_K+\sum_{\kappa}\sigma^\kappa\bhV_{\kappa}\right)^{-1} \sum_{\kappa}\mu^{\kappa}\bhM_{\kappa}\right]\\
\bV_k&=\mathbb E_{\bsigma,\bmu}\left[\sigma^k\left(\lambda\bI_K+\sum_{\kappa}\sigma^\kappa\bhV_{\kappa}\right)^{-1}\right].
\end{split}
\end{equation}
where the expectations $\mathbb E_{\bsigma,\bmu}$ are taken with respect to the joint distribution $\rho$.

\subsubsection{Uniform covariances} Let us consider the simpler case $\bSigma_k\equiv \Delta\bI_d$, with $\Delta>0$. In this case, the saddle-point equations can take a more compact form that is particularly suitable for a numerical solution. Moreover, for reasons of symmetry we can write
\begin{equation}
\bQ_k\equiv\bQ,\quad \bV_k\equiv\bV,\quad\bhQ_k\equiv \frac{1}{K\Delta}\bhQ_k,\quad \bhV_k\equiv\frac{1}{K\Delta}\bhV,\quad \forall k.
\end{equation}
Let us define the following $K\times K$ matrices
\begin{itemize}
\item $\mathbf M\in\mathbb R^{K\times K}$ (resp.~$\hat{\mathbf M}\in\mathbb R^{K\times K}$) is the matrix obtained concatenenating the vectors $\bM_k$ (resp.~$\bhM_k$);
\item $\bTheta=\left(\bmu_k^\top\bmu_{k'}\right)_{kk'}$ is the Gram matrix of the means;
\item $\bF\in\mathbb R^{K\times K}$ is the matrix obtained concatenenating the vectors $\bff_k$;
\item $\bH\in\mathbb R^{K\times K}$ is the matrix obtained concatenenating the vectors $\bh_k$;
\item $\bPi=\mathrm{diag}(\rho_k)\in\mathbb R^{K\times K}$ is a diagonal matrix with elements $\Pi_{kk'}=\delta_{kk'}\rho_k$.
\end{itemize}
The saddle-point equations then can be rewritten as
\begin{equation}\comprimi
\begin{split}
\bQ&=\Delta\left(\lambda\bI_K+\bhV\right)^{-1}\left(\bhQ+ \hat{\mathbf M}\bTheta\hat{\mathbf M}^\top\right)\left(\lambda \bI_K+\bhV\right)^{-1} \\
\mathbf M&=\left(\lambda\bI_K+\bhV\right)^{-1} \hat{\mathbf M}\vec \Theta\\
\bV&=\Delta\left(\lambda\bI_K+\bhV\right)^{-1},
\end{split}\quad
\begin{split}
\bhQ&=\alpha\Delta\mathbb E_{\bXi}\left[\bF\bPi\bF^\top\right]\\
\bhV&=-\alpha\Delta\bQ^{-1/2}\mathbb E_{\bXi}\left[\bF\vec\Pi\bXi^\top\right]\\
\hat{\mathbf M}&=\alpha\mathbb E_{\bXi}\left[\bF\bPi\right]\\
\bbb&=\mathbb E_{\bXi}\left[(\bH-\mathbf M)\bPi\bOne_K\right].
\end{split}
\label{eq:spDelta}
\end{equation}
Here and in the following $\bOne_K$ is the vector of $K$ components all equal to $1$. These expressions are particularly suitable for a numerical implementation, because involve matrix multiplications and inversions of $K$-dimensional objects only.

\paragraph{Quadratic loss --- }
If we consider a quadratic loss $\ell(\by,\bx)=\frac{1}{2}\left(\by-\bx\right)^2$, then an explicit formula for the proximal can be found, namely
\begin{equation}
\bff_k=(\bI_K+\bV)^{-1}(\be_K-\bomega_k)
\end{equation}
so that the second set of saddle-point equations \eqref{eq:spDelta} can be written as
\begin{equation}\comprimi
\begin{split}
\bhQ&=\alpha(\bI_K+\bV)^{-1}\left[(\bI_K-\mathbf M-\bbb\otimes\vec 1_K)\vec \Pi(\bI_K-\mathbf M-\bbb\otimes\vec 1_K)^\top+\bQ\right](\bI_K+\bV)^{-1}\\
\hat{\mathbf M}&=\alpha (\bI_K+\bV)^{-1}(\bI_K-\mathbf M-\bbb\otimes\vec 1_K)\bPi\\
\bhV&=\alpha\Delta (\bI_K+\bV)^{-1}.
\end{split}
\end{equation}
Observe at this point that we can explicitly solve for $\bV$ using the equation for it in eqs.~\eqref{eq:spDelta}. In particular, $\bV$ satisfies the equation $\lambda \bV^2+(\alpha+\lambda-\Delta)\bV=\Delta\bI_K$. Being $\bV$ positive definite, it follows that it is diagonal, $\bV=V\bI_K$ with diagonal element
\begin{equation}
V=\frac{\Delta(1-\alpha)-\lambda+\sqrt{(\Delta-\alpha\Delta-\lambda)^2+4\Delta\lambda}}{2\lambda},\quad \hat V=\frac{\alpha\Delta}{1+V},
\end{equation}
so that
\begin{equation}\label{eq:spsquare}
\comprimi
\begin{split}
\bQ&=\frac{\Delta}{(\lambda+\Delta\hat V)^2}\left(\bhQ+ \hat{\mathbf M}\vec \Theta \hat{\mathbf M}^\top\right) \\
{\mathbf M}&=\frac{\hat{\mathbf M}\bTheta}{\lambda+\Delta\hat V},\\
\bbb&=(\bI_K-{\mathbf M})\bPi\bOne_K,
\end{split}
\qquad
\begin{split}
\bhQ&=\frac{\alpha\left[(\bI_K-{\mathbf M}-\bbb\otimes\bOne_K) \bPi(\bI_K-{\mathbf M}-\bbb\otimes\bOne_K)^\top+\bQ\right]}{(1+V)^2}\\
\hat{\mathbf M}&=-\frac{\alpha(\bI_K-{\mathbf M}-\bbb\otimes\bOne_K)\bPi}{1+V}.
\end{split}
\end{equation}
In the $\lambda\to 0$ limit, for $\alpha<1$ it is convenient to rescale $\bhQ\mapsto \lambda^2\bhQ$ and $\hat{\mathbf M}\mapsto\lambda\hat{\mathbf M}$, so that
\begin{equation}
\comprimi
\begin{split}
\bQ&=\Delta(1-\alpha)^2\left(\bhQ+ \hat{\mathbf M}\vec \Theta \hat{\mathbf M}^\top\right),\\
{\mathbf M}&=(1-\alpha)\hat{\mathbf M}\bTheta,\\
\bbb&=(\bI_K-{\mathbf M})\bPi\bOne_K,
\end{split}
\qquad
\begin{split}
\bhQ&=\frac{\alpha\left[(\bI_K-{\mathbf M}-\bbb\otimes\bOne_K) \bPi(\bI_K-{\mathbf M}-\bbb\otimes\bOne_K)^\top+\bQ\right]}{\Delta^2(1-\alpha)^2},\\
\hat{\mathbf M}&=-\frac{\alpha(\bI_K-{\mathbf M}-\bbb\otimes\bOne_K)\bPi}{\Delta(1-\alpha)}.
\end{split}
\end{equation}

\paragraph{Cross-entropy loss ---} We consider now the relevant case of the cross entropy loss
\begin{equation}
    \ell(\by, \bx) = -\sum\limits_{k=1}^{K}y_{k}\ln\frac{e^{x_{k}}}{\sum_{\kappa=1}^{K} e^{x_{\kappa}}}.
\end{equation}
If $\by\in\{\be_k\}_{k\in[K]}$, the loss can be written in the form $\ell(\by, \bx) =-\by^\top\bx+ \ln\sum_\kappa e^{x_{\kappa}}$. If we introduce the \textit{softmax function} $\mathrm{\bf soft}\colon\mathbb R^K\to\mathbb R^K$
\begin{equation}
    \partial_{\bx}\ell(\by,\bx) = -\by+\mathrm{\bf soft}(\bx), \qquad \mathrm{soft}_k(\bx)\equiv \frac{\exp\left( x_{k}\right)}{\sum_{\kappa}\exp\left( x_{\kappa}\right)}\label{app:eq:proxce}
\end{equation}
the proximal equation for the cross-entropy loss is the solution of the equations:
\begin{equation}
    \bV^{-1}(\bh_k-\bomega_{k}) - \vec{e}_{k}+\mathrm{\bf soft}(\bh_k) = \mathbf 0 \Longleftrightarrow \bff_k= \vec{e}_{k}-\mathrm{\bf soft}(\bV\bff_k+\vec{\omega}_k)\quad \forall k\in[K],
\end{equation}
having only one solution for which, however, there is no closed-form expression. The equation can be solved numerically, and in this way we obtained the results in Section \ref{sec:examples:multi}. 

The saddle-point equations can be written rescaling $\bQ\mapsto \lambda^{-2}\bQ$, $\bV\mapsto\lambda^{-1}\bV$, ${\mathbf M}\mapsto \lambda^{-1}{\mathbf M}$, $\bbb\mapsto \lambda^{-1}\bbb$, $\bhV\mapsto \lambda\bhV$. They become
\begin{equation}
\begin{split}
\bQ&=\Delta\left(\bI_K+\bhV\right)^{-1}\left(\bhQ+ \hat{\mathbf M}\vec \Theta \hat{\mathbf M}^\top\right)\left(\bI_K+\bhV\right)^{-1}, \\
{\mathbf M}&=\left(\bI_K+\bhV\right)^{-1} \hat{\mathbf M}\vec \Theta\\
\bV&=\Delta\left(\bI_K+\bhV\right)^{-1},
\end{split}\qquad
\begin{split}
\bhQ&=\alpha\Delta\mathbb E_{\bXi}\left[\bF\bPi\bF^\top\right],\\
\bhV&=\displaystyle-\alpha\Delta\bQ^{-1/2}\mathbb E_{\bXi}\left[\bF\vec\Pi\bXi^\top\right],\\
\hat{\mathbf M}&=\alpha\mathbb E_{\bXi}\left[\bF\bPi\right],\\
\bbb&=\mathbb E_{\bXi}\left[(\bH-{\mathbf M})\bPi\right],
\end{split}
\end{equation}
so that the dependence on $\lambda$ disappears everywhere except in the equation for the proximal $\bff_k$
\begin{equation}
\bff_k=
\arg\min_{\bx}\left[\frac{1}{2}\bx^\top \bV\bx+\lambda\ell\left(\be_k,\frac{\bV\bx+\bomega_k}{\lambda}\right)\right],
\end{equation}
which, in the $\lambda\to0$ limit, becomes
\begin{equation}
\bff_k=\arg\min_{\bx}\left[\frac{1}{2}\bx^\top \bV\bx+\min_\mu\{(\be_\mu-\be_k)^\top(\bV\bx+\bomega_k)\}\right].\label{app:eq:lam0}
\end{equation}
Note that in this limit, minimising the cross-entropy loss yields precisely the max-margin estimator \cite{rosset2003margin}.

\subsection{\boldmath The $K=2$ case with scalar labels}\label{sec:app:simpli:k2}
The formulas for the $K=2$ case can be derived directly from the general analysis given above imposing $L=1$. In particular, let us assume that the two clusters are labeled with $e_1=+1$ and $e_2=-1$. Using as classifier
\begin{equation}
    \varphi(x)=\mathrm{sign}(x)
\end{equation}
the expression of the average errors is
\begin{equation}
\begin{split}
\epsilon_g&=\sum_{k\in[2]}\rho_k\mathbb E_\xi[\theta\left((-1)^{k}\omega_k^\star\right)]=\sum_{k\in[2]}\frac{\rho_k}{2}\mathrm{erfc}\left((-1)^{k-1}\frac{m_k^\star+b^\star}{\sqrt{2q_k^\star}}\right),\\
\epsilon_t&=\sum_{k\in[2]}\rho_k\mathbb E_\xi[\theta\left((-1)^{k}h_k^\star\right)],\\
\epsilon_\ell&=\sum_{k\in[2]}\rho_k\mathbb E_\xi[\ell((-1)^k,h_k^\star)].
\end{split}
\end{equation}
We will further explore this case, considering some special cases in the following.
\subsubsection{Example: $\ell_1$ regularization}
In this Section we derive the saddle-point equations for the the case in which the two cluster have opposite means $\bmu_1=-\bmu_2\equiv\bmu$, and the same diagonal covariance matrix, $\bSigma_1=\bSigma_2\equiv \bSigma$, with $\Sigma_{ij}=\sigma_i\delta_{ij}$ and $\sigma_i>0$. In this case, for symmetry reasons, the overlaps simplify and we have:
\begin{align}
V_1 = V_2 \equiv V,&& q_1 = q_2 \equiv q, &&\quad m_{+} = -m_{-} \equiv  m,\\
\hat{V}_{+} = \hat{V}_{-} \equiv  \frac{1}{2}\hat{V},&& \hat{q}_{+} = \hat{q}_{-} \equiv  \frac{1}{2}\hat{q},&& \hat{m}_{+} = - \hat{m}_{-} \equiv  \frac{1}{2}\hat{m}.
\end{align}
We define
\begin{equation}
\frac{1}{d}\sum_{i=1}^d\delta(\sigma-\sigma_i)\delta(\mu-\sqrt d\mu_i)\xrightarrow{d\to+\infty} p(\sigma,\mu)
\end{equation}
joint distribution of the covariance diagonal elements and of the mean elements. We will denote $\mathbb E_{\mu,\sigma}[\bullet]$ the average with respect to this measure. We will focus in particular on the form of the saddle-point equations obtained from the prior contribution assuming $\ell_1$ regularization, i.e., $r(\bw)=\sum_i|w_i|$, and let us introduce the corresponding \textit{soft-thresholding operator}:
\begin{equation}
\mathrm{Prox}_{\lambda|\cdot|}(x)=\sign(x)\max\{|x|-\lambda,0\}.
\end{equation}
Observe that $\mathrm{Prox}_{\alpha\lambda|\cdot|}(\alpha x)=\alpha\mathrm{Prox}_{\lambda|\cdot|}(x)$ for $\alpha>0$. Its derivative given by $\mathrm{Prox}_{\lambda|\cdot|}'(x)=\theta(|x|>\lambda)$. The saddle point equations from the prior part simply read:
\begin{align}
	V &= \frac{1}{\hat{V}}~\mathbb{E}_{\mu,\sigma,\xi}\left[\mathrm{Prox}'_{\frac{\lambda}{\sigma\hat V}|\cdot|}\left(\frac{\hat m\mu+\sqrt{\hat q\sigma}\xi}{\hat V\sigma}\right)\right],\\
	q &=\mathbb{E}_{\mu,\sigma,\xi}\left[ \sigma~\left(\mathrm{Prox}_{\frac{\lambda}{\sigma\hat V}|\cdot|}\left(\frac{\hat m\mu+\sqrt{\hat q\sigma}\xi}{\hat V\sigma}\right)\right)^2\right],\\
	m &= \mathbb{E}_{\mu,\sigma,\xi}\left[\mu \mathrm{Prox}_{\frac{\lambda}{\sigma\hat V}|\cdot|}\left(\frac{\hat m\mu+\sqrt{\hat q\sigma}\xi}{\hat V\sigma}\right)\right].
\end{align}
The averages over $\xi$ can be performed explicitely using the simple expression of the proximal in this case. If we define the auxiliary functions
\begin{equation}
\begin{split}
\phi_\pm^0(v,u,\lambda)&\equiv\frac{1}{2}\mathrm{erfc}\left(\frac{\lambda\pm v}{\sqrt{2u}}\right)\\
\phi_\pm^1(u,v,\lambda)&= \sqrt{\frac{u}{2\pi}} e^{-\frac{(v\pm\lambda)^2}{2u}}-\frac{v\pm\lambda}{2}\mathrm{erfc}\left(\frac{\lambda\pm v}{\sqrt{2u}}\right),\\
\phi^2_\pm(v,u,\lambda)&=-\sqrt{\frac{u}{2\pi}}e^{-\frac{(\lambda\pm v)^2}{2u}}(\lambda\pm v) + \frac{u+\left(\lambda\pm v\right)^2}{2}\text{erfc}\left(\frac{\lambda\pm v}{\sqrt{2u}}\right).
\end{split}
\end{equation}
then
\begin{equation}
\begin{split}
V &= \frac{1}{\hat{V}}\mathbb{E}_{\mu,\sigma}\left[\phi_+^0(\mu\hat m,\sigma\hat q,\lambda)+\phi_-^0(\mu\hat m,\sigma\hat q,\lambda)\right]\\
q&=\mathbb  E_{\mu,\sigma}\left[\frac{\phi^2_+(\mu\hat m,\sigma\hat q,\lambda)+\phi^2_-(\mu\hat m,\sigma\hat q,\lambda)}{\sigma\hat V^2}\right],\\
m&=\mathbb E_{\mu,\sigma}\left[\frac{\mu\phi_-^1(\mu\hat m,\sigma q,\lambda)-\mu\phi_+^1(\mu\hat m,\sigma q,\lambda) }{\sigma\hat{V}}\right].
\end{split}
\end{equation}

\paragraph{Gaussian means, homogenous covariances ---} If $p(\mu,\sigma)=\mathcal N(\mu|0,1)\delta(\sigma-\Delta)$, i.e., the means have i.i.d.~Gaussian entries and $\bSigma=\Delta\bI_d$, then
\begin{equation}
\begin{split}
V&= \frac{1}{\hat{V}}\mathbb{E}_{z}\left[\text{erfc}\left(\frac{\lambda+\hat{m}z}{\sqrt{2\Delta\hat{q}}}\right)\right],\\
q & =\frac{1}{\Delta\hat{V}^2}\left\{-\frac{e^{-\frac{1}{2}\frac{\lambda^2}{\hat{m}^2+\Delta \hat{q}}}}{\sqrt{2\pi(\hat{m}^2+\Delta \hat{q})}} \frac{2(\Delta\hat{q})^2\lambda}{\hat{m}^2+\Delta \hat{q}}+\mathbb{E}_{z}\left[\left(\lambda+\hat{m}z\right)^2\text{erfc}\left(\frac{\lambda+\hat{m}z}{\sqrt{2\Delta\hat{q}}}\right)\right]\right\},\\
m & = \frac{1}{\Delta\hat{V}} 	\left\{\frac{e^{-\frac{1}{2}\frac{\lambda^2}{\hat{m}^2+\Delta \hat{q}}}}{\sqrt{2\pi(\hat{m}^2+\Delta \hat{q})}} \frac{2\Delta\hat{q}\hat{m}\lambda}{\hat{m}^2+\Delta \hat{q}} +\mathbb{E}_{z\sim\mathcal{N}(0,1)}\left[\left(\lambda+\hat{m}z\right)z~\text{erfc}\left(\frac{\lambda+\hat{m}z}{\sqrt{2\Delta\hat{q}}}\right)\right]\right\},
\end{split}
\end{equation}
with $z\sim \mathcal N(0,1)$.

\paragraph{Covariance correlated with sparse means ---} In Section \ref{sec:examples:sparse} we considered the case of sparse means correlated with the covariance matrices. In particular, we considered
\begin{equation}\label{sparsecorr}
p(\sigma,\mu)=p\mathcal N(\mu|0,1)\delta(\sigma-\Delta_1)+(1-p)\delta(\mu)\delta(\sigma-\Delta_0).
\end{equation}
The saddle-point equations are therefore
\begin{align}
V=&\frac{1}{\hat V}\left[p\mathbb{E}_{\mu}\left[\text{erfc}\left(\frac{\lambda+\hat{m}\mu}{\sqrt{2\Delta_1\hat{q}}}\right)\right]+(1-p)\text{erfc}\left(\frac{\lambda}{\sqrt{2\Delta_0\hat{q}}}\right)\right]\\
\begin{split}q=&\frac{p}{\Delta_1\hat{V}^2}\left\{-\frac{e^{-\frac{1}{2}\frac{\lambda^2}{\hat{m}^2+\Delta_1 \hat{q}}}}{\sqrt{2\pi(\hat{m}^2+\Delta_1 \hat{q})}} \frac{2(\Delta_1\hat{q})^2\lambda}{\hat{m}^2+\Delta_1 \hat{q}}+\mathbb{E}_{z}\left[\left(\lambda+\hat{m}z\right)^2\text{erfc}\left(\frac{\lambda+\hat{m}z}{\sqrt{2\Delta_1\hat{q}}}\right)\right]\right\}\\
&-\lambda(1-p)\sqrt{\frac{\Delta_0\hat{q}}{2\pi}}e^{-\frac{\lambda^2}{2\Delta_0 q}}+ \frac{1-p}{2}(\Delta_0\hat{q}+\lambda^2)\text{erfc}\left(\frac{\lambda}{\sqrt{2\Delta_0\hat{q}}}\right)
\end{split}\\
m=&\frac{p}{\Delta_1\hat{V}} 	\left\{\frac{e^{-\frac{1}{2}\frac{\lambda^2}{\hat{m}^2+\Delta_1 \hat{q}}}}{\sqrt{2\pi(\hat{m}^2+\Delta_1 \hat{q})}} \frac{2\Delta_1\hat{q}\hat{m}\lambda}{\hat{m}^2+\Delta_1 \hat{q}} +\mathbb{E}_{z}\left[\left(\lambda+\hat{m}z\right)z~\text{erfc}\left(\frac{\lambda+\hat{m}z}{\sqrt{2\Delta_1\hat{q}}}\right)\right]\right\}.
\end{align}
In Section \ref{sec:examples:sparse} we compare the performance obtained adopting an $\ell_1$ regularization with the corresponding one obtained using $\ell_2$, $r(\bw)=\sum_iw_i^2$. For the sake of completeness, we give here the expression of the saddle-point equations in that case as well. In this case, the prior term $\Psi_w$ can be written explicitly after a Gaussian integration as
\begin{equation}
\Psi_w(\hat m,\hat Q,\hat V)=-\frac{1}{2d}\tr\ln\left(\lambda\bI_d+\hat V\vec \Sigma\right)+\frac{1}{2}\tr\left[\left(\lambda\bI_d+\hat V\bSigma\right)^{-1}\left(\hat m_k^2\bmu\bmu^\top+\frac{\hat q}{d}\bSigma\right)\right].
\end{equation}
In the setting given by eq.~\eqref{sparsecorr} the saddle point equations are then
\begin{subequations}
\begin{align}
q &
=p\frac{\hat m^2\Delta_1+\hat q\Delta_1^2}{(\lambda+\hat V\Delta_1)^2}+\frac{(1-p)\hat q\Delta_0^2}{(\lambda+\hat V\Delta_0)^2}\\
V&
=p\frac{\Delta_1}{
\lambda+\hat V\Delta_1}+\frac{(1-p)\Delta_0}{
\lambda+\hat V\Delta_0} \\
m&
=\frac{\hat mp}{\lambda+\hat V\Delta_1}.
\end{align}
\end{subequations}
}

\section{Bayes optimal error}
\label{sec:app:bayes}
In this Appendix, we derive a formula for the Bayes optimal classification error in the case of $K$ clusters with the same covariance $\bSigma_k=\Delta\bI_d$ in the large $d$ limit, assuming that a dataset $\{(\bx^\nu,\by^\nu)\}_{\nu\in[n]}$ of correctly labeled points is available. As usual, we will assume $\nicefrac{n}{d}=\alpha$ finite. The distribution of a pair $(\by,\bx)$ is given by
\begin{equation}
p(\by,\bx|\bMM)=\sum_k y_k\frac{\rho_k\exp\left(-\frac{1}{2\Delta}\left\|\bx-\bmu_k\right\|^2\right)}{(2\pi\Delta)^{\frac{d}{2}}}.
\end{equation}
where $\bMM\in\mathbb R^{d\times K}$ is the matrix of concatenated means $\bmu_k$ \textit{estimated} from the dataset, so that
\begin{equation}
\begin{split}
p(\bMM|\{\by^\nu,\bx^\nu\}_{\nu})&\propto p(\{\bx^\nu\}_{\nu}|\bMM,\{\by^\nu\}_{\nu})P_\bmu(\bMM)\\
&\propto P_\bmu(\bMM)\prod_{\nu=1}^n\sum_{k}y^\nu_k\exp\left(-\frac{1}{2\Delta}\left\|\bx^{\nu}-\bmu_k\right\|^2\right).
\end{split}
\end{equation}

We will assume in the following the distribution
\begin{equation}
P_\bmu(\bMM)=\frac{\exp\left(-\frac{d}{2}\mathrm{tr}[\bMM \bTheta^{-1} \bMM^\top]\right)}{(2\pi)^{\frac{Kd}{2}}d^{-K/2}|\bTheta|^{1/2}}
\end{equation}
where $\bTheta\in\mathbb R^{K\times K}$ is a given positive definite covariance matrix. In this way
\begin{equation}
\mathbb E\left[\bMM^\top \bMM\right]=\bTheta.
\end{equation}
The conditional distribution for the label $\by^0$ of a new point $\bx^0$,
\begin{multline}
p(\by^0|\bx^0,\{\by^\nu,\bx^\nu\}_\nu)\propto\mathbb E_{\bMM|\{\by^\nu,\bx^\nu\}_\nu}[p(\by,\bx|\bMM)]\\
=\int \dd\bMM P_\bmu(\bMM)\sum_k y_k^0\rho_k\exp\left(-\frac{\left\|\bx^0-\bmu_k\right\|^2}{2\Delta}\right)\prod_{\nu=1}^n\sum_{k}y^\nu_k\exp\left(-\frac{\left\|\bx^{\nu}-\bmu_k\right\|^2}{2\Delta}\right).
\end{multline}
If $\bn=(n_k)_k$ is the vector of the number of examples $n_k$ in the class $k$, then
\begin{multline}
p(\by^0|\bx^0,\{\by^\nu,\bx^\nu\}_\nu) \propto\int\dd \bMM P_\bmu(\bMM)\prod_{k=1}^K\left[\rho_k^{y^0_k}\exp\left(-\sum_{\nu=0}^n\frac{y_k^\nu\left\|\bx^{\nu}-\bmu_k\right\|^2}{2\Delta}\right)\right]\\
=\exp\left[\sum_ky_k^0\left(\ln\rho_k-\frac{\|\bx\|^2}{2\Delta}\right)-\frac{1}{2}\ln\det\left(1+\frac{1}{d\Delta}\mathrm{diag}(\bn+\by^0)\bTheta\right)\right]\\
\times\exp\left[\frac{1}{2\Delta}\mathrm{tr}\left[\left(\sum_{\nu=0}^n\by^\nu\otimes \bx^{\nu}\right)^\top\left(d\Delta\bTheta^{-1}+\mathrm{diag}(\bn+\by)\right)^{-1} \left(\sum_{\nu=0}^n\by^\nu\otimes \bx^{\nu}\right)\right]\right].
\end{multline}
In the following we will denote by $\star$ the true label of $\bx$. Let $\bPi=\mathrm{diag}(\rho_k)$. Then we can write the previous expression as
\begin{multline}
p(\by^0|\bx^0,\{\by^\nu,\bx^\nu\}_\nu) \propto
\exp\left[\sum_ky_k\left(\ln\rho_k-\frac{\|\bx^0\|^2}{2\Delta}\right)-\frac{1}{2}\ln\det\left(1+\frac{1}{\Delta}\alpha\bPi\bTheta\right)\right]\\
\times\exp\left[\frac{1}{2\Delta}\mathrm{tr}\left[\left(\frac{1}{d}\sum_{\nu=0}^n\by^\nu\otimes \bx^{\nu}\right)^\top\left(\Delta\bTheta^{-1}+\alpha\bPi\right)^{-1} \left(\sum_{\nu=0}^n\by^\nu\otimes \bx^{\nu}\right)\right]\right]
\end{multline}
Observe now that
\begin{equation}
\frac{1}{d\Delta}\bx^0\sum_{\nu=1}^ny_k^\nu\bx^\nu\xrightarrow{n,d\to+\infty}\alpha\rho_k\frac{
\Theta_{\star,k}+\eta_k Z_{k}}{\Delta},\qquad \eta_k\equiv\sqrt{\Delta\left(1+\frac{\Delta}{\alpha\rho_k}\right)},\quad Z_k\sim\mathcal N(0,1),
\end{equation}
so that, defining the vector $\ba^\star=(a_k)_{k\in[K]}$ with elements
\begin{equation}
a_k^\star\equiv \alpha\rho_k\frac{
\Theta_{\star,k}+\eta_k Z_{k}}{\Delta},
\end{equation}
and neglecting the $\by^0$-independent contributions, the expression above can be rewritten as
\begin{equation}
p(\by^0|\bx^0,\{\by^\nu,\bx^\nu\}_\nu) \propto
\exp\left[\sum_ky^0_k\ln\rho_k+\left(\ba^\star+\frac{1}{2}\by^0\right)^\top\left(\Delta\bTheta^{-1}+\alpha\bPi\right)^{-1} \by^0\right]
\end{equation}
where we have also used the fact that $\|\bx^0\|^2=d\Delta+O(1)$. 
{This means that the Bayes optimal generalization error is
\begin{equation}
\varepsilon_g^{\text{BO}}=\sum_k\rho_k\mathbb P\left[\arg\max_{\kappa}\left(\ln\rho_\kappa+\left(\ba^k+\frac{1}{2}\be_\kappa\right)^\top\left(\Delta\bTheta^{-1}+\alpha\bPi\right)^{-1} \be_\kappa\right)\neq k\right].
\end{equation}}
If $\bTheta=\bI_K$ and the clusters have same weights, $\rho_k\equiv\nicefrac{1}{K}\Leftrightarrow \bPi=\nicefrac{1}{K}\bI_K$, then $\eta_k\equiv \eta$ and
\begin{equation}
\varepsilon_g^{\text{BO}}=\mathbb P\left[\frac{1}{\eta}<\max_{\kappa\in[K-1]}Z_\kappa+Z\right],
\end{equation}
that is the formula given in \cite{Thrampoulidis2020}.

\section{Experiments with real data}
\label{sec:app:realdata}
In this Appendix we discuss the experiments of Section \ref{sec:examples:mnist} with real data sets.

\paragraph{Numerical details} Consider a real data set $\{(\bx^{\nu}, y^{\nu})\}_{\nu=1}^{n_{\tot}}$ with $n_{\tot}$ samples which we assume are independent. As a pre-processing step we center, normalise and flatten the inputs $\bx^{\nu}$ into $d$-dimensional vectors. For both the MNIST \cite{lecunMNIST} and Fashion-MNIST \cite{xiao2017fashion} data sets used in the experiments we have normalised the inputs by 255, such that components $x^{\nu}_{i}\in[0,1]$. In what follows we focus on binary classification tasks and encode the labels as $y^{\nu}\in\{-1,1\}$. For example, for the MNIST and Fashion-MNIST data sets we have $d=784$ and $n_{\tot} = 7\times 10^{4}$, and we split the inputs into two classes depending on the task of interest, e.g. odd vs. even digits and clothes vs. accessories items, respectively. Define the empirical distribution over the data set:
\begin{align}
    \hat{P}(\bx, y) = \frac{1}{n_{\tot}}\sum\limits_{\nu=1}^{n_{\tot}}\delta(\bx - \bx^{\nu})\delta(y - y^{\nu})
\end{align}
The question we want to answer is: how well can we approximate the learning curves $(\epsilon_g, \epsilon_t)$ on a given ERM classification task by approximating $\hat{P}$ with a Gaussian mixture distribution? To answer this question, we consider a Gaussian mixture distribution $P_2$ as defined in Eq.~\eqref{joint} with the same means and covariances as $\hat{P}$:
\begin{align}
    \hat{\bmu}_{k} = \frac{1}{n_{\tot}}\sum\limits_{\nu=1}^{n_{\tot}}\bx^{\nu}~\mathbb{I}\left(\bx^{\nu}\in\mathcal{C}_{k}\right), &&\hat{\bSigma}_{k} = \frac{1}{n_{\tot}}\sum\limits_{\nu=1}^{n_{\tot}}(\bx^{\nu}-\bmu_{k})(\bx^{\nu}-\bmu_{k})^{\top}~\mathbb{I}\left(\bx^{\nu}\in\mathcal{C}_{k}\right)
\end{align}
\noindent for $k\in\{+,-\}$ labelling the two clusters. Similarly, the class probabilities $\rho_{k}$ are also estimated from the full data set:
\begin{align}
    \hat{\rho}_{k} = \frac{1}{n_{\tot}}\sum\limits_{\nu=1}^{n_{\tot}}\mathbb{I}\left(\bx^{\nu}\in\mathcal{C}_{k}\right).
\end{align}
The parameters $(\hat{\bmu}_{k}, \hat{\bSigma}_{k}, \hat{\rho}_{k})$ completely characterise the approximating Gaussian mixture distribution $P_2$, and together with Theorem \ref{the:1} can be used to compute the theoretical learning curves $(\epsilon_g, \epsilon_t)$ as in Fig.~\ref{fig:mnist} of the main. Note that this discussion can be easily generalised to the case in which a non-linear feature map $\bvarphi:\mathbb{R}^{d}\to \mathbb{R}^{p}$ is applied to the data prior to fitting. The only difference is that the empirical distribution $\hat{P}$ is defined over the features $\{(\bv^{\nu}, y^{\nu})\}_{\nu=1}^{n_{\tot}}$ where $\bv^{\nu}=\bvarphi(\bx^{\nu})$, and the Gaussian mixture approximation $P_2$ is defined with respect to the empirical features distribution. Figure \ref{fig:mnistrf} of the main manuscript shows an example where a random feature map $\bv = \text{erf}\left(\bF\bx\right)$ with $\bF\in\mathbb{R}^{p\times d}$ a random Gaussian projection applied to MNIST and fashion MNIST before the fitting with different ratios $\gamma = p/d$. 

The theoretical learning curves are then compared with two sets of finite instance simulations. First, we simulate the learning problem on synthetic data sampled from the approximating Gaussian mixture distribution $P_2$, and the learning curves are computed by averaging over $10$ instances of the problem. Second, we simulate the learning problem on the real data set. The real data set is split into training and test sets, and for a given sample complexity $\alpha = n/d$ we sub-sample $n=\alpha d$ points from the training set. The averaged learning curves are computed over different instances of the sub-sampling, with replacement. 

\paragraph{Discussion} As expected, we find good agreement between theory and simulations with synthetic data drawn from the approximating Gaussian mixture distribution $P_2$, even for relatively small input dimensions (e.g. $d=784$ for MNIST). Surprisingly, we have found that in many cases the Gaussian mixture is a good approximation to the real data curves, see Figs.~\ref{fig:mnist} and \ref{fig:mnistrf} for examples of logistic regression on input space and with random features. Figure \ref{fig:app:trained} shows an example where the feature map $\bvarphi$ is given by removing the last layer of the following fully-connected 2-layer neural network pre-trained on the full MNIST odd vs. even data set: 
\begin{python}
Sequential(
  (0): Linear(in_features=784, out_features=784, bias=False)
  (1): ReLU()
  (2): Linear(in_features=784, out_features=1, bias=False)
  (3): Tanh()
)
\end{python}
\noindent with the training performed by minimising the square loss with the Adam optimiser and random initialisation. However, we have also found cases in which the approximation is not as sharp, see blue curves in Fig.~\ref{fig:app:2vs10}. Understanding the factors determining the quality of the approximation in real data sets is an interesting question we expect to address in future work.

\begin{figure}[ht]
    \centering
    \includegraphics[scale=0.45]{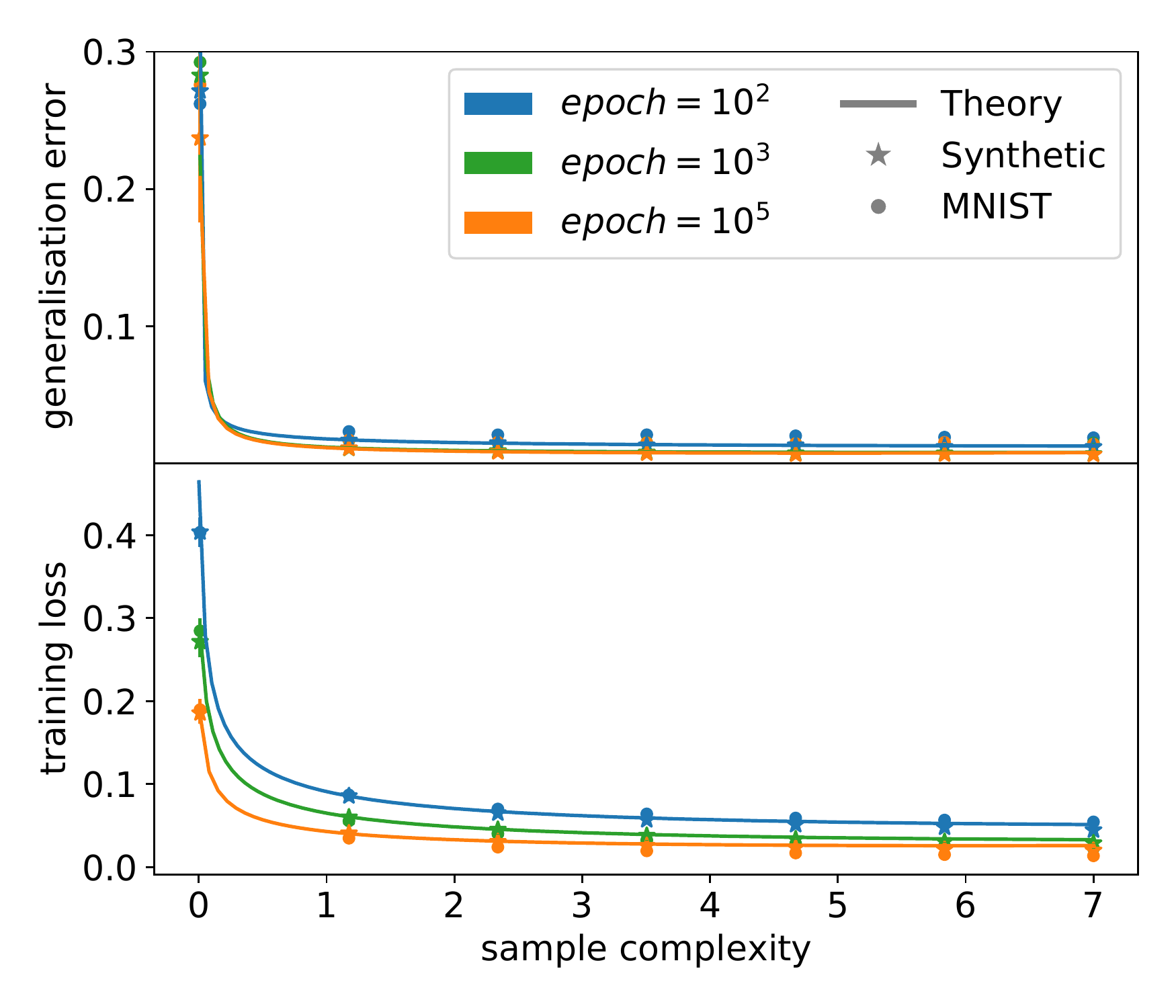}
    \caption{Generalisation error and training loss for logistic regression on MNIST with a feature map $\bvarphi$ obtained by training 2-layer fully connected neural network, with $\ell_2$ penalty and fixed $\lambda=0.05$. The different curves show the performance at different stages of training.}
    \label{fig:app:trained}
\end{figure}


\begin{figure}[t]
    \centering
    \includegraphics[width=0.45\textwidth]{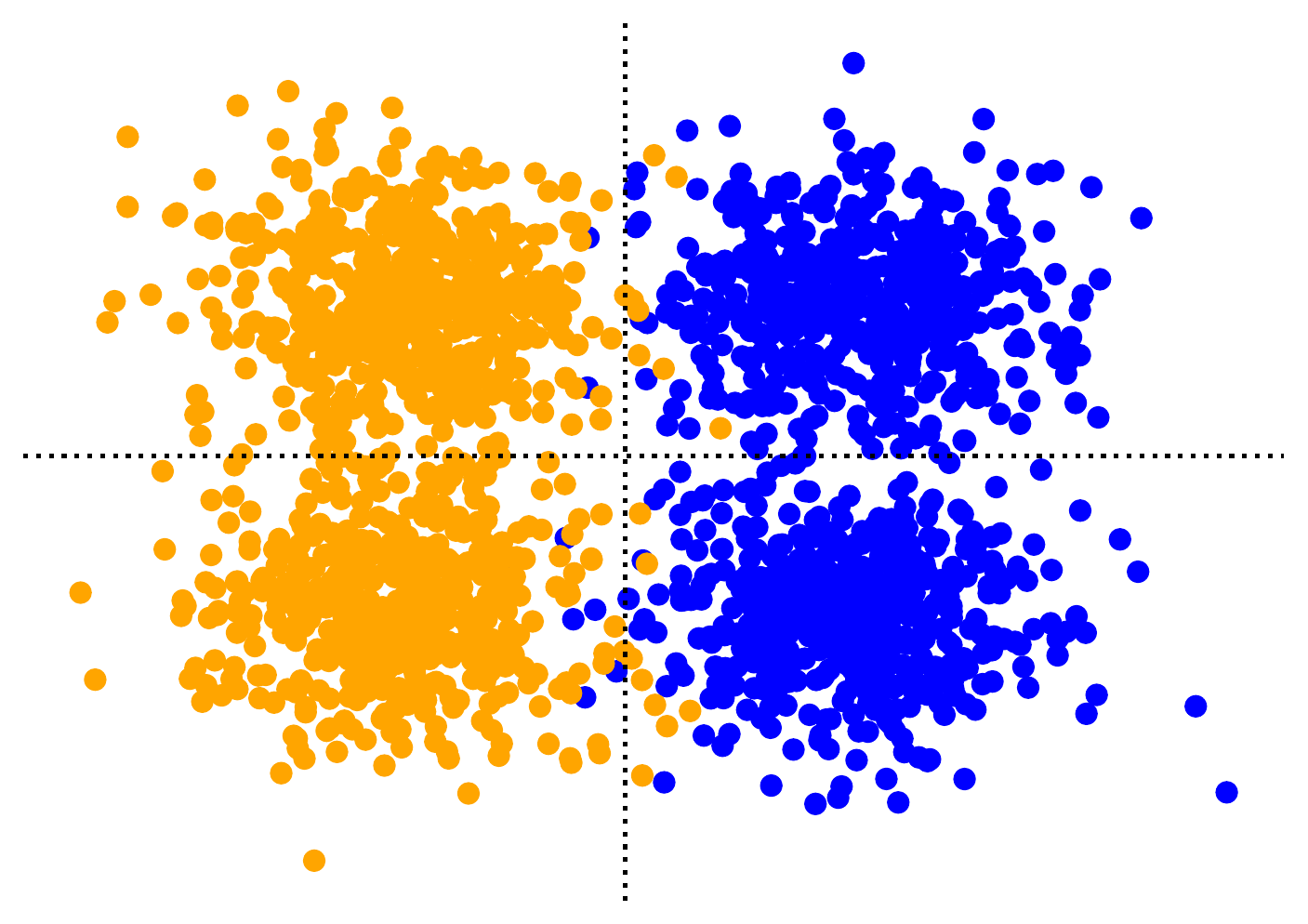}
    \includegraphics[width=0.45\textwidth]{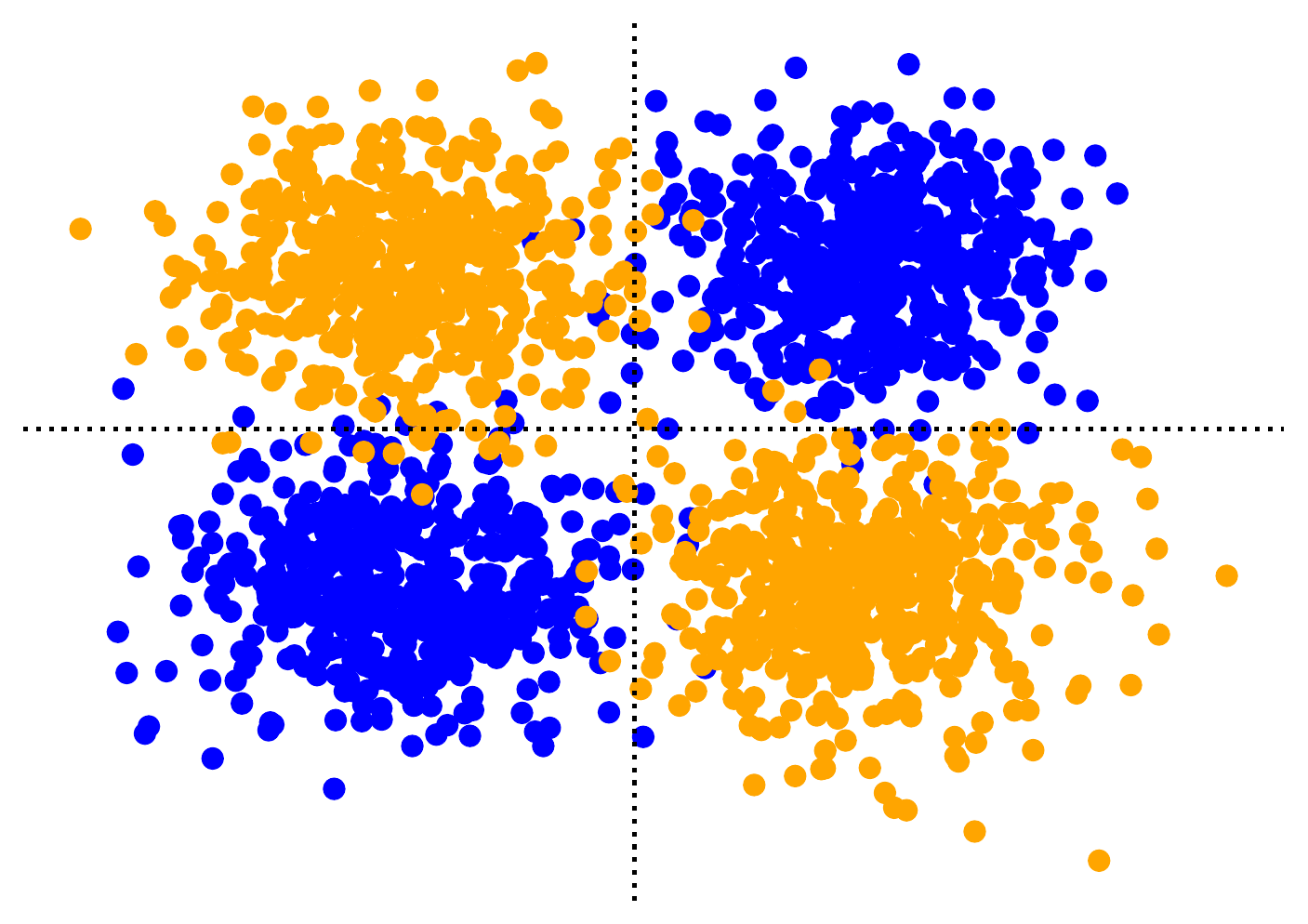}\\
    \includegraphics[width=0.45\textwidth]{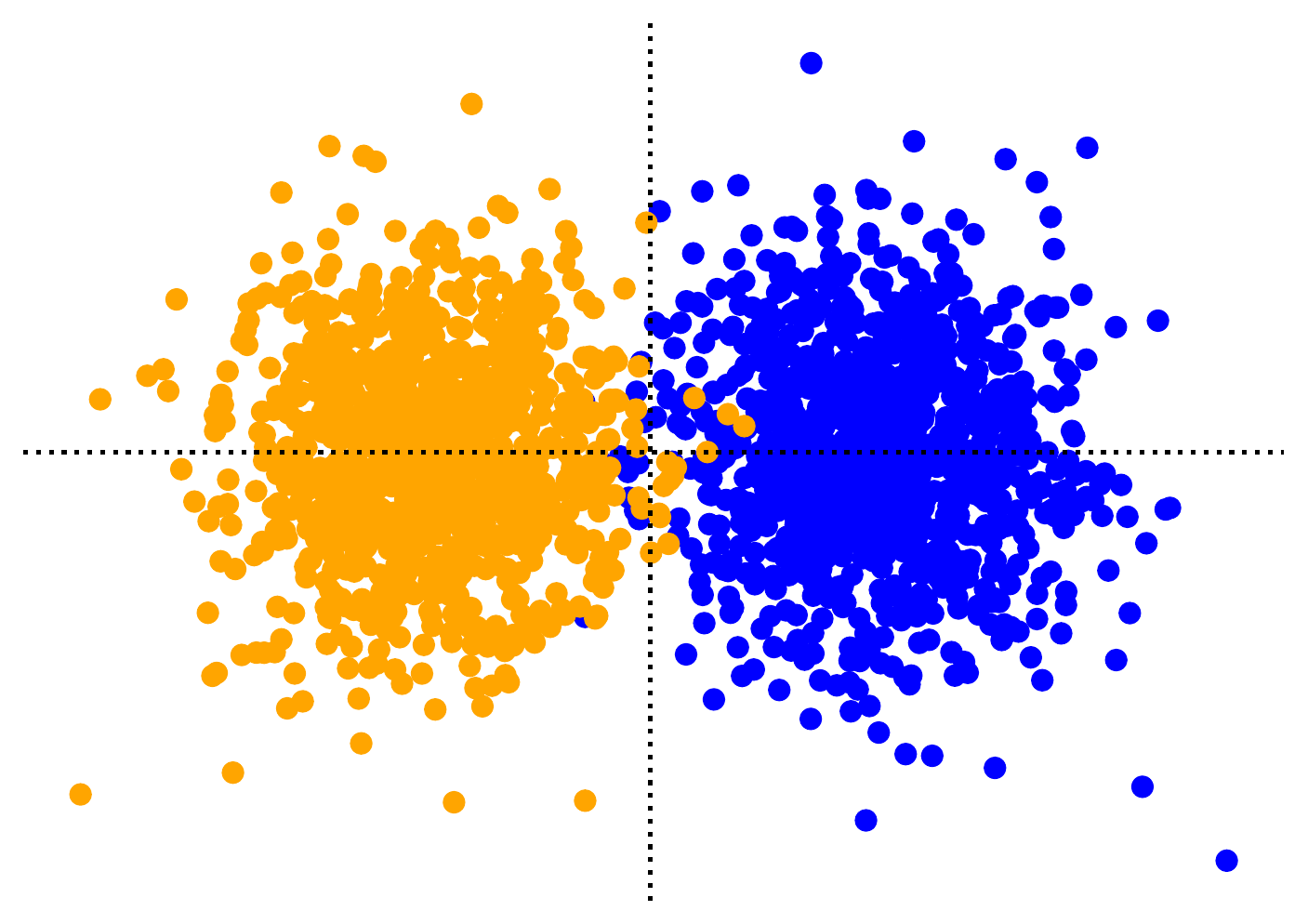}
    \includegraphics[width=0.45\textwidth]{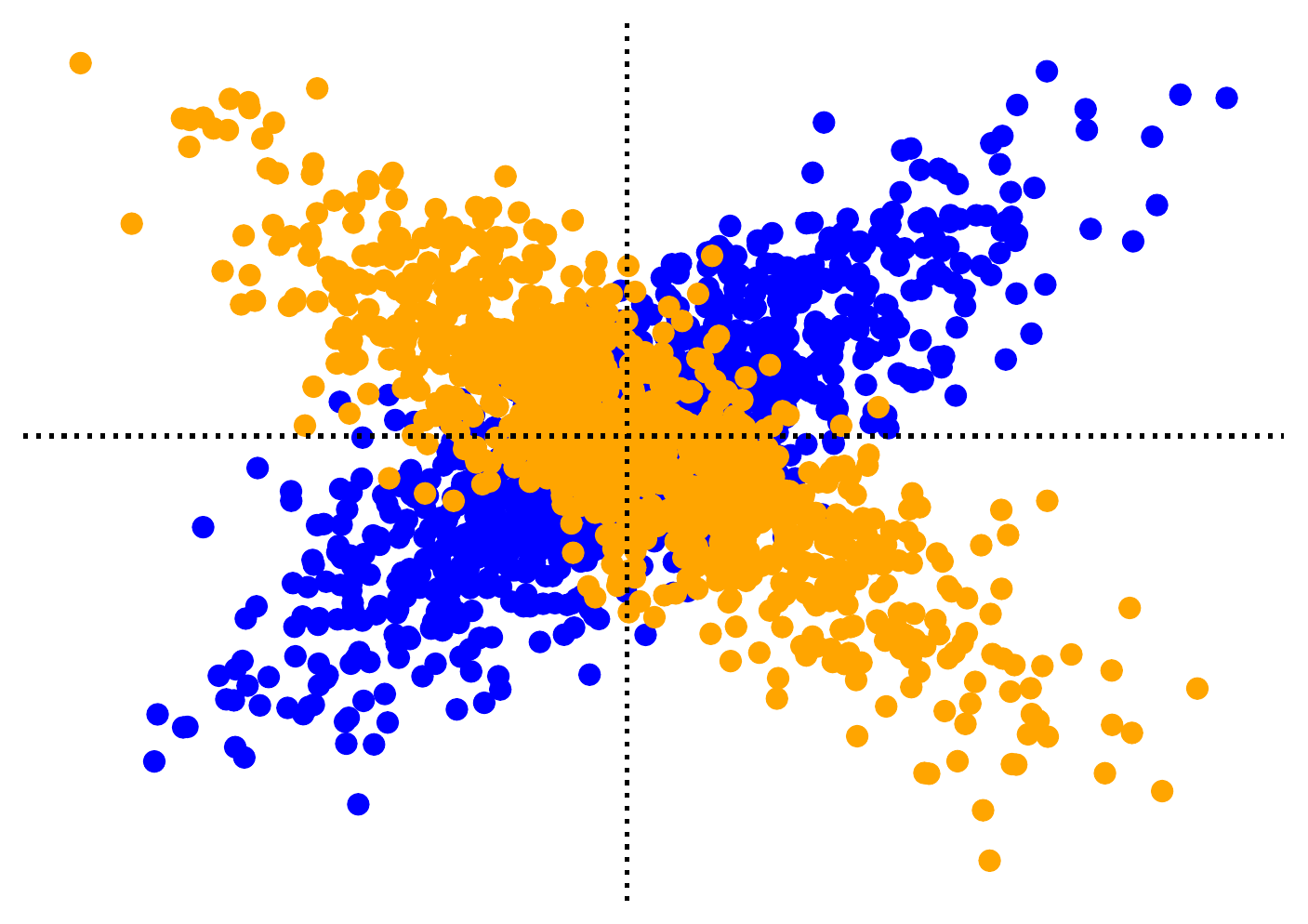}
    \caption{Two dimensional projection of the setting described in eq.~\eqref{eq:app:4vs2}. (\textbf{Left}) Realisable case,  (\textbf{Right}) Non-realisable case (XOR function).}
    \label{fig:app:xor}
\end{figure}

\begin{figure}[H]
    \centering
    \includegraphics[width=\textwidth]{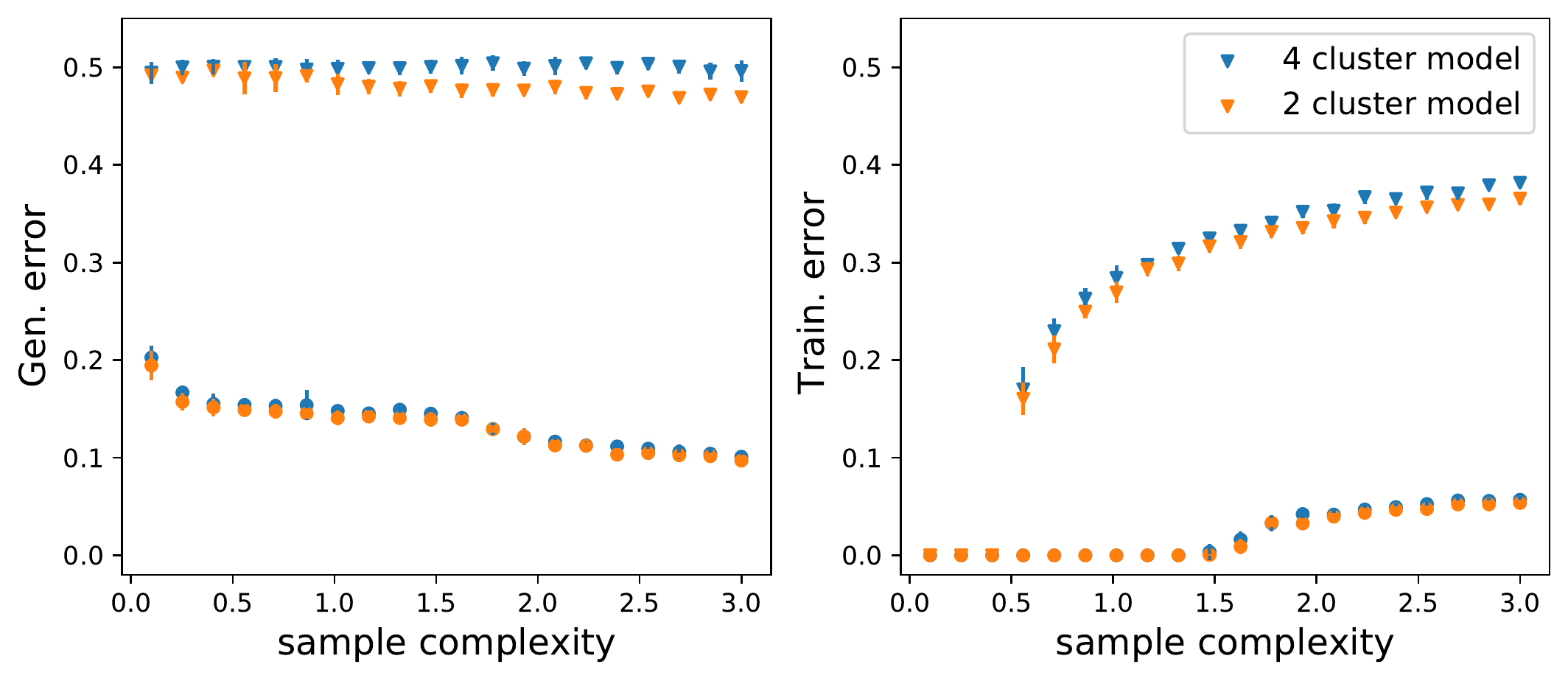}
    \caption{(\textbf{Left}) Generalisation and (\textbf{right}) training errors as a function of the sample complexity for logistic regression with $\ell_2$ penalty and $\lambda=10^{-4}$ for the four models pictured in Fig.~\ref{fig:app:xor}. Points denote the separable model (bottom curve), and triangles denote the non-realisable xor model (top curves). We have chosen a balanced scenario with $\Delta = 0.5$.}
    \label{fig:app:2vs4}
\end{figure}
\begin{figure}[H]
    \centering
    \includegraphics[scale=0.45]{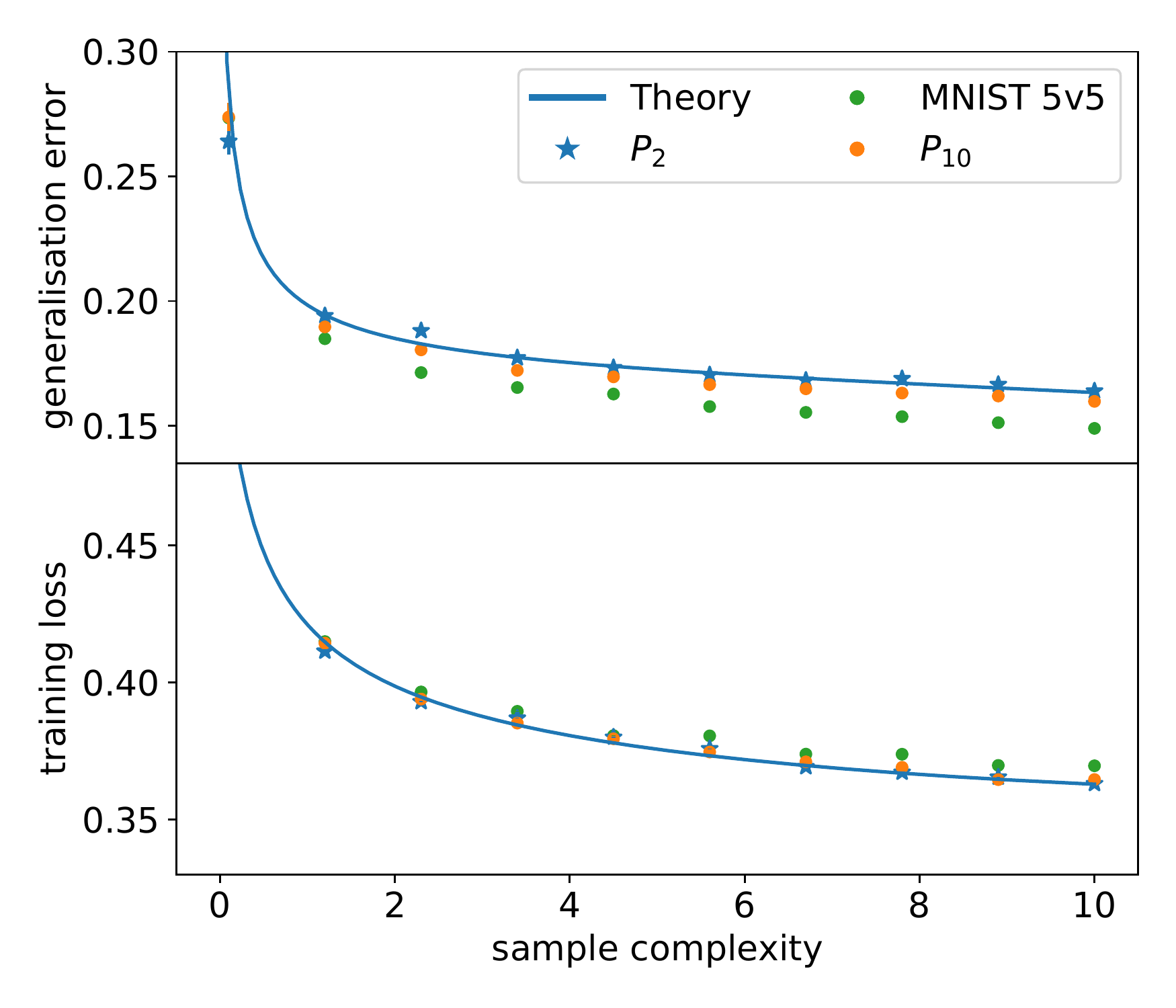}
    \caption{Generalisation error and training loss for logistic regression on the task of classifying $\{0, 1, 2, 3, 4\}$ vs $\{5,6,7,8,9\}$ digits of MNIST, as a function of the sample complexity for fixed $\ell_2$ penalty $\lambda = 0.1$. The blue curves show the 2-Gaussian cluster approximation $P_2$ (solid for theory, points for finite size simulations), while the orange points show the $10$-Gaussian cluster approximation $P_{10}$, which lies systematically below. The green points denote simulations on the true data set.}
    \label{fig:app:2vs10}
\end{figure}

\paragraph{Multiclass vs. binary approximation --} In the cases previously discussed, we have considered a $K=2$ cluster approximation $P_{2}$ to the empirical data distribution $\hat{P}$. However, the data sets considered here (MNIST and Fashion-MNIST) are originally composed of $10$ classes, and therefore we should ask the question of whether a $K=10$ cluster approximation $P_{10}$ where we fit the means and covariances of each original class is any different from the approximation studied above. In principle, these two approximations can have very different statistical properties. For instance, from Theorem \ref{th:2} it follows that the generalisation and training errors of Gaussian mixtures only depend on the statistics of the local field $\lambda = \bW\bx$ conditioned on the labels, which in the binary setting considered here is $y\in\{+,-\}$. Conditioned on $y=\pm$, this local field is simply a Gaussian random variable under $P_2$, while it is a multi-modal random variable under $P_{10}$. Therefore, there is \emph{a priori} no reason for these two approximations to give the same learning curves.

As an example, consider a $K=4$ Gaussian mixture distribution with a common variance $\Sigma_{k} = \Delta \mat{I}_{d}$ and with means:
\begin{align}
    \bmu_{1} = \be_{1}+\be_{2}, && \bmu_{2} = \be_{1}-\be_{2}, && \bmu_{3} = -\be_{1}+\be_{2}, && \bmu_{4} = -\be_{1}-\be_{2}
    \label{eq:app:4vs2}
\end{align}
\noindent where $\be_{i}\in\mathbb{R}^{d}$ is the canonical basis vector of $\mathbb{R}^{d}$, with entries $e_{ij} = \delta_{ij}$. We consider two label assignments: a) a realisable case in which clusters $1$ and $2$ are assigned label $+1$, and clusters $3$ and $4$ are assigned $-1$ and b) a non-realisable case in which clusters $1$ and $4$ are assigned $+1$ and clusters $2$ and $3$ are assigned $-1$ (XOR function), see Fig.~\ref{fig:app:xor} (\textit{top}) for an illustration. Now consider a dual $K=2$ Gaussian mixture model with means and covariances $(\bmu_{\pm}, \bSigma_{\pm})$ chosen to match the class means and covariances of the $K=4$ mixture, see Fig.~\ref{fig:app:xor} (\textit{bottom}) for an illustration. In Fig.~\ref{fig:app:2vs4} we compare the learning curves of the $K=4$ model with the $K=2$ counterpart with matched class means and covariances. While in the realisable case $a)$ both have identical performance under the error bars, in the non-realisable case $b)$ the performance in are significantly different.  

Indeed, a similar behaviour can be observed in the real data experiments. Fig.~\ref{fig:app:2vs10} compares the real learning curves of a MNIST 5v5 binary classification task (classifying five first digits vs. five last) with the two different Gaussian mixture approximations: $P_{10}$ where we fit the means and covariances of each individual cluster and $P_{2}$, where we fit only the class-wise means and covariances. While both approximations capture the high-level behaviour of the learning curves, $P_{10}$ is closer to the real learning curve than $P_{2}$.

\paragraph{Note on numerical instabilities} When dealing with means and covariance matrices estimated from real data sets, we have observed that for small regularisation strength $\lambda\ll 1$ the self-consistent equations from Theorem \ref{the:1} can develop spurious fixed points corresponding to negative values of the overlap parameters $q_{\pm} = \bW^{\top}\bSigma_{\pm}\bW$ -- which is clearly not possible since $\bSigma_{\pm}$ is a positive-definite matrix. This is observed across different scenarios, and is independent of the choice of loss or the particular way the equations are solved. In fact, the minimum value of $\lambda$ below which the spurious fixed point develop seems to depend only on the conditioning number of the covariance matrices.

\bibliographystyle{unsrt}
\bibliography{refs}

\begin{thebibliography}{10}

\bibitem{Geman1992}
Stuart Geman, Elie Bienenstock, and Ren\'e Doursat.
\newblock Neural networks and the bias/variance dilemma.
\newblock {\em Neural Computation}, 4(1):1--58, 1992.

\bibitem{Hastie2001}
Trevor Hastie, Robert Tibshirani, and Jerome Friedman.
\newblock {\em The Elements of Statistical Learning}.
\newblock Springer Series in Statistics. Springer New York Inc., New York, NY,
  USA, 2001.

\bibitem{Belkin2019}
Mikhail Belkin, Daniel Hsu, Siyuan Ma, and Soumik Mandal.
\newblock Reconciling modern machine-learning practice and the classical
  bias{\textendash}variance trade-off.
\newblock {\em Proceedings of the National Academy of Sciences},
  116(32):15849--15854, 2019.

\bibitem{Hastie2020}
Trevor Hastie, Andrea Montanari, Saharon Rosset, and Ryan~J. Tibshirani.
\newblock Surprises in high-dimensional ridgeless least squares interpolation.
\newblock {\em Preprint arXiv:1903.08560}, 2020.

\bibitem{Belkin2018}
Mikhail Belkin, Siyuan Ma, and Soumik Mandal.
\newblock To understand deep learning we need to understand kernel learning.
\newblock In Jennifer Dy and Andreas Krause, editors, {\em Proceedings of the
  35th International Conference on Machine Learning}, volume~80 of {\em
  Proceedings of Machine Learning Research}, pages 541--549. PMLR, 10--15 Jul
  2018.

\bibitem{Bartlett2020}
Peter~L. Bartlett, Philip~M. Long, G{\'a}bor Lugosi, and Alexander Tsigler.
\newblock Benign overfitting in linear regression.
\newblock {\em Proceedings of the National Academy of Sciences},
  117(48):30063--30070, 2020.

\bibitem{Mei2019}
Song Mei and Andrea Montanari.
\newblock The generalization error of random features regression: Precise
  asymptotics and double descent curve.
\newblock {\em Communications on Pure and Applied Mathematics}, 2019.
\newblock To appear, preprint arXiv:1908.05355.

\bibitem{Gerace2020}
Federica Gerace, Bruno Loureiro, Flornet Krzakala, Marc M{\'e}zard, and Lenka
  Zdeborov{\'a}.
\newblock Generalisation error in learning with random features and the hidden
  manifold model.
\newblock In {\em 37th International Conference on Machine Learning}, 2020.

\bibitem{ghorbani2019limitations}
Behrooz Ghorbani, Song Mei, Theodor Misiakiewicz, and Andrea Montanari.
\newblock Limitations of lazy training of two-layers neural network.
\newblock In H.~Wallach, H.~Larochelle, A.~Beygelzimer, F.~d\textquotesingle
  Alch\'{e}-Buc, E.~Fox, and R.~Garnett, editors, {\em Advances in Neural
  Information Processing Systems}, volume~32, pages 9111--9121, 2019.

\bibitem{Goldt2020}
Sebastian Goldt, Marc M{\'e}zard, Florent Krzakala, and Lenka Zdeborov{\'a}.
\newblock Modeling the influence of data structure on learning in neural
  networks: The hidden manifold model.
\newblock {\em Physical Review X}, 10(4):041044, 2020.

\bibitem{Goldt2020b}
Sebastian Goldt, Bruno Loureiro, Galen Reeves, Florent Krzakala, Marc Mézard,
  and Lenka Zdeborová.
\newblock {The Gaussian equivalence of generative models for learning with
  shallow neural networks}.
\newblock {\em Preprint arXiv:2006.14709}, 2020.

\bibitem{Loureiro2021}
Bruno Loureiro, Cédric Gerbelot, Hugo Cui, Sebastian Goldt, Florent Krzakala,
  Marc Mézard, and Lenka Zdeborová.
\newblock Capturing the learning curves of generic features maps for realistic
  data sets with a teacher-student model.
\newblock {\em Preprint arXiv:2102.08127}, 2021.

\bibitem{Sur2020}
Tengyuan Liang and Pragya Sur.
\newblock {A precise high-dimensional asymptotic theory for Boosting and
  minimum-$\ell_1$-norm interpolated classifiers}.
\newblock {\em Preprint arXiv:2002.01586}, 2020.

\bibitem{Mignacco2020}
Francesca Mignacco, Florent Krzakala, Yue Lu, Pierfrancesco Urbani, and Lenka
  Zdeborova.
\newblock The role of regularization in classification of high-dimensional
  noisy {G}aussian mixture.
\newblock In Hal~Daumé III and Aarti Singh, editors, {\em Proceedings of the
  37th International Conference on Machine Learning}, volume 119 of {\em
  Proceedings of Machine Learning Research}, pages 6874--6883. PMLR, 13--18 Jul
  2020.

\bibitem{Refinetti2021}
Maria Refinetti, Sebastian Goldt, Florent Krzakala, and Lenka Zdeborová.
\newblock Classifying high-dimensional gaussian mixtures: Where kernel methods
  fail and neural networks succeed.
\newblock {\em Preprint arXiv:2102.11742}, 2021.

\bibitem{Candes2020}
Emmanuel~J Cand{\`e}s, Pragya Sur, et~al.
\newblock The phase transition for the existence of the maximum likelihood
  estimate in high-dimensional logistic regression.
\newblock {\em The Annals of Statistics}, 48(1):27--42, 2020.

\bibitem{Donoho2020}
Vardan Papyan, X.~Y. Han, and David~L. Donoho.
\newblock Prevalence of neural collapse during the terminal phase of deep
  learning training.
\newblock {\em Proceedings of the National Academy of Sciences},
  117(40):24652--24663, 2020.

\bibitem{Couillet2020}
Mohamed El~Amine Seddik, Cosme Louart, Mohamed Tamaazousti, and Romain
  Couillet.
\newblock Random matrix theory proves that deep learning representations of
  {GAN}-data behave as {G}aussian mixtures.
\newblock In Hal~Daumé III and Aarti Singh, editors, {\em Proceedings of the
  37th International Conference on Machine Learning}, volume 119 of {\em
  Proceedings of Machine Learning Research}, pages 8573--8582. PMLR, 13--18 Jul
  2020.

\bibitem{Donoho2008}
David Donoho and Jiashun Jin.
\newblock Higher criticism thresholding: Optimal feature selection when useful
  features are rare and weak.
\newblock {\em Proceedings of the National Academy of Sciences},
  105(39):14790--14795, 2008.

\bibitem{Thrampoulidis2020}
Christos Thrampoulidis, Samet Oymak, and Mahdi Soltanolkotabi.
\newblock Theoretical insights into multiclass classification: A
  high-dimensional asymptotic view.
\newblock In H.~Larochelle, M.~Ranzato, R.~Hadsell, M.~F. Balcan, and H.~Lin,
  editors, {\em Advances in Neural Information Processing Systems}, volume~33,
  pages 8907--8920. Curran Associates, Inc., 2020.

\bibitem{Thrampoulidis2019}
Zeyu Deng, Abla Kammoun, and Christos Thrampoulidis.
\newblock A model of double descent for high-dimensional binary linear
  classification.
\newblock {\em Preprint arXiv:1911.05822}, 2020.

\bibitem{Couillet2019}
Xiaoyi Mai, Zhenyu Liao, and Romain Couillet.
\newblock A large scale analysis of logistic regression: Asymptotic performance
  and new insights.
\newblock In {\em ICASSP 2019 - 2019 IEEE International Conference on
  Acoustics, Speech and Signal Processing (ICASSP)}, pages 3357--3361, 2019.

\bibitem{Mai2020}
Xiaoyi Mai and Zhenyu Liao.
\newblock High dimensional classification via regularized and unregularized
  empirical risk minimization: Precise error and optimal loss.
\newblock {\em Preprint arXiv:1905.13742}, 2020.

\bibitem{Dobriban2018}
Edgar Dobriban and Stefan Wager.
\newblock {High-dimensional asymptotics of prediction: Ridge regression and
  classification}.
\newblock {\em The Annals of Statistics}, 46(1):247 -- 279, 2018.

\bibitem{Thrampoulidis2020b}
Ganesh Kini and Christos Thrampoulidis.
\newblock Analytic study of double descent in binary classification: The impact
  of loss.
\newblock {\em Preprint arXiv:2001.11572}, 2020.

\bibitem{Sifaou2019}
Houssem Sifaou, Abla Kammoun, and Mohamed-Slim Alouini.
\newblock Phase transition in the hard-margin support vector machines.
\newblock In {\em 2019 IEEE 8th International Workshop on Computational
  Advances in Multi-Sensor Adaptive Processing (CAMSAP)}, pages 415--419, 2019.

\bibitem{Thrampoulidis2021}
Ke~Wang and Christos Thrampoulidis.
\newblock Binary classification of gaussian mixtures: Abundance of support
  vectors, benign overfitting and regularization.
\newblock 2021.

\bibitem{Long2021}
Niladri~S. Chatterji and Philip~M. Long.
\newblock Finite-sample analysis of interpolating linear classifiers in the
  overparameterized regime.
\newblock {\em Preprint arXiv:2004.12019}, 2021.

\bibitem{Belkin2021}
Yuan Cao, Quanquan Gu, and Mikhail Belkin.
\newblock Risk bounds for over-parameterized maximum margin classification on
  sub-gaussian mixtures.
\newblock {\em Preprint arXiv:2104.13628}, 2021.

\bibitem{salehi2019impact}
Fariborz Salehi, Ehsan Abbasi, and Babak Hassibi.
\newblock The impact of regularization on high-dimensional logistic regression.
\newblock {\em Preprint arXiv:1906.03761}, 2019.

\bibitem{thrampoulidis2018precise}
Christos Thrampoulidis, Ehsan Abbasi, and Babak Hassibi.
\newblock Precise error analysis of regularized $ m $-estimators in high
  dimensions.
\newblock {\em IEEE Transactions on Information Theory}, 64(8):5592--5628,
  2018.

\bibitem{stojnic2013framework}
Mihailo Stojnic.
\newblock A framework to characterize performance of lasso algorithms.
\newblock {\em Preprint arXiv:1303.7291}, 2013.

\bibitem{bayati2011dynamics}
Mohsen Bayati and Andrea Montanari.
\newblock The dynamics of message passing on dense graphs, with applications to
  compressed sensing.
\newblock {\em IEEE Transactions on Information Theory}, 57(2):764--785, 2011.

\bibitem{krzakala2012probabilistic}
Florent Krzakala, Marc M{\'e}zard, Francois Sausset, Yifan Sun, and Lenka
  Zdeborov{\'a}.
\newblock Probabilistic reconstruction in compressed sensing: algorithms, phase
  diagrams, and threshold achieving matrices.
\newblock {\em Journal of Statistical Mechanics: Theory and Experiment},
  2012(08):P08009, 2012.

\bibitem{donoho2013information}
David~L Donoho, Adel Javanmard, and Andrea Montanari.
\newblock Information-theoretically optimal compressed sensing via spatial
  coupling and approximate message passing.
\newblock {\em IEEE transactions on information theory}, 59(11):7434--7464,
  2013.

\bibitem{javanmard2013state}
Adel Javanmard and Andrea Montanari.
\newblock State evolution for general approximate message passing algorithms,
  with applications to spatial coupling.
\newblock {\em Information and Inference: A Journal of the IMA}, 2(2):115--144,
  2013.

\bibitem{berthier2020state}
Raphael Berthier, Andrea Montanari, and Phan-Minh Nguyen.
\newblock State evolution for approximate message passing with non-separable
  functions.
\newblock {\em Information and Inference: A Journal of the IMA}, 9(1):33--79,
  2020.

\bibitem{manoel2017multi}
Andre Manoel, Florent Krzakala, Marc M{\'e}zard, and Lenka Zdeborov{\'a}.
\newblock Multi-layer generalized linear estimation.
\newblock In {\em 2017 IEEE International Symposium on Information Theory
  (ISIT)}, pages 2098--2102. IEEE, 2017.

\bibitem{Jiashun2009}
Jiashun Jin.
\newblock Impossibility of successful classification when useful features are
  rare and weak.
\newblock {\em Proceedings of the National Academy of Sciences},
  106(22):8859--8864, 2009.

\bibitem{Shao2011}
Jun Shao, Yazhen Wang, Xinwei Deng, and Sijian Wang.
\newblock {Sparse linear discriminant analysis by thresholding for high
  dimensional data}.
\newblock {\em The Annals of Statistics}, 39(2):1241 -- 1265, 2011.

\bibitem{Qing2012}
Qing Mai, Hui Zou, and Ming Yuan.
\newblock {A direct approach to sparse discriminant analysis in ultra-high
  dimensions}.
\newblock {\em Biometrika}, 99(1):29--42, 12 2012.

\bibitem{Li2017}
Yanfang Li and Jinzhu Jia.
\newblock {L1 least squares for sparse high-dimensional LDA}.
\newblock {\em Electronic Journal of Statistics}, 11(1):2499 -- 2518, 2017.

\bibitem{Cover1965}
Thomas~M Cover.
\newblock Geometrical and statistical properties of systems of linear
  inequalities with applications in pattern recognition.
\newblock {\em IEEE transactions on electronic computers}, (3):326--334, 1965.

\bibitem{gardner1988space}
Elizabeth Gardner.
\newblock The space of interactions in neural network models.
\newblock {\em Journal of physics A: Mathematical and general}, 21(1):257,
  1988.

\bibitem{jacot2020kernel}
Arthur Jacot, Berfin {\c{S}}im{\c{s}}ek, Francesco Spadaro, Cl{\'e}ment
  Hongler, and Franck Gabriel.
\newblock Kernel alignment risk estimator: Risk prediction from training data.
\newblock {\em Preprint arXiv:2006.09796}, 2020.

\bibitem{bordelon2020}
Blake Bordelon, Abdulkadir Canatar, and Cengiz Pehlevan.
\newblock Spectrum dependent learning curves in kernel regression and wide
  neural networks.
\newblock In {\em International Conference on Machine Learning}, pages
  1024--1034. PMLR, 2020.

\bibitem{parikh2014proximal}
Neal Parikh and Stephen Boyd.
\newblock Proximal algorithms.
\newblock {\em Foundations and Trends in optimization}, 1(3):127--239, 2014.

\bibitem{bauschke2011convex}
Heinz~H Bauschke, Patrick~L Combettes, et~al.
\newblock {\em Convex analysis and monotone operator theory in Hilbert spaces},
  volume 408.
\newblock Springer, 2011.

\bibitem{celentano2020lasso}
Michael Celentano, Andrea Montanari, and Yuting Wei.
\newblock {The Lasso with general Gaussian designs with applications to
  hypothesis testing}.
\newblock {\em Preprint arXiv:2007.13716}, 2020.

\bibitem{bolthausen2014iterative}
Erwin Bolthausen.
\newblock {An iterative construction of solutions of the TAP equations for the
  Sherrington--Kirkpatrick model}.
\newblock {\em Communications in Mathematical Physics}, 325(1):333--366, 2014.

\bibitem{bayati2011lasso}
Mohsen Bayati and Andrea Montanari.
\newblock {The LASSO risk for Gaussian matrices}.
\newblock {\em IEEE Transactions on Information Theory}, 58(4):1997--2017,
  2011.

\bibitem{gerbelot2020asymptotic}
Cedric Gerbelot, Alia Abbara, and Florent Krzakala.
\newblock {Asymptotic Errors for Teacher-Student Convex Generalized Linear
  Models (or: How to Prove Kabashima's Replica Formula)}.
\newblock {\em Preprint arXiv:2006.06581}, 2020.

\bibitem{aubin2019committee}
Benjamin Aubin, Antoine Maillard, Jean Barbier, Florent Krzakala, Nicolas
  Macris, and Lenka Zdeborov{\'a}.
\newblock The committee machine: Computational to statistical gaps in learning
  a two-layers neural network.
\newblock {\em Journal of Statistical Mechanics: Theory and Experiment},
  2019(12):124023, 2019.

\bibitem{krzakala2012statistical}
Florent Krzakala, Marc M{\'e}zard, Fran{\c{c}}ois Sausset, YF~Sun, and Lenka
  Zdeborov{\'a}.
\newblock Statistical-physics-based reconstruction in compressed sensing.
\newblock {\em Physical Review X}, 2(2):021005, 2012.

\bibitem{rush2018finite}
Cynthia Rush and Ramji Venkataramanan.
\newblock Finite sample analysis of approximate message passing algorithms.
\newblock {\em IEEE Transactions on Information Theory}, 64(11):7264--7286,
  2018.

\bibitem{bayati2015universality}
Mohsen Bayati, Marc Lelarge, Andrea Montanari, et~al.
\newblock Universality in polytope phase transitions and message passing
  algorithms.
\newblock {\em Annals of Applied Probability}, 25(2):753--822, 2015.

\bibitem{chen2021universality}
Wei-Kuo Chen and Wai-Kit Lam.
\newblock Universality of approximate message passing algorithms.
\newblock {\em Electronic Journal of Probability}, 26:1--44, 2021.

\bibitem{github}
Bruno Loureiro, Gabriele Sicuro, Cédric Gerbelot, Alessandro Pacco, Florent
  Krzakala, and Lenka Zdeborová.
\newblock {GaussMixtureProject}, October 2021.
\newblock \url{https://github.com/IdePHICS/GaussMixtureProject}.

\bibitem{Donoho2018}
Scott~Shaobing Chen, David~L. Donoho, and Michael~A. Saunders.
\newblock Atomic decomposition by basis pursuit.
\newblock {\em SIAM Journal on Scientific Computing}, 20(1):33--61, 1998.

\bibitem{pedregosa2011scikit}
Fabian Pedregosa, Ga{\"e}l Varoquaux, Alexandre Gramfort, Vincent Michel,
  Bertrand Thirion, Olivier Grisel, Mathieu Blondel, Peter Prettenhofer, Ron
  Weiss, Vincent Dubourg, et~al.
\newblock Scikit-learn: Machine learning in python.
\newblock {\em The Journal of Machine Learning Research}, 12:2825--2830, 2011.

\bibitem{lecunMNIST}
Yann LeCun and Corinna Cortes.
\newblock {\em ATT Labs [Online]}, 2010.
\newblock Database released under CC BY-SA 3.0 license at
  \url{http://yann.lecun.com/exdb/mnist/}.

\bibitem{xiao2017fashion}
Han Xiao, Kashif Rasul, and Roland Vollgraf.
\newblock {Fashion-MNIST: a novel image dataset for benchmarking machine
  learning algorithms}.
\newblock {\em Preprint arXiv:1708.07747}, 2017.
\newblock Database released under MIT licence at
  \url{https://github.com/zalandoresearch/fashion-mnist}.

\bibitem{rahimi2007random}
Ali Rahimi and Benjamin Recht.
\newblock {Random Features for Large-Scale Kernel Machines}.
\newblock In {\em NIPS}, pages 1177--1184, 2007.

\bibitem{bauschke2003bregman}
Heinz~H Bauschke, Jonathan~M Borwein, and Patrick~L Combettes.
\newblock Bregman monotone optimization algorithms.
\newblock {\em SIAM Journal on control and optimization}, 42(2):596--636, 2003.

\bibitem{bauschke2018regularizing}
Heinz~H Bauschke, Minh~N Dao, and Scott~B Lindstrom.
\newblock Regularizing with bregman--moreau envelopes.
\newblock {\em SIAM Journal on Optimization}, 28(4):3208--3228, 2018.

\bibitem{zdeborova2016statistical}
Lenka Zdeborov{\'a} and Florent Krzakala.
\newblock Statistical physics of inference: Thresholds and algorithms.
\newblock {\em Advances in Physics}, 65(5):453--552, 2016.

\bibitem{gerbelot2021graph}
C{\'e}dric Gerbelot and Rapha{\"e}l Berthier.
\newblock Graph-based approximate message passing iterations.
\newblock {\em arXiv preprint arXiv:2109.11905}, 2021.

\bibitem{tibshirani2013lasso}
Ryan~J Tibshirani.
\newblock The lasso problem and uniqueness.
\newblock {\em Electronic Journal of statistics}, 7:1456--1490, 2013.

\bibitem{mezard1987spin}
Marc M{\'e}zard, Giorgio Parisi, and Miguel Virasoro.
\newblock {\em Spin glass theory and beyond: An Introduction to the Replica
  Method and Its Applications}, volume~9.
\newblock World Scientific Publishing Company, 1987.

\bibitem{rosset2003margin}
Saharon Rosset, Ji~Zhu, and Trevor Hastie.
\newblock Margin maximizing loss functions.
\newblock In {\em NIPS}, pages 1237--1244, 2003.

\end{thebibliography}

\end{document}